\PassOptionsToPackage{unicode}{hyperref}
\PassOptionsToPackage{hyphens}{url}
\PassOptionsToPackage{dvipsnames,svgnames,x11names}{xcolor}
\documentclass[
  12pt]{article}
\usepackage{tabularx}
\usepackage{amsmath,amssymb}
\usepackage{amsthm}
\usepackage{mathrsfs}
\usepackage{comment}

\newtheorem{theorem}{Theorem}
\newtheorem{assumption}{Assumption}
\newtheorem{lemma}{Lemma}

\newtheorem{remark}{Remark}

\usepackage{titlesec}
\titlespacing*{\section}{0pt}{0.8ex plus 0.2ex minus 0.1ex}{0.5ex}
\titlespacing*{\subsection}{0pt}{0.8ex plus 0.2ex minus 0.1ex}{0.4ex}
\titlespacing*{\subsubsection}{0pt}{0.6ex plus 0.1ex minus 0.1ex}{0.3ex}
\titlespacing*{\paragraph}
  {0pt}   
  {0.4ex plus 0.1ex minus 0.1ex} 
  {0.6em} 

\usepackage{iftex}
\ifPDFTeX
  \usepackage[T1]{fontenc}
  \usepackage[utf8]{inputenc}
  \usepackage{textcomp} 
\else 
  \usepackage{unicode-math}
  \defaultfontfeatures{Scale=MatchLowercase}
  \defaultfontfeatures[\rmfamily]{Ligatures=TeX,Scale=1}
\fi
\usepackage{lmodern}
\ifPDFTeX\else  
\fi
\IfFileExists{upquote.sty}{\usepackage{upquote}}{}
\IfFileExists{microtype.sty}{
  \usepackage[]{microtype}
  \UseMicrotypeSet[protrusion]{basicmath} 
}{}
\makeatletter
\@ifundefined{KOMAClassName}{
  \IfFileExists{parskip.sty}{%
    \usepackage{parskip}
  }{
    \setlength{\parindent}{0pt}
    \setlength{\parskip}{6pt plus 2pt minus 1pt}}
}{
  \KOMAoptions{parskip=half}}
\makeatother
\usepackage{xcolor}
\makeatletter
\ifx\paragraph\undefined\else
  \let\oldparagraph\paragraph
  \renewcommand{\paragraph}{
    \@ifstar
      \xxxParagraphStar
      \xxxParagraphNoStar
  }
  \newcommand{\xxxParagraphStar}[1]{\oldparagraph*{#1}\mbox{}}
  \newcommand{\xxxParagraphNoStar}[1]{\oldparagraph{#1}\mbox{}}
\fi
\ifx\subparagraph\undefined\else
  \let\oldsubparagraph\subparagraph
  \renewcommand{\subparagraph}{
    \@ifstar
      \xxxSubParagraphStar
      \xxxSubParagraphNoStar
  }
  \newcommand{\xxxSubParagraphStar}[1]{\oldsubparagraph*{#1}\mbox{}}
  \newcommand{\xxxSubParagraphNoStar}[1]{\oldsubparagraph{#1}\mbox{}}
\fi
\makeatother

\usepackage{longtable,booktabs,array}
\usepackage{calc} 
\usepackage{etoolbox}
\makeatletter
\patchcmd\longtable{\par}{\if@noskipsec\mbox{}\fi\par}{}{}
\makeatother
\IfFileExists{footnotehyper.sty}{\usepackage{footnotehyper}}{\usepackage{footnote}}
\makesavenoteenv{longtable}
\usepackage{graphicx}
\makeatletter
\def\maxwidth{\ifdim\Gin@nat@width>\linewidth\linewidth\else\Gin@nat@width\fi}
\def\maxheight{\ifdim\Gin@nat@height>\textheight\textheight\else\Gin@nat@height\fi}
\makeatother
\setkeys{Gin}{width=\maxwidth,height=\maxheight,keepaspectratio}
\makeatletter
\def\fps@figure{htbp}
\makeatother

\makeatletter
\@ifpackageloaded{caption}{}{\usepackage{caption}}
\AtBeginDocument{%
\ifdefined\contentsname
  \renewcommand*\contentsname{Table of contents}
\else
  \newcommand\contentsname{Table of contents}
\fi
\ifdefined\listfigurename
  \renewcommand*\listfigurename{List of Figures}
\else
  \newcommand\listfigurename{List of Figures}
\fi
\ifdefined\listtablename
  \renewcommand*\listtablename{List of Tables}
\else
  \newcommand\listtablename{List of Tables}
\fi
\ifdefined\figurename
  \renewcommand*\figurename{Figure}
\else
  \newcommand\figurename{Figure}
\fi
\ifdefined\tablename
  \renewcommand*\tablename{Table}
\else
  \newcommand\tablename{Table}
\fi
}
\@ifpackageloaded{float}{}{\usepackage{float}}
\floatstyle{ruled}
\@ifundefined{c@chapter}{\newfloat{codelisting}{h}{lop}}{\newfloat{codelisting}{h}{lop}[chapter]}
\floatname{codelisting}{Listing}

\makeatother
\makeatletter
\@ifpackageloaded{caption}{}{\usepackage{caption}}
\@ifpackageloaded{subcaption}{}{\usepackage{subcaption}}
\makeatother


\ifLuaTeX
  \usepackage{selnolig}  
\fi
\usepackage[]{natbib}
\usepackage{bookmark}
\usepackage{amsmath}

\usepackage{bbm}
\usepackage{tikz}
\usepackage{subcaption}
\usepackage{algorithm}
\usepackage{algorithmic}

\usepackage{tikz}
\usetikzlibrary{decorations.pathreplacing}
\usetikzlibrary{arrows.meta,positioning}
\IfFileExists{xurl.sty}{\usepackage{xurl}}{} 
\hypersetup{
  pdftitle={Bayesian Deep Discrete Encoders with Unknown Layer Widths for Multivariate Data with Arbitrary Marginal Distributions},
  pdfauthor={Author 1; Author 2},
  pdfkeywords={3 to 6 keywords, that do not appear in the title},
  colorlinks=true,
  linkcolor={blue},
  filecolor={Maroon},
  citecolor={Blue},
  urlcolor={Blue},
  pdfcreator={LaTeX via pandoc}}

\newcommand{\anon}{1}

\usepackage{multirow}
\usepackage{makecell}
\usepackage{mathtools}
\usepackage{rotating}
\usepackage{lscape}
\usepackage{multicol}
\usepackage{csquotes}
\usepackage{setspace}
\usepackage{cases}
\usepackage{chngcntr}

\newlength\myindent
\theoremstyle{remark}

\newif\ifblinded
\blindedtrue

\begin{document}

\def\spacingset#1{\renewcommand{\baselinestretch}%
{#1}\small\normalsize} \spacingset{1}


\if1\anon
{
  \title{Identifiable Bayesian Deep Generative Copulas with Unknown Layer Widths for Data with Arbitrary Marginal Distributions}
  \author{Joseph Feldman\thanks{\textit{feldmanjr@wustl.edu}}\\
    Department of Statistics and Data Science\\
    Washington University in St. Louis \\[0.5em]
    Yuqi Gu\thanks{\textit{yuqi.gu@columbia.edu}} \\
    Department of Statistics\\
    Columbia University}
\date{}
  \maketitle
} \fi

\if0\anon
{
  \bigskip
  \bigskip
  \bigskip
  \begin{center}
    {\LARGE\bf Title}
\end{center}
  \medskip
} \fi

\date{}

\begin{abstract}
Deep generative models offer powerful tools for multivariate data analysis, but their black-box architectures are often unidentified and difficult to interpret. We introduce the Deep Discrete Encoder (DDE) Copula, an identifiable and interpretable generative model for multivariate data with arbitrary marginal distributions. The model places a hierarchical directed network of binary latent variables inside a copula framework, enabling flexible dependence modeling for mixed discrete and continuous data. Estimation is based on rank likelihoods, which decouple marginal modeling from posterior inference on the DDE parameters and avoid specifying the marginal distributions.
We establish conditions for identification of the DDE copula parameters, ensuring that layer-specific parameters provide meaningful summaries of multivariate dependence. We also prove quotient-space posterior consistency for continuous margins under the exact rank likelihood and treat the extended rank likelihood for tied or mixed margins as a generalized likelihood, with concentration under an additional contrast condition.
For computation, we propose a stochastic expectation-maximization algorithm for \emph{maximum a posteriori} estimation, together with initialization strategies that improve convergence. To learn network dimension adaptively, we extend Bayesian rank-selection priors to infer layer-specific widths. Simulations show strong finite-sample performance, and a personality-survey analysis reveals interpretable hierarchical latent structure in complex multivariate data.

\end{abstract}

\noindent%
\vfill

\newpage
\spacingset{1.75} 
\setlength{\abovedisplayskip}{6pt plus 2pt minus 2pt}
\setlength{\belowdisplayskip}{6pt plus 2pt minus 2pt}
\setlength{\abovedisplayshortskip}{3pt plus 1pt minus 1pt}
\setlength{\belowdisplayshortskip}{4pt plus 1pt minus 1pt}

\section{Introduction}
The widespread popularity of AI tools are due to the impressive predictive capabilities of underlying deep generative models (DGMs). Many DGMs rely on black-box, deep learning architectures governed by extremely large numbers of parameters relative to the training sample size. Such DGMs are fundamentally unidentified, and the estimated parameters lack interpretation of the substantive relationships in the data. From a statistical perspective, this is problematic; indeed, statistical inference is a cornerstone for scientific discovery.

In this work, we contribute to a rapidly developing literature focusing on developing statistically identified and simultaneously interpretable DGMs. Specifically, we introduce the Deep Discrete Encoder (DDE) Copula, a generative model with multiple latent layers for multivariate data with arbitrary marginal distributions. The DDE copula features a hierarchical, pyramid-like directed graph where the latent variables in each layer are binary and conditionally independent given parents (See Figure \ref{fig:DDE_right}). We then nest this structure -- capable of capturing complex dependencies among diverse data types -- within a copula to accommodate variables with arbitrary scales and marginal features. 

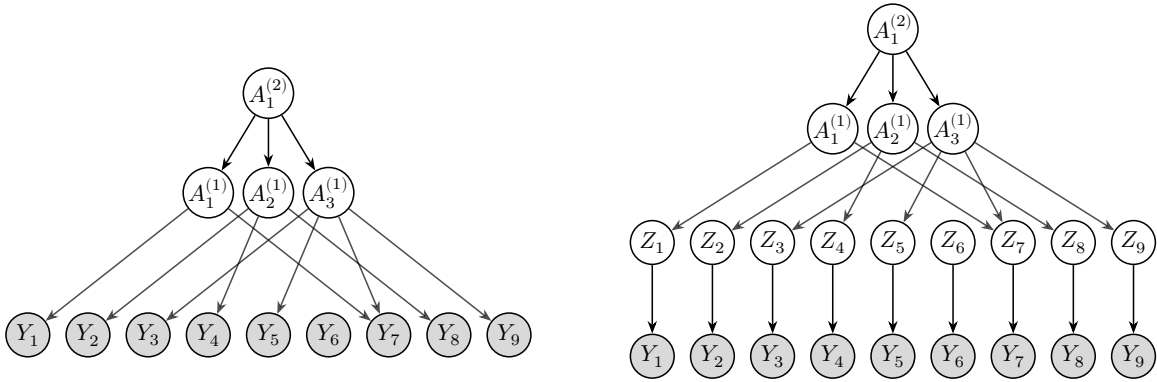
\begin{figure}[h!]
\centering

\begin{subfigure}[t]{0.48\textwidth}
\centering
\resizebox{\linewidth}{!}{%
\begin{tikzpicture}[
  >=Stealth, thick,
  latent/.style={circle,draw,minimum size=8mm,inner sep=0pt},
  obs/.style={latent,fill=gray!30},
  x=1.1cm, y=1.3cm
]
\path[use as bounding box] (-4.9,0.7) rectangle (4.9,-4.4);

\node[latent] (La2) at (0,0) {$A^{(2)}_{1}$};

\node[latent] (La11) at (-1,-1.4) {$A^{(1)}_{1}$};
\node[latent] (La12) at ( 0,-1.4) {$A^{(1)}_{2}$};
\node[latent] (La13) at ( 1,-1.4) {$A^{(1)}_{3}$};

\foreach \i/\x in {1/-4,2/-3,3/-2,4/-1,5/0,6/1,7/2,8/3,9/4}
  \node[obs] (Ly\i) at (\x,-3.4) {$Y_{\i}$};

\draw[->] (La2) -- (La11);
\draw[->] (La2) -- (La12);
\draw[->] (La2) -- (La13);

\foreach \j in {1,7}   \draw[->,opacity=0.7] (La11) -- (Ly\j);
\foreach \j in {2,4,8} \draw[->,opacity=0.7] (La12) -- (Ly\j);
\foreach \j in {3,5,7,9} \draw[->,opacity=0.7] (La13) -- (Ly\j);

\end{tikzpicture}%
}
\caption{$D=2$-Layer DDE.}
\label{fig:DDE_left}
\end{subfigure}\hspace{0.02\textwidth}
\begin{subfigure}[t]{0.48\textwidth}
\centering
\resizebox{\linewidth}{!}{%
\begin{tikzpicture}[
  >=Stealth, thick,
  latent/.style={circle,draw,minimum size=8mm,inner sep=0pt},
  obs/.style={latent,fill=gray!30},
  x=1.1cm, y=1.3cm
]
\path[use as bounding box] (-4.9,0.7) rectangle (4.9,-5.3);

\node[latent] (a2) at (0,0) {$A^{(2)}_{1}$};

\node[latent] (a11) at (-1,-1.4) {$A^{(1)}_{1}$};
\node[latent] (a12) at ( 0,-1.4) {$A^{(1)}_{2}$};
\node[latent] (a13) at ( 1,-1.4) {$A^{(1)}_{3}$};

\foreach \i/\x in {1/-4,2/-3,3/-2,4/-1,5/0,6/1,7/2,8/3,9/4}
  \node[latent] (z\i) at (\x,-3.0) {$Z_{\i}$};

\foreach \i/\x in {1/-4,2/-3,3/-2,4/-1,5/0,6/1,7/2,8/3,9/4}
  \node[obs] (y\i) at (\x,-4.6) {$Y_{\i}$};

\draw[->] (a2) -- (a11);
\draw[->] (a2) -- (a12);
\draw[->] (a2) -- (a13);

\foreach \j in {1,7}     \draw[->,opacity=0.7] (a11) -- (z\j);
\foreach \j in {2,4,8}   \draw[->,opacity=0.7] (a12) -- (z\j);
\foreach \j in {3,5,7,9} \draw[->,opacity=0.7] (a13) -- (z\j);

\foreach \i in {1,...,9} \draw[->] (z\i) -- (y\i);

\end{tikzpicture}%
}
\caption{$D=2$-layer DDE Copula  }
\label{fig:DDE_right}
\end{subfigure}

\caption{Comparison of architectures with and without an intermediate latent layer.}
\label{fig:DDE_copDDE}
\end{figure}

We first provide identifiability conditions for the DDE copula and show that when these are satisfied, the DDE copula is \emph{strictly identifiable}. This is vitally important for real data analyses, as the parameters governing each layer of the network provide directly interpretable and increasingly fine grained insight into multivariate relationships. As we show in a real data analysis (Section \ref{sec:realdat}), these novel insights demonstrate how the DDE copula is a powerful exploratory model.

From a computational standpoint,  we introduce several algorithmic novelties that enable scalable Bayesian estimation of the DDE Copula with unknown layer widths. First, we decouple marginal modeling from DDE posterior inference using an extended rank likelihood \citep{hoffRL}. The extended rank likelihood leverages multivariate ranks -- not the observed values -- among the data to estimate the DDE parameters. As such, the univariate marginals are nuisance parameters, which greatly facilitates estimation. Due to this attractive property, the extended rank likelihood has been used for Bayesian estimation of Gaussian copulas and Gaussian mixture copulas \citep{murray2013bayesian,feldman2022bayesian, feldman2024nonparametric}. We show how to leverage the rank likelihood for deep copula models.

{We also clarify the large-sample target of rank likelihood inference for the DDE copula: for continuous margins, we prove quotient-space posterior consistency under the exact rank likelihood; for discrete or mixed margins, where ties make the extended rank likelihood \citep{hoffRL} a generalized likelihood rather than an exact likelihood, we state posterior concentration under an explicit extended-rank contrast condition (Section \ref{subsec:rl-identifiability}).} We also confirm excellent finite sample behavior in simulation (Section \ref{sec:sim}). 

Next, we construct an efficient expectation-maximization algorithm for \emph{maximum a posteriori} (MAP) estimation of the DDE parameters under the rank likelihood (Section \ref{sec:spec_est}). The algorithm conveniently accommodates the rank likelihood, as an intractable E-step is replaced by a simulation step for all latent variables \citep{booth1999maximizing, li2025sparse}. The stochastic E-step also further simplifies the M-step, enabling reasonable computation times for large networks and data sets. We also provide novel initialization strategies that facilitate rapid convergence (Section \ref{sec:init}).

Finally, we demonstrate how state-of-the-art Bayesian priors for rank-selection in Gaussian factor models \citep{legramanti2020bayesian} may be extended to automate layer-specific width selection for the DDE copula architecture. Our formulation enables information sharing between layers and provides structured sparsity that improves interpretability. The ability to estimate layer-specific widths under the DDE copula provides analysts with an additional tool to discover hierarchical latent structures in their data.

The DDE copula is directly motivated by recent work on Bayesian Pyramids \citep{gu2023bayesian} and Deep Discrete Encoders (Figure \ref{fig:DDE_left}; \cite{lee2026dde}), two deep and statistically identifiable generative models for multivariate data. However, the distinguishing features of the DDE copula are motivated by two primary limitations of these preceding models. First, the Deep Discrete Encoder features parametric models for each study variable at the observed data layer of the architecture, while the Bayesian Pyramid is only compatible with multivariate categorical data. The DDE copula is broadly compatible with mixed discrete and continuous variables \emph{without} any marginal distributional assumptions. Second, both models require pre-specification of at least some of the unknown layer widths. 
We instead treat the dimensions of the DDE network as unknown model components and
infer each layer-specific width within our MAP estimation algorithm.



More broadly, the DDE copula sits at the intersection of identifiable deep generative
models and semiparametric copula models.
Classical deep latent-variable architectures, including deep belief networks \citep{hinton2009deep}, deep Boltzmann machines \citep{salakhutdinov2009deep}, and variational autoencoders \citep{kingma2014autoencoding}, provide flexible representation learning but are typically not identifiable without additional structural assumptions \citep{ivae} and are difficult to interpret. In contrast, the recent DDE and Bayesian
Pyramid frameworks use sparse discrete latent layers to obtain interpretable hierarchical
representations in parametric models \citep{gu2023bayesian,lee2026dde}. On the other hand, semiparametric Gaussian copula
models accommodate mixed margins through rank likelihoods
\citep{hoffRL,murray2013bayesian,feldman2024nonparametric}, but typically rely on
shallow continuous latent factors. The DDE copula combines these strengths by retaining
a statistically identifiable deep discrete latent hierarchy while allowing arbitrary marginal distributions, and further addresses the crucial question of automatically estimating the latent layer widths in a statistically principled manner.

This paper is organized as follows: In Section \ref{sec:DDE}, we review the Deep Discrete Encoder of \cite{lee2026dde}. In Section \ref{sec:DDEcop}, we introduce the DDE copula, demonstrate how to leverage the rank likelihood for estimation, and establish key theoretical properties including identifiability and {rank-likelihood large-sample guarantees}. Section \ref{sec:spec_est} includes details on Bayesian specification of the DDE copula, as well as estimation and initialization algorithms. We provide an extensive simulation study in Section \ref{sec:sim}, and conclude with an analysis of survey responses on human personality in Section \ref{sec:realdat}. Concluding thoughts and ideas for future research are provided in Section \ref{sec:disc}. The Supplementary Material contains all proofs of the theoretical results and additional numerical results. Codes implementing the proposed approach are available \url{https://github.com/jfeldman396/DDE-Copula/tree/main}.

\section{Deep Discrete Encoders}\label{sec:DDE}

We first describe the Deep Discrete Encoder \citep[DDE,][]{lee2026dde}, which provides the basis for our proposed innovations. A $D$-layer DDE forms a directed graphical model with a pyramid-like shape. At the shallowest and widest layer ($d=0$) are the $J$-dimensional observed data $\boldsymbol{Y} = (Y_{1}, \dots, Y_{J})$. Each subsequent layer $d \in \{1, \dots, D\}$ is governed by $K^{(d)}$ binary latent variables $\boldsymbol A^{(d)} = (A^{(d)}_{1}, \dots, A^{(d)}_{K^{(d)}}) \in \{0,1\}^{K^{(d)}}$ with $K^{D}< \dots < K^{1} < J$. For identifiability, the dimension of each layer $d$ is constrained to be $\lfloor K^{(d-1)}/3\rfloor$, with $K^{(1)} = \lfloor J /3 \rfloor$ \citep{gu2023bayesian, lee2026dde}. In Figure \ref{fig:DDE_left}, we provide an illustration of this architecture.

To describe the data-generating model under a $D$-dimensional DDE, the top layer latent variables are assumed to be independent Bernoulli random variables with parameter $\boldsymbol \pi = (\pi_{1}, \dots, \pi_{K^{(D)}})$. Then, latent variables in each subsequent middle layer are conditionally independent given the binary variables in the previous layer, modeled with additive logistic regressions. That is, conditional on any $d$-th level binary representation $\boldsymbol \alpha^{(d)} \in \{0,1\}^{K^{(d)}}$, the joint probability for any $\boldsymbol \alpha^{(d-1)} \in \{0,1\}^{K^{(d-1)}}$ factors as
\begin{align}\label{recursmiddle}
p(\boldsymbol A^{(d-1)} = \boldsymbol \alpha^{(d-1)} \mid \boldsymbol \alpha^{(d)}) &= \prod_{k=1}^{K^{(d-1)}} p(A^{(d-1)}_{k} = \alpha^{(d-1)}_{k}\mid \boldsymbol \alpha^{(d)})\\
&= \prod_{k=1}^{K^{(d-1)}} g_{\text{logistic}}(\beta^{(d)}_{k,0}+ \sum_{\ell =1}^{K^{(d)}} \beta^{(d)}_{k,l} \alpha^{(d)}_{l}\label{recurslogistic} )
\end{align}
Here, $g_{\text{logistic}}$ is the logistic link function $g(x) = 1/1+e^{-x}$. Modifications are available for non-linear and non-additive relationships between the layers \citep{lee2024grcdm, liu2025exploratory}.
This formulation implies a set of layer-specific weight matrices $\boldsymbol B^{(d)} = \{\beta_{jk}\}, \  \forall d \in \{1, \dots, D\}$, where the $d$-th layer-specific dimension is $K^{(d-1)} \times K^{(d)}$.

At the observed data layer, the likelihood is
\begin{equation}\label{outcome}
    p(\boldsymbol Y \mid \boldsymbol \alpha^{(1)}) = \prod_{j=1}^{J} p_{j}\{g_{j}(Y_{j};\beta^{(1)}_{j,0}  + \sum_{k =1}^{K^{(1)}} \beta^{(1)}_{j,k} \alpha^{(1)}_{k}, \gamma_{j})\}
\end{equation}
where $p_{j}$ is the probability density/mass function of a user-specified parametric family corresponding to the scale of $Y_{j}$ (e.g., count, continuous, binary), $\gamma_{j}$ is a dispersion parameter when required, and $g_{j}$ is the appropriate link function. For example, for count-valued $Y_{j}$, the user may specify $p_{j}$ to be the Poisson probability mass function with link function $g_{j}(x) = e^{x}$ and no dispersion parameter.

From the construction \eqref{recursmiddle}--\eqref{outcome}, the joint probability of the observed and latent variables is 
\begin{equation} \label{eq::joint}
    p(\boldsymbol Y, \{\boldsymbol A^{(d)}\}_{d \in [D]}) = p(\boldsymbol Y \mid \boldsymbol A^{(1)})\prod_{d=2}^{D}p(\boldsymbol A^{(d-1)} \mid \boldsymbol A^{(d)})p(\boldsymbol A^{(D)})
\end{equation}
This distribution is parameterized by the set $\boldsymbol{\theta} = (\boldsymbol B, \boldsymbol \pi, \boldsymbol \gamma)$, where  $\boldsymbol B= \{\boldsymbol B^{(d)}\}_{d=1}^{D}$ are the layer-specific weight matrices and $\boldsymbol \gamma = (\gamma_{1},\dots, \gamma_{J})$. Marginalization over the latent variables in \eqref{eq::joint} reveals that the distribution of \textcolor{black}{$\boldsymbol Y$ is a mixture with $ 2^{\sum_{d=1}^{D} K^{(d)}}$  components}. This yields a highly expressive model capable of capturing hierarchical dependence in increasingly complex and high-dimensional data, including text, images, and multi-modal data sources \citep{lee2026dde}. A primary statistical advantage lies in the observation that, despite the saturated marginal model, the DDE requires estimation of only $J + K^{(D)} + \sum_{d=2}^{D} K^{(d)} \times K^{(d-1)}$ parameters, which comprises the exponential family dispersion parameters $\boldsymbol \gamma$, the top-level Bernoulli probabilities $\boldsymbol \pi$, and the layer-specific weight matrices $\boldsymbol B$, respectively.

In practice, the specification of a DDE requires the user to fix the widths of the layers given the dimension of $
\boldsymbol Y$, along with the parametric family for each marginal model $p_{j}$ \citep{lee2026dde}. For the former, the authors proposed several ad-hoc approaches based on score-based metrics and layer-specific spectral estimators. For the latter, depending on the structure of the data, the model is amenable to specification of different exponential families, including conditionally Normal, Poisson, and Bernoulli distributions. With these choices, the authors developed a stochastic approximation expectation-maximization algorithm \citep[SAEM,][]{delyon1999convergence} to infer $\boldsymbol \theta$. Next, we propose methods that enable \emph{learning} of the DDE layer widths, in addition to a \emph{relaxation} of the parametric marginal assumption.

\section{The DDE Copula and Its Theoretical Properties}\label{sec:DDEcop}

The reliance of the DDE on parametric marginal models for each $Y_j$ limits its applicability. For example, in survey data with ordinal responses (Section \ref{sec:realdat}), each $Y_j$ has bounded, discrete support, making it difficult to specify a suitable parametric model $p_j$ in \eqref{outcome}. As demonstrated in our simulations (Section \ref{sec:sim}), misspecification of marginal features such as boundedness, skewness, or multimodality can degrade estimation of the latent dependence structure through $\boldsymbol B$.

Copulas provide a natural remedy. By Sklar's theorem \citep{sklar1959fonctions}, any joint distribution $\boldsymbol F$ can be decomposed into marginal distribution functions $\{F_{Y_j}\}_{j=1}^J$ and a copula $\boldsymbol C$ parameterized by $\boldsymbol \Theta$,
$\boldsymbol F(y_1,\dots,y_J)
=
\boldsymbol C_{\boldsymbol{\Theta}}\{F_{Y_1}(y_1),\dots,F_{Y_J}(y_J)\}$.
We construct a DDE copula by combining the hierarchical dependence structure of a $D$-layer DDE with arbitrary marginals. This yields a flexible model capable of handling data with diverse marginal features, while preserving interpretable dependence through the DDE parameters.

We introduce a latent Gaussian layer $\boldsymbol Z$ between $\boldsymbol A^{(1)}$ and $\boldsymbol Y$. Fixing the middle-layer structure \eqref{recursmiddle}--\eqref{recurslogistic}, we define
\begin{align}
p(\boldsymbol Z \mid \boldsymbol \alpha^{(1)})
&=
\prod_{j=1}^{J}
p_j\{g_j(Z_j; \beta^{(1)}_{j,0} + \sum_{k=1}^{K^{(1)}} \beta^{(1)}_{j,k} \alpha_k^{(1)}, \gamma_j)\}, \label{condZ} \\
Y_j &= F_{Y_j}^{-1}\{F_{Z_j}(Z_j)\}, \quad j=1,\dots,J. \label{eq::PIT}
\end{align}
The transformation \eqref{eq::PIT} induces a copula in which dependence is governed by the DDE on $\boldsymbol Z$, while the marginals $\{F_{Y_j}\}$ remain unrestricted. The DDE copula architecture is visualized in Figure \ref{fig:DDE_right}.
For computational convenience, we take $p_j$ to be Gaussian and $g_j$ the identity,
$p(\boldsymbol Z \mid \boldsymbol \alpha^{(1)})
=
\prod_{j=1}^{J}
\phi\!\left(Z_j; \beta^{(1)}_{j,0} + \sum_{k=1}^{K^{(1)}} \beta_{j,k} \alpha_k^{(1)}, \gamma_j\right)$,
which yields substantial computational gains while retaining flexible marginal modeling.


Marginally, $\boldsymbol Z$ follows a Gaussian mixture induced by $\boldsymbol A$. The DDE copula thus resembles a Gaussian mixture copula \citep{feldman2024nonparametric}, but achieves substantially greater modeling capacity and statistical parsimony. For example, with $J=50$, $K^{(1)}=16$, and $K^{(2)}=5$, the implied mixture has \textcolor{black}{$2^{K^{(1)} + K^{(2)}}=2^{21}$} components, yet requires only $50 + 5 + 50\times16 + 16\times5$ parameters. This feature enables the DDE copula to be estimated on modestly-sized data sets.

\subsection{Identifiability Theory of the Canonical DDE Copula}
\label{subsec:copuladde-identifiability}
We next establish identifiability of the dependence parameters in the DDE copula. 
The result is stated for the copula law induced by the latent Gaussian vector \(Z\), 
because the univariate marginal distributions of the observed variables \(Y_1,\ldots,Y_J\) 
are left unrestricted and are therefore nuisance parameters.

Recall that the DDE copula introduces a latent Gaussian layer \(\boldsymbol Z=(Z_1,\ldots,Z_J)\)
between the shallowest binary latent layer \(\boldsymbol A^{(1)}\) and the observed data \(\boldsymbol Y\). 
For \(\boldsymbol \alpha\in\{0,1\}^{K^{(1)}}\), write
\[
    Z_j\mid \boldsymbol A^{(1)}=\boldsymbol \alpha
    \sim 
    N(\mu_{j\boldsymbol \alpha},\gamma_j),
    \qquad
    \mu_{j\boldsymbol \alpha}
    =
    \beta^{(1)}_{j0}
    +
    \sum_{k=1}^{K^{(1)}}\beta^{(1)}_{jk}\alpha_k,
    \qquad
    \gamma_j>0 .
\]
The observed variable is generated by the monotone marginal transformation
$Y_j
    =
    F_{Y_j}^{-1}\{F_{Z_j}(Z_j)\}$, $j=1,\ldots,J,$
where \(F_{Y_j}\) is unrestricted and \(F_{Z_j}\) is the marginal distribution function of
\(Z_j\) under the DDE model. Let
$U_j = F_{Z_j}(Z_j)$, $j=1,\ldots,J.$
Then \(\boldsymbol U=(U_1,\ldots,U_J)\) has uniform margins, and its joint distribution is the
copula induced by the latent DDE model. We denote this copula by \(C_{\boldsymbol{\Theta}}\), where
\[
    \boldsymbol{\Theta} =
    \Bigl(
       \boldsymbol \pi,~\{\boldsymbol B^{(d)}\}_{d=1}^D,~\boldsymbol \gamma,~\{\boldsymbol G^{(d)}\}_{d=1}^D
    \Bigr)
\]
collects the DDE parameters. Here \(\boldsymbol \pi=(\pi_1,\ldots,\pi_{K^{(D)}})\) are the independent
Bernoulli probabilities in the top latent layer, \(\boldsymbol B^{(d)}\) is the coefficient matrix between
layers \(d\) and \(d-1\), \(\boldsymbol \gamma=(\gamma_1,\ldots,\gamma_J)\), and
$g^{(d)}_{\ell k}=1\{\beta^{(d)}_{\ell k}\ne 0\}$.
defines the graphical matrix \(\boldsymbol G^{(d)}\). We use the convention \(K^{(0)}=J\), so that rows
of \(\boldsymbol B^{(1)}\) correspond to \(Z_1,\ldots,Z_J\), whereas rows of \(\boldsymbol B^{(d)}\), \(d\ge2\), correspond
to the variables in \(\boldsymbol A^{(d-1)}\).

\paragraph*{Canonical parameterization.}
The copula is invariant under strictly increasing transformations of each coordinate \(Z_j\).
Within the Gaussian first layer, positive affine transformations
\(Z_j\mapsto a_j Z_j+b_j\), \(a_j>0\), preserve the copula while changing
\((\beta^{(1)}_{j0},\beta^{(1)}_{j1},\ldots,\beta^{(1)}_{jK^{(1)}},\gamma_j)\).
We therefore impose the canonical normalization
\begin{equation}
    E_{\boldsymbol{\Theta}}(Z_j)=0,
    \qquad
    \operatorname{Var}_{\boldsymbol{\Theta}}(Z_j)=1,
    \qquad j=1,\ldots,J .
    \label{eq:canonical-normalization}
\end{equation}
Equivalently, if
$\eta^{(1)}_{\boldsymbol{\alpha}}
    =
    P_{\boldsymbol{\Theta}}(\boldsymbol A^{(1)}={\boldsymbol{\alpha}}),$ ${\boldsymbol{\alpha}}\in\{0,1\}^{K^{(1)}},$
then
\begin{equation}
    \sum_{{\boldsymbol{\alpha}}}\eta^{(1)}_{\boldsymbol{\alpha}} \mu_{j{\boldsymbol{\alpha}}}=0,
    \qquad
    \gamma_j+
    \sum_{{\boldsymbol{\alpha}}}\eta^{(1)}_{\boldsymbol{\alpha}} \mu_{j{\boldsymbol{\alpha}}}^2
    =
    1,
    \qquad j=1,\ldots,J .
    \label{eq:canonical-moment-equations}
\end{equation}
The first equation is \(E(Z_j)=0\); the second follows from \(E(Z_j^2)=E[\gamma_j+\mu_{j,A^{(1)}}^2]=1\),
which equals \(\operatorname{Var}(Z_j)=1\) precisely because the mean is zero.

As in latent-variable models generally, the latent coordinates within each layer are only
identifiable up to label permutations. We say that two parameters \(\boldsymbol{\Theta}\) and \(\widetilde{\boldsymbol{\Theta}}\)
are equivalent, written \(\boldsymbol{\Theta}\sim_K\widetilde{\boldsymbol{\Theta}}\), if there exist permutations
\(\sigma^{(d)}\in S_{K^{(d)}}\), \(d=1,\ldots,D\), such that all latent-layer coordinates,
columns and rows of the corresponding coefficient matrices, graphical matrices, and top-layer
probabilities are transformed according to these permutations, while the observed coordinates
\(j=1,\ldots,J\) are fixed.

\begin{assumption}[Known widths and positivity]
\label{ass:known-widths-positivity}
The number of latent layers \(D\) and the layer widths
\(K=(K^{(1)},\ldots,K^{(D)})\) are known. Moreover,
$\pi_k\in(0,1)$, $k=1,\ldots,K^{(D)}$,
all coefficients in \(\boldsymbol B^{(d)}\) are finite, and \(\gamma_j>0\) for every \(j=1,\ldots,J\).
Consequently, every latent configuration in every layer has positive probability.
\end{assumption}

\begin{assumption}[Faithfulness and orientation]
\label{ass:faithfulness-orientation}
For every layer \(d=1,\ldots,D\),
$g_{\ell k}^{(d)}=1$ if and only if
    $\beta_{\ell k}^{(d)}\ne0$.
No graphical matrix \(G^{(d)}\) has an all-zero column. To remove the trivial ambiguity
created by replacing a binary latent variable \(A_k^{(d)}\) by \(1-A_k^{(d)}\), we impose
the orientation condition
\begin{equation}
    \sum_{\ell=1}^{K^{(d-1)}} \beta_{\ell k}^{(d)} >0,
    \qquad
    k=1,\ldots,K^{(d)},
    \qquad
    d=1,\ldots,D .
    \label{eq:orientation-condition}
\end{equation}
\end{assumption}

\begin{assumption}[Two pure children per latent plus separation]
\label{ass:pure-separation}
For each layer \(d=1,\ldots,D\) and each latent variable \(A_k^{(d)}\), there exist
two distinct pure children
$r^{(d)}_{k,1},\ r^{(d)}_{k,2}
    \in
    \{1,\ldots,K^{(d-1)}\},$
such that
$g^{(d)}_{r^{(d)}_{k,a},k}=1$,
    $g^{(d)}_{r^{(d)}_{k,a},\ell}=0$
    for all $\ell\ne k$, $a=1,2$.
The \(2K^{(d)}\) pure-child indices are distinct. In addition, letting
$\mathcal P_d
    =
    \{r^{(d)}_{k,a}: k=1,\ldots,K^{(d)},\ a=1,2\},$
for every pair \({\boldsymbol{\alpha}}\ne{\boldsymbol{\alpha}}'\) in \(\{0,1\}^{K^{(d)}}\), there exists
    $r\in \{1,\ldots,K^{(d-1)}\}\setminus\mathcal P_d$
such that
\begin{equation}
    \sum_{k=1}^{K^{(d)}}\beta^{(d)}_{rk}(\alpha_k-\alpha_k')\ne0 .
    \label{eq:separation-condition}
\end{equation}
\end{assumption}

\begin{remark}[A simpler sufficient condition]
\label{rem:three-pure-children}
Assumption~\ref{ass:pure-separation} is implied by the more transparent condition that
each latent variable in every layer has three distinct pure children with nonzero coefficients.
Indeed, if \({\boldsymbol{\alpha}}\ne{\boldsymbol{\alpha}}'\), choose \(k\) such that \(\alpha_k\ne\alpha'_k\); the third pure
child of \(A_k^{(d)}\), not used in the two pure-child groups, satisfies
$\sum_{\ell}\beta^{(d)}_{r\ell}(\alpha_\ell-\alpha_\ell')
    =
    \beta^{(d)}_{rk}(\alpha_k-\alpha_k')
    \ne0$.
This matches the three-pure-children condition stated informally in \citet{lee2026dde}
(their Conditions~A and~B together are equivalent to Assumption~\ref{ass:pure-separation}).
\end{remark}

\begin{theorem}[Strict identifiability of the canonical DDE copula]
\label{thm:strict-identifiability-copuladde}
Suppose Assumptions~\ref{ass:known-widths-positivity}--\ref{ass:pure-separation} hold,
and suppose the canonical normalization \eqref{eq:canonical-normalization} is imposed.
If two canonical DDE-copula parameters \(\boldsymbol{\Theta}\) and \(\widetilde{\boldsymbol{\Theta}}\) induce the same
copula law,
    $C_{\boldsymbol{\Theta}} = C_{\widetilde{\boldsymbol{\Theta}}},$
then $\boldsymbol{\Theta}\sim_K \widetilde{\boldsymbol{\Theta}}$.
Thus the canonical DDE-copula parameter is identifiable from the copula distribution of
\(\boldsymbol Y\), up to permutation of latent variables within each layer.
\end{theorem}

\subsection{Extended Rank Likelihoods}\label{sec:RL}
Though the DDE copula is identified with arbitrary marginals, estimation remains challenging in part because the marginal distributions $F_{Y_{j}}$ enter the likelihood. Thus, the modeler needs to specify $F_{Y_{j}}$ for each variable, which, in high dimensions, is arduous. In addition, developing scalable estimation algorithms for both sets of parameters -- the DDE dependence structure and the marginal distribution functions -- is increasingly complicated.

We instead separate dependence from marginal modeling using the extended rank likelihood \citep[RL;][]{hoffRL}, which provides an approximation to the full data likelihood under the DDE copula with several advantages. The RL renders $\{F_{Y_j}\}$ as nuisance parameters for estimation of the DDE copula. As a result, this avoids the need to specify each marginal model which improves scalability and extends the DDE copula to accommodate mixed data types \citep{feldman2022bayesian, feldman2024nonparametric}. {Section~\ref{subsec:rl-identifiability} states the corresponding large-sample guarantees for the exact continuous-margin rank likelihood and for the extended rank likelihood viewed as a generalized likelihood with ties.}
Under \eqref{eq::PIT}, $Y_j$ and $Z_j$ are monotone transforms, implying
$
y_{ij} < y_{\ell j} \;\Rightarrow\; z_{ij} < z_{\ell j}.
$
Define the rank-consistent set
\begin{equation}\label{eq:rank}
\mathcal{R}(\boldsymbol Y)
=
\{\boldsymbol Z \in \mathbbm{R}^{n \times J}: y_{ij} < y_{\ell j} \Rightarrow z_{ij} < z_{\ell j}, i = 1, \dots, n, \ j = 1,\dots, J, \}.
\end{equation}
 The likelihood for the DDE copula may be written 
\begin{align}
    &p(\boldsymbol Y, \boldsymbol{A}\mid \boldsymbol \theta, \{F_{Y_{j}}\}_{j=1}^{J}) =     p(\boldsymbol Y, \boldsymbol{A}, \boldsymbol Z \in \mathcal{R}(\boldsymbol Y) \mid \boldsymbol \theta, \{F_{Y_{j}}\}_{j=1}^{J}) \label{eq:equiv}\\
    &= \underbrace{p( \boldsymbol{A},\boldsymbol Z \in \mathcal{R}(\boldsymbol Y) \mid \boldsymbol \theta)}_{\text{extended rank likelihood}}p(\boldsymbol Y \mid \boldsymbol{A},\boldsymbol Z \in \mathcal{R}(\boldsymbol Y),\boldsymbol \theta, \{F_{Y_{j}}\}_{j=1}^{J})\label{eq:decomp}.
\end{align}

The equivalence in \eqref{eq:equiv} is by construction: observing the data $\boldsymbol Y$ implies that the corresponding latent data $\boldsymbol Z$ must fall into the rank-consistent set \eqref{eq:rank}. The first term on the right-hand side of \eqref{eq:decomp} is the extended rank likelihood, which depends only on $\boldsymbol\theta$ and not on $\{F_{Y_j}\}$. Thus, inference for $\boldsymbol\theta$ under the extended rank likelihood $p\{\boldsymbol A, \boldsymbol Z \in \mathcal R(\boldsymbol Y) \mid \boldsymbol\theta\}$ can proceed without specifying marginals.

Intuitively, $\boldsymbol Z$ captures much of the information about the DDE dependence structure. Prior work shows that the RL is sufficient for estimating Gaussian copula parameters \citep{hoffRL} and motivates Bayesian rank-likelihood inference for mixed data \citep{murray2013bayesian}. More recently, \cite{feldman2024nonparametric} studied posterior consistency for rank-likelihood inference in Gaussian mixture copulas. 

\subsection{Rank-likelihood Identifiability and Posterior Consistency}
\label{subsec:rl-identifiability}
Section~\ref{subsec:copuladde-identifiability} establishes identifiability of the
canonical DDE-copula parameter from the population copula law \(C_{\boldsymbol{\Theta}}\).
We now connect this result to rank-likelihood inference. The key distinction is
between continuous margins and margins with ties. With continuous margins, the rank
likelihood is the exact likelihood of the observed rank data, so a Doob martingale argument
can be applied on the quotient parameter space. With discrete or mixed margins, the
extended rank likelihood \citep{hoffRL} remains margin-free, but ties make it a generalized
likelihood rather than the exact likelihood of the observed weak-rank/tie pattern. We
therefore treat the continuous-margin and mixed-margin cases separately.

Let \(\boldsymbol U_i=(U_{i1},\ldots,U_{iJ})\), \(i=1,\ldots,n\), be i.i.d. from
the DDE copula \(C_{\boldsymbol{\Theta}}\), and let
\(Y_{ij}=F_{Y_j}^{-1}(U_{ij})\) for arbitrary nondegenerate margins
\(F=(F_{Y_1},\ldots,F_{Y_J})\). For a realized data array
\(\boldsymbol Y^{(n)}\), define the rank-consistent set on the copula scale by
\[
    \mathcal D_n(\boldsymbol Y^{(n)})
    =
    \Bigl\{
        \boldsymbol u\in(0,1)^{n\times J}:
        y_{ij}<y_{\ell j}\Rightarrow u_{ij}<u_{\ell j},
        \quad
        i,\ell=1,\ldots,n,
        \ j=1,\ldots,J
    \Bigr\}.
    \label{eq:Dn-rank-set}
\]
The marginal extended rank likelihood is
\begin{equation}
        L_n^R(\boldsymbol{\Theta};\boldsymbol Y^{(n)})
    =
    P_{\boldsymbol{\Theta}}\{\boldsymbol U^{(n)}\in
    \mathcal D_n(\boldsymbol Y^{(n)})\}.
    \label{eq:extended-rank-likelihood}
\end{equation}
Because \(F_{Z_j}\) is continuous and strictly increasing under
Assumption~\ref{ass:known-widths-positivity}, this is equivalent to the latent-Gaussian
rank likelihood based on \(\boldsymbol Z^{(n)}\in\mathcal R(\boldsymbol Y^{(n)})\). It
depends on \(\boldsymbol{\Theta}\) only through \(C_{\boldsymbol{\Theta}}\), and not on
the unknown margins.

Let \(\mathcal K=(K^{(1)},\ldots,K^{(D)})\) denote the active layer widths, and let
\(\mathcal T_K\) be the canonical parameter space satisfying the normalization
\eqref{eq:canonical-normalization}. Let
$\Psi_K=\mathcal T_K/\!\sim_K$
be the quotient space under within-layer permutation equivalence. We write
\(\psi=[\boldsymbol{\Theta}]\in\Psi_K\), and equip \(\Psi_K\) with a metric \(d_K\) generating
the quotient topology. Let \(\bar\Pi\) denote the prior induced on \(\Psi_K\). Since
\(L_n^R(\boldsymbol{\Theta};\boldsymbol Y^{(n)})\) is invariant under latent-label
permutations, write \(L_n^R(\psi;\boldsymbol Y^{(n)})\) for its common value on
\(\psi\). For \(A\in\mathcal B(\Psi_K)\), define the quotient-space rank-likelihood
posterior by
\begin{equation}
    \bar\Pi_n^R(A\mid \boldsymbol Y^{(n)})
    =
    \frac{
        \int_A L_n^R(\psi;\boldsymbol Y^{(n)})\,\bar\Pi(d\psi)
    }{
        \int_{\Psi_K}L_n^R(\psi;\boldsymbol Y^{(n)})\,\bar\Pi(d\psi)
    }.
    \label{eq:quotient-rank-posterior}
\end{equation}
The denominator is positive because \(\mathcal D_n(\boldsymbol Y^{(n)})\) is nonempty and
canonical DDE copulas in \(\mathcal T_K\) have positive densities on \((0,1)^J\).

Let
$\mathscr R_n
    =
    \sigma\!\bigl(
        1\{Y_{ij}<Y_{\ell j}\},
        1\{Y_{ij}=Y_{\ell j}\}:
        i,\ell=1,\ldots,n,
        \ j=1,\ldots,J
    \bigr)$
be the sigma-field generated by the coordinatewise weak ranks and ties, and define
$\mathscr R_\infty=\sigma\!\left(\bigcup_{n\ge1}\mathscr R_n\right).$
When all margins are continuous, ties occur with probability zero, and \(\mathscr R_n\)
coincides almost surely with the ordinary coordinatewise rank sigma-field.

\begin{assumption}[Compact canonical parameter space and prior support]
\label{ass:rank-compact-prior}
The active layer widths $\mathcal K$ are fixed. The canonical parameter space \(\mathcal T_K\) is a
compact subset of the normalized parameter space, and every
\(\boldsymbol{\Theta}\in\mathcal T_K\) satisfies
Assumptions~\ref{ass:known-widths-positivity}--\ref{ass:pure-separation}. The quotient
space \(\Psi_K=\mathcal T_K/\!\sim_K\) is equipped with its Borel \(\sigma\)-field. The prior
\(\Pi\) is supported on \(\mathcal T_K\) and induces a prior \(\bar\Pi\) that gives positive
mass to every nonempty open subset of \(\Psi_K\).
\end{assumption}
{The compactness condition is a theoretical truncation to a large nondegenerate subset of the
canonical parameter space; the computational prior in Section~\ref{sec:spec_est} is used as a
practical approximation to this population target.}

\subsubsection{Continuous Margins}

With continuous margins, ties occur with probability zero,
\(\mathcal D_n(\boldsymbol Y^{(n)})\) is the observed-rank cell, and
\(L_n^R\) is the exact rank-data likelihood.


\begin{theorem}[Exact rank-likelihood posterior consistency for continuous margins]
\label{thm:continuous-rank-posterior-consistency}
Suppose Assumption~\ref{ass:rank-compact-prior} holds, each \(F_{Y_j}\) is continuous,
and the strict copula-identifiability conditions of
Theorem~\ref{thm:strict-identifiability-copuladde} hold on \(\mathcal T_K\). Then:
\begin{enumerate}
\item[(i)] For every \(\psi_0=[\boldsymbol{\Theta}_0]\in\Psi_K\), the infinite rank
experiment identifies \(\psi_0\): there exists an \(\mathscr R_\infty\)-measurable map
\(h:\mathscr R_\infty\to\Psi_K\) such that
$h(\mathscr R_\infty)=\psi_0$,
    $P_{\psi_0}^{\infty}$-almost surely.

\item[(ii)] For \(\bar\Pi\)-almost every \(\psi_0\in\Psi_K\), and for every
\(\epsilon>0\), the following holds \(P_{\psi_0}^{\infty}\)-almost surely:
$\bar\Pi_n^R
    \left(
        \psi\in\Psi_K:
        d_K(\psi,\psi_0)<\epsilon
        \,\middle|\,
        \mathscr R_n
    \right)
    \longrightarrow 1$.
\end{enumerate}
\end{theorem}

\paragraph*{Interpretation.}
For continuous margins, Theorem~\ref{thm:continuous-rank-posterior-consistency}
has the status of an ordinary posterior consistency result for an exact likelihood. Because
ties occur with probability zero, \(\mathcal D_n(\boldsymbol Y^{(n)})\) is the observed-rank
cell and \(L_n^R\) is the likelihood of the rank data. Thus the posterior is a regular
conditional posterior on the quotient space, and the Doob argument gives consistency for
\(\bar\Pi\)-almost every true equivalence class once the infinite ranks identify \(\psi_0\).

\subsubsection{Discrete or Mixed Margins}
\label{subsubsec:rl-mixed-margins}

When ties occur, the extended rank likelihood in \cite{hoffRL} remains margin-free, but it is no longer
an exact likelihood for the observed weak-rank/tie pattern. Tied observations impose no
strict ordering constraints in \(\mathcal D_n(\boldsymbol Y^{(n)})\), so the extended-rank
event need not have probability equal to the probability of the observed tie pattern. Thus,
with ties, the extended rank likelihood should be treated as a generalized likelihood rather
than as the probability mass function of the observed weak ranks. Concrete examples,
including the \(n=2,J=1\) tied case and a binary-threshold calculation, are given in the Supplement.

For mixed margins, posterior concentration of the extended-rank posterior requires more than
coarsened-rank identifiability. We use a high-level regularity condition, stated formally as
Assumption S.2 on ``extended-rank contrast consistency'' in the Supplementary Material: under the true
pair \((\psi_0,F)\), the normalized extended-rank log likelihood admits a deterministic
large-sample contrast that is uniquely maximized at \(\psi_0\), and convergence to this
contrast is uniform over the compact quotient space \(\Psi_K\). This is the
generalized-posterior analogue of ordinary likelihood separation. We do not verify this
condition for all mixed-margin DDE copulas in this manuscript; verification would require a
separate analysis of high-dimensional order-constrained probabilities and is left for future
work.

\begin{theorem}[Generalized posterior concentration under the extended rank likelihood]
\label{thm:mixed-rank-generalized}
Suppose Assumption~\ref{ass:rank-compact-prior} holds, and suppose Assumption S.2 on
``extended-rank contrast consistency'' in the Supplementary Material holds for the true pair
\((\psi_0,F)\). Then, for every \(\epsilon>0\),
\[
    \bar\Pi_n^R
    \left(
        \psi\in\Psi_K:
        d_K(\psi,\psi_0)<\epsilon
        \,\middle|\,
        \boldsymbol Y^{(n)}
    \right)
    \longrightarrow 1
\]
almost surely under \(P_{\psi_0,F}^{\infty}\).
\end{theorem}

The proof is a generalized posterior concentration argument. Uniform convergence
transfers the strict separation of the limiting contrast to the finite-sample extended-rank
log likelihood outside any fixed neighborhood of \(\psi_0\). The posterior numerator outside
the neighborhood is then exponentially smaller than a denominator lower bound obtained by
integrating over a positive-prior-mass set near the contrast maximizer. The full proof is given in the Supplement.

\paragraph*{Interpretation.}
Assumption S.2 is stronger than coarsened-rank identifiability. Injectivity of
\(\psi\mapsto Q_{\psi,F}\) says that infinite weak ranks contain enough information to
distinguish \(\psi\); the contrast condition additionally requires the particular
extended-rank generalized likelihood to exploit this information in a uniformly separating
way. The simulation settings below are favorable for this condition, with rich discrete
margins and identifiable DDE graph structure, and the empirical recovery patterns are
consistent with the predicted concentration behavior.

\section{Bayesian Specification and Estimation}\label{sec:spec_est}

We perform Bayesian inference for $\boldsymbol{\theta}$ by targeting the extended rank likelihood posterior
\begin{equation}
p\{\boldsymbol \theta \mid \boldsymbol Z \in \mathcal{R}(\boldsymbol Y)\}
\propto 
\int 
p\{\boldsymbol{A}, \boldsymbol Z \in \mathcal{R}(\boldsymbol Y) \mid \boldsymbol \theta\}
\, p(\boldsymbol \theta) 
\, d \boldsymbol A. 
\label{eq:RLDDEpost}
\end{equation}

Two key challenges arise. First, we seek to infer the layer-specific widths $\{K^{(d)}\}_{d=1}^{D}$ rather than fixing them a priori, as in \cite{lee2026dde}. To this end, we introduce cumulative shrinkage process (CSP; \cite{legramanti2020bayesian}) priors on the elements of each $\boldsymbol B^{(d)}$, which induce sparsity and enable automatic dimension selection within each layer. While CSP priors have been used to infer shallow-layer widths in related models \citep{gu2023bayesian}, our formulation extends this mechanism to all layers of the DDE. In doing so, we enable information sharing using layer-specific activity indicators.

Second, direct evaluation of \eqref{eq:RLDDEpost} is intractable. The rank likelihood involves the integral
$
p\{\boldsymbol{A}, \boldsymbol Z \in \mathcal{R}(\boldsymbol Y) \mid \boldsymbol \theta\}
=
\int_{\mathcal{R}(\boldsymbol Y)} 
p(\boldsymbol A, \boldsymbol Z \mid \boldsymbol \theta)
\, d\boldsymbol Z$,
while the posterior \eqref{eq:RLDDEpost} also marginalizes over the high-dimensional latent $\boldsymbol A$. Although data augmentation \citep{tanner1987calculation} is a natural approach for MCMC-based approaches, the joint latent space $(\boldsymbol Z, \boldsymbol A)$ is high-dimensional and strongly dependent, making such methods computationally inefficient.

To address this, we develop a stochastic optimization procedure inspired by expectation--maximization \citep[EM,][]{dempster1977maximum} and its recent variants for Bayesian dimension selection in linear regression and Gaussian factor models \citep{rovckova2014emvs, rovckova2016fast, rovckova2018spike, li2025sparse}. The resulting algorithm targets maximum a posteriori (MAP) estimation of \eqref{eq:RLDDEpost}, yielding point estimates of the layer-specific weight matrices along with MAP estimates of the effective dimension in each layer.
\subsection{Prior Elicitation}\label{sec:prior}

To complete the Bayesian specification of the DDE copula under extended rank likelihood, we assume prior independence,
$
p(\boldsymbol \theta) 
=
\prod_{d=1}^{D} p(\boldsymbol B^{(d)})
\prod_{j=1}^{J} p(\gamma_{j})
\prod_{k=1}^{K^{(D)}} p(\pi_{k}),
$
with $\gamma_{j} \sim \mbox{InverseGamma}(a,b)$ and flat priors $\pi_{k} \propto 1$.

We initialize a maximal-width DDE for $\boldsymbol Z$ and induce sparsity through the prior on $\boldsymbol B^{(d)}$, which also enables layer-wise dimension selection. Let $K^{(d)}_{\text{max}}$ denote the initialized width of layer $d$, chosen according to the identifiability conditions in Section~\ref{sec:DDEcop} as $K^{(d)}_{\text{max}} = \lfloor K^{(d-1)}/3 \rfloor$. For example, in a two-layer DDE with $J=50$, we take $K^{(1)}_{\text{max}} = 16$ and $K^{(2)}_{\text{max}} = 5$. The goal is then to shrink redundant dimensions through the prior.

For each element $\beta_{jk}^{(d)}$ of $\boldsymbol B^{(d)}$, we impose independent cumulative shrinkage process (CSP) priors. Specifically, for $j = 1,\dots,K^{(d-1)}_{\text{max}}$ and $k = 1,\dots,K^{(d)}_{\text{max}}$,
\begin{align}
\beta^{(d)}_{jk} 
\sim 
\{1- \mathbbm{1}(c^{(d)}_{k} \leq k)\}\,\phi(\lambda^{(d)}_{0k})
&+ 
\mathbbm{1}(c^{(d)}_{k} \leq k)\,\phi(\lambda^{(d)}_{1}), \label{eq:CSP}\\
p(c^{(d)}_{k} = \ell \mid \omega^{(d)}_\ell) &= \omega^{(d)}_{\ell}, \quad \ell = 1,\dots,K^{(d)}_{\text{max}},\\
\omega_{\ell}^{(d)} = v_{\ell}^{(d)} \prod_{m<\ell}(1-v_{m}^{(d)}), &
\quad v_{m}^{(d)} \sim \mbox{Beta}(1,\alpha^{(d)}). \label{eq:stick}
\end{align}
Here, $\phi(\lambda)$ denotes the Laplace density with rate $\lambda$, differing from the Gaussian formulation in \cite{legramanti2020bayesian}. This choice yields exact zeros under MAP estimation and facilitates dimension selection; See Section \ref{sec:EM}. We assign hyperpriors $\lambda_{0k}^{(d)} \sim \mbox{Gamma}(a^{(d)}, b^{(d)})$ and fix $\lambda_{1}^{(d)} \ll \mathbbm{E}[\lambda_{0k}^{(d)}]$ to distinguish spike and slab components. The parameter $\lambda_{1}^{(d)}$ may be layer-specific; see the supplement for tuning strategies.

Marginalizing over $c_{k}^{(d)}$, the prior for $\beta_{jk}^{(d)}$ is a spike-and-slab mixture with ordered shrinkage: the probability of assignment to the spike increases with $k$. Consequently, the effective layer width is
$K^{(d)*} = \sum_{k=1}^{K^{(d)}_{\text{max}}} \mathbbm{1}(c_{k}^{(d)} > k),$
since $\mathbbm{1}(c_{k}^{(d)} > k)$ indicates whether the $k$th column of $\boldsymbol B^{(d)}$ is drawn from the slab. When $c_{k}^{(d)} \leq k$, the corresponding column is shrunk toward zero and the associated latent unit is redundant. Thus, the CSP prior induces a prior for the unknown layer-specific dimension $K^{(d)^{*}}$ through the spike and slab indicator variables $c_{k}^{(d)}$. Finally, we fix $\alpha^{(d)} = K^{(d)}_{\text{max}}$, motivated by that in the infinite-width limit the expected number of active components satisfies $\sum_{\ell=1}^{\infty} \mathbbm{E}[\mathbbm{1}(c_{k}^{(d)} > k)] = \alpha^{(d)}$ \citep{legramanti2020bayesian}.

\subsection{MAP Estimation via Coordinate Ascent Monte Carlo EM}\label{sec:EM}
 
 We proceed to outline our expectation–maximization (EM; \cite{dempster1977maximum}) algorithm targeting maximum a posteriori (MAP) estimation of \eqref{eq:RLDDEpost}. Under the CSP prior \eqref{eq:CSP}--\eqref{eq:stick}, the parameter set $\boldsymbol \theta$ is augmented to include the layer-specific dimensions $\{K^{(d)*}\}_{d=1}^{D}$.

Let $\boldsymbol \lambda_{0}^{(d)} = \{\lambda^{(d)}_{0k}\}_{k=1}^{K^{(d)}_{\text{max}}}$ and $\boldsymbol v^{(d)} = \{v^{(d)}_{k}\}_{k=1}^{K^{(d)}_{\text{max}}}$. We partition the full parameter set into model parameters $\boldsymbol{\theta} = (\{\boldsymbol B^{(d)}\}_{d=1}^{D}, \{K^{(d)*}\}_{d=1}^{D}, \{\gamma_j\}_{j=1}^{J}, \boldsymbol \pi)$, latent variables $\boldsymbol \theta_{\text{lat}} = (\{\boldsymbol A^{(d)}\}_{d=1}^{D}, \boldsymbol Z)$, and hyperparameters $\boldsymbol \theta_{\text{hyp}} = (\{\boldsymbol v^{(d)}\}_{d=1}^{D}, \{\boldsymbol \lambda_{0}^{(d)}\}_{d=1}^{D})$. Given an initialization $\boldsymbol \theta^{0}$, the EM update at iteration $t+1$ is
\begin{align}
\boldsymbol \theta^{t+1}
&=
\arg\max_{\boldsymbol \theta}
\mathbb{E}_{\boldsymbol \theta_{\text{lat}}, \boldsymbol \theta_{\text{hyp}} \mid \boldsymbol Z \in \mathcal R(\boldsymbol Y), \boldsymbol \theta^t}
\big[\log p\{\boldsymbol \theta, \boldsymbol \theta_{\text{lat}}, \boldsymbol \theta_{\text{hyp}} \mid \boldsymbol Z \in \mathcal R(\boldsymbol Y)\}\big]
\nonumber\\
&=
\arg\max_{\boldsymbol \theta} Q(\boldsymbol \theta \mid \boldsymbol \theta^t).
\label{eq:Qfn_tight_mid2}
\end{align}

The hierarchical structure of the DDE and prior independence across layers yields a layer-wise decomposition of \eqref{eq:Qfn_tight_mid2}
\begin{equation}\label{eq:Qfndecomp}
Q(\boldsymbol \theta \mid \boldsymbol \theta^{t})
=
\sum_{d=1}^{D} Q(\boldsymbol \theta^{(d)} \mid \boldsymbol \theta^{t}).
\end{equation}
This enables layer-wise updates within the optimization routine, reducing the M-step to a collection of independent optimization problems. In particular, for $2 \leq d < D$, each $\boldsymbol B^{(d)}$ is updated via $K^{(d-1)}_{\text{max}}$ independent regression problems, while the first layer is optimized using an expectation–conditional maximization step \citep{meng1993maximum}, alternating updates of $\boldsymbol B^{(1)}$ and $\boldsymbol \gamma$.  At the top layer, $\boldsymbol \theta^{(D)} = \boldsymbol{\pi}$. The effective layer widths $\{K^{(d)*}\}$ are determined as a function of the updated spike indicators $\{c_{k}^{(d)}\}$ (Algorithm \ref{alg:mcem}).

The main computational challenge is evaluating expectations under $p\{\boldsymbol \theta_{\text{lat}}, \boldsymbol \theta_{\text{hyp}} \mid \boldsymbol Z \in \mathcal D(\boldsymbol Y), \boldsymbol \theta^t\}$,
which involves (intractable) integration over the rank-constrained set $\mathcal D(\boldsymbol Y)$ and summation over exponentially many configurations of $\{\boldsymbol A^{(d)}\}$.

We address this using Monte Carlo EM \citep{booth1999maximizing}. At each iteration, we draw $C$ samples $(\boldsymbol \theta_{\text{lat}}^{(c)}, \boldsymbol \theta_{\text{hyp}}^{(c)})$ from the conditional posterior and approximate
$
Q(\boldsymbol \theta \mid \boldsymbol \theta^t)
\approx
\frac{1}{C}
\sum_{c=1}^{C}
\log p\{\boldsymbol \theta, \boldsymbol \theta_{\text{lat}}^{(c)}, \boldsymbol \theta_{\text{hyp}}^{(c)} \mid \boldsymbol Z \in \mathcal R(\boldsymbol Y)\} = \hat{Q}(\boldsymbol \theta \mid \boldsymbol{\theta}^t).
$
Following \cite{lee2026dde}, we set $C=1$ and observe strong empirical performance.

Sampling from 
$p\{\boldsymbol \theta_{\text{lat}}, \boldsymbol \theta_{\text{hyp}} \mid \boldsymbol Z \in \mathcal D(\boldsymbol Y), \boldsymbol \theta^t\}$
is carried out via Gibbs updates for $\boldsymbol Z$, $\{\boldsymbol A^{(d)}\}$, $\{\boldsymbol \lambda_{0}^{(d)}\}$, and $\{\boldsymbol v^{(d)}\}$. The latter two remain conditionally conjugate under the CSP prior and are deferred to the supplement, while we detail the updates for $\boldsymbol Z$ and $\{\boldsymbol A^{(d)}\}$ below.

\paragraph*{Extended Rank Likelihood Re-sampling}
For $\boldsymbol Z \sim p\{\boldsymbol Z \mid \boldsymbol{Z} \in \mathcal{R}(\boldsymbol Y), \boldsymbol \theta^t, \text{--}\}$, the RL implies
\begin{equation}
    p\{\boldsymbol Z_i \mid \boldsymbol Z \in \mathcal{R}(\boldsymbol{Y}), \theta^{(t)},\text{--}\}\sim \prod_{j=1}^{J} N(\beta^{(1)}_{j,0} + \sum_{k=1}^{K^{(1)}_{\text{max}}} \beta_{j,k}^{(1)} \boldsymbol \alpha^{(1)}_{i,k}, \gamma_{j})\mathbbm{1}\{\boldsymbol Z \in \mathcal{R}(\boldsymbol Y)\}\label{eq:RLsamp}.
\end{equation}
That is, each  component of the vector $\boldsymbol Z_i$ is normally distributed subject to the rank consistency between the entire latent data matrix $\boldsymbol Z$ and observed data matrix $\boldsymbol Y$. This implies that each $Z_{ij}$ may be sampled from a truncated normal distribution, with lower and upper bounds determined by the largest $Z_{\ell j}$ such that $Y_{\ell j} < Y_{ij}$ and smallest $Z_{r j}$ such that $Y_{r j} > Y_{ij}$, respectively. Due to conditional independence, the joint distribution of $\boldsymbol Z_i$ is simply the product of similarly formed truncated normals. Thus, sampling from \eqref{eq:RLsamp} is simple, requiring a sequence of univariate truncated normal sampling steps, which are parallelizable across $J$. We summarize this procedure in Algorithm~\ref{alg:rank_aug}.
{
\setlength{\itemsep}{0pt}        
\setlength{\parsep}{0pt}
\setlength{\parskip}{0pt}
\renewcommand{\baselinestretch}{0.9}\normalsize
\begin{algorithm}[h]
\caption{One Gibbs sweep of RL data augmentation under the DDE copula}
\label{alg:rank_aug}
\begin{algorithmic}
\REQUIRE $\boldsymbol Y$, current $\boldsymbol Z\in\mathcal R(\boldsymbol Y)$,
$\boldsymbol A^{(1)},\boldsymbol B^{(1)},\boldsymbol\gamma$

\STATE $\boldsymbol\mu \leftarrow \boldsymbol A^{(1)}\boldsymbol B^{(1)}$

\FOR{$j=1,\dots,J$}
\FOR{$i=1,\dots,N$}

\STATE{Compute truncation bounds} 

\[L_{ij}=\max\{Z_{i'j}:Y_{i'j}<Y_{ij}\},\qquad
U_{ij}=\min\{Z_{i'j}:Y_{i'j}>Y_{ij}\}\]

\STATE{Sample}
\[
Z_{ij} \sim \mathrm{TN}(\mu_{ij},\gamma_j;L_{ij},U_{ij})
\]

where $\mathrm{TN}(\mu, \sigma^{2}, a,b)$ is the normal density with mean $\mu$ and variance $\sigma^{2}$ truncated to the interval (a,b).

\ENDFOR
\ENDFOR

\end{algorithmic}
\end{algorithm}
}
 \paragraph*{Binary Latent Variables}
Sampling from $p[\{\boldsymbol A^{(d)}\}_{d=1}^{D} \mid \boldsymbol \theta^t,\boldsymbol{Z} \in \mathcal{R}(\boldsymbol Y), \text{---}]$ is computationally challenging due to the exponential size of the joint state space. Rather than performing a sequential Gibbs sweep over individual coordinates, we employ an approximation based on a product-form (mean-field) representation of the full conditional. For each observation $i$, we approximate the joint conditional distribution by independent Bernoulli updates across layers and coordinates,
\begin{equation}\label{eq:MFVI}
    p(\boldsymbol A_i\mid \boldsymbol{Z}^t_i, \boldsymbol \theta^t, \text{--}) \approx \prod_{d=1}^{D} \prod_{k=1}^{K^{(d)}} \text{Bernoulli}(\pi_{ik}^{(d)}),
\end{equation}

where $\boldsymbol A_i = (\boldsymbol A_i^{(1)}, \dots, \boldsymbol A_i^{(D)})$. When used within the EM algorithm, we find that sampling from this approximation enables efficient exploration of the posterior relative to a true Gibbs step, resulting in a significantly more accurate estimation of the layer-specific dimensions $\{K^{(d)}\}_{d=1}^{D}$ (Section \ref{sec:sim}).

At iteration $t$, the update proceeds by computing local log-odds for each latent variable while holding all other coordinates fixed at their current values. Specifically, for each $d \in \{1,\dots, D\}$ and $k \in \{1, \dots, K^{(d)}_{\text{max}}\}$ we evaluate
$
\Delta_{ik}^{(d),t} =
\log p(A_{ik}^{(d)} = 1 \mid \boldsymbol A^{-(d), t}_{i,-k}, \boldsymbol Z^t_i, \boldsymbol \theta ^t, \text{--})
-
\log p(A_{ik}^{(d)} = 0 \mid \boldsymbol A^{-(d), t}_{i,-k}, \boldsymbol Z^t_i, \boldsymbol \theta ^t, \text{--}),
$
and set $\pi_{ik}^{(d), t} = \operatorname{logistic}\!\big(\Delta_{ik}^{(d),t}\big)$. Here, $\boldsymbol A^{-(d), t}_{i,-k}$ are the binary latent variables without $A_{ik}^{(d)}$.
The quantities $\Delta_{ik}^{(d),t}$ admit simple closed forms due to the conditional independence structure of the DDE, which may be found in the supplementary materials. All log-odds are computed using the same previous configuration $\boldsymbol A_i^{t-1}$, allowing updates to be performed in parallel across coordinates and layers. Given $\{\pi_{ik}^{(d),t}\}$, we then sample independently: $A_{ik}^{(d),t} \sim \text{Bernoulli}(\pi_{ik}^{(d),t}), \ \text{for all}\ d,k$.







This update can be viewed as an approximate Gibbs step: rather than sequentially sampling from exact full conditionals, we construct a product-form approximation and draw all coordinates simultaneously. The resulting update resembles a mean-field variational step \citep{blei2017variational}, but instead of optimizing the Bernoulli parameters governing each $A_{ik}^{(d)}$, we sample directly from the approximate family. Although this procedure doesn't produce exact samples from the target distribution, it yields substantial computational gains relative to a full Gibbs sampling step and performs well in high-dimensional settings; see Section \ref{sec:sim} for detailed comparisons.
{
\setlength{\itemsep}{0pt}        
\setlength{\parsep}{0pt}
\setlength{\parskip}{0pt}
\renewcommand{\baselinestretch}{0.9}\normalsize
\begin{algorithm}[ht]
\caption{Coordinate Ascent Monte Carlo EM for the DDE copula ($C=1$)}
\label{alg:mcem}
\begin{algorithmic}

\REQUIRE $\boldsymbol Y$, $\mathcal R(\boldsymbol Y)$, initial $\boldsymbol\theta^{0}$

\FOR{$t = 0,1,2,\dots$ until convergence}

\STATE \textbf{E-step for $\boldsymbol{\theta_{\text{lat}}}$ and $\boldsymbol{\theta_{\text{hyp}}}$: Sample from the following full conditionals}:

\STATE \hspace{1em} \textbullet\ $\boldsymbol Z^{t} \sim p\{\boldsymbol Z \mid \boldsymbol Z \in \mathcal R(\boldsymbol Y), \boldsymbol \theta^{t}, \text{--}\}$ \quad (Alg.~\ref{alg:rank_aug})

\STATE \hspace{1em} \textbullet\ $\{\boldsymbol A^{(d),t}\}_{d=1}^{D} \approx p[\{\boldsymbol A^{(d)}\} \mid \boldsymbol Z^{t}, \boldsymbol \theta^{t}, \text{--}]$ \quad (Eq.~\ref{eq:MFVI})

\STATE \hspace{1em} \textbullet\ $\{\boldsymbol \lambda_{0}^{(d),t}\}_{d=1}^{D} \sim p(\cdot \mid \boldsymbol Z^{t}, \boldsymbol \theta^{t}) \ (\text{CSP conjugate updates})$

\STATE \hspace{1em} \textbullet\ $\{\boldsymbol v^{(d),t}\}_{d=1}^{D} \sim p(\cdot \mid \boldsymbol Z^{t}, \boldsymbol \theta^{t}) \ (\text{CSP conjugate updates})$
\vspace{1em}
\STATE \hspace{1em} \textbf{Set}
$\widehat Q(\boldsymbol \theta \mid \boldsymbol \theta^{t})$ 
via \eqref{eq:Qfndecomp}
\vspace{1em}
\STATE \textbf{Conditional Maximization-step}

\STATE \hspace{1em} \textbullet\ $\boldsymbol \pi^{t+1}
= \arg\max_{\boldsymbol \pi} \widehat Q(\boldsymbol \pi,\text{--})$

\STATE \hspace{1em} \textbullet\ For $d = D-1,\dots,2$:
\[
\boldsymbol B^{(d),t+1}
=
\arg\max_{\boldsymbol B^{(d)}} \widehat Q(\boldsymbol B^{(d)}\mid \text{--})
\]

\STATE \hspace{1em} \textbullet\ $\boldsymbol B^{(1),t+1}
=
\arg\max_{\boldsymbol B^{(1)}} \widehat Q(\boldsymbol B^{(1)}, \boldsymbol \gamma^{t}\mid\text{--})$

\STATE \hspace{1em} \textbullet\ $\boldsymbol \gamma^{t+1}
=
\arg\max_{\boldsymbol \gamma} \widehat{Q}(\boldsymbol B^{(1),t+1}, \boldsymbol \gamma\mid \text{--})$ 

\STATE \hspace{1em} \textbullet \ $c_k^{(d),t+1} = \arg\max_{c_k^{(d)}} \widehat Q(c_k^{(d)} \mid \text{--})$, 
$K^{(d)*,t+1}
=
\sum_{k=1}^{K_{\max}^{(d)}} \mathbbm{1}(c_k^{(d),t+1} > k), 
\ d=1,\dots,D, \ k = 1,\dots, K^{(d)}_{\text{max}}$

\ENDFOR

\end{algorithmic}
\end{algorithm}
}

The full coordinate ascent Monte Carlo EM algorithm for MAP estimation of the $D$-layer DDE copula is summarized in Algorithm~\ref{alg:mcem}, with details provided in the supplement. The Monte Carlo E-step also simplifies the maximization step. When $C=1$, the decomposition of $Q(\boldsymbol{\theta} \mid \boldsymbol{\theta}^t)$ in \eqref{eq:Qfndecomp}, together with conditional independence across nodes within each layer, reduces the M-step to a collection of independent, parallelizable weighted-$\ell_1$-regularized logistic or linear regressions. For the $j$th row in layer $d$, the penalty is
$
\sum_{k=1}^{K^{(d)}_{\text{max}}}
\Big[
(1-\mathbbm{1}\{c_{k}^{(d),t+1} \leq k\}) \lambda^{(d),t+1}_{0k}
+
\mathbbm{1}\{c_{k}^{(d),t+1} \leq k\} \lambda^{(d)}_{1}
\Big]
\lvert \beta_{jk} \rvert,
$
arising from the Laplace mixture specification of the CSP prior.  The current estimate of $c_{k}^{(d)}$ determines whether the $k$th coefficient is penalized by the spike rate $\lambda^{(d)}_{1}$ or the slab rate $\lambda^{(d),t+1}_{0k}$.

This thresholding scheme differs from \cite{li2025sparse}, which employs a soft weighting of spike and slab penalties in factor models. This hard-thresholding also induces hierarchical gating: if $c_{k}^{(d),t+1} \leq k$ in layer $d$, we deterministically set the $k$th row of $\boldsymbol B^{(d+1),t+1}$ to zero. Empirically, this improves convergence and estimation accuracy.

Despite the efficiency gains from the approximate Gibbs update \eqref{eq:MFVI}, the EM algorithm may exhibit sensitivity to local modes in the high-dimensional posterior. To mitigate this, we incorporate deterministic annealing \citep{rovckova2014emvs} by introducing a temperature parameter $\tau \in (0,1]$ and targeting a tempered posterior proportional to
$p\{\boldsymbol \theta, \boldsymbol \theta_{\text{lat}}, \boldsymbol \theta_{\text{hyp}} \mid \boldsymbol Z \in \mathcal{R}(\boldsymbol Y)\}^{\tau}$.
This modification induces simple, layer-wise adjustments to both the E-step and M-step; see the supplement for more details.

\subsection{Spectral Initialization of Latent Variables and Parameters}
\label{sec:init}
 Initialization of both $\boldsymbol \theta$ and $\boldsymbol \theta_{\text{hyp}}$ is paramount to accurate estimation of the DDE copula. Although the rank likelihood enforces marginal rank-consistency between $\boldsymbol Y$ and $\boldsymbol Z$, the marginal distributions of $Y_j$ and $Z_j$ may differ substantially. In particular, $Y_j$ may be discrete, skewed, or have zero-inflation, while the corresponding latent variable $Z_j$ is typically a factorial mixture of Gaussians induced by the binary latent layers. Moreover, dependence in $\boldsymbol Z$ arises from a mixture of Gaussian components that is subsequently warped by the marginal transformations $F_{Y_j}$, making direct initialization difficult.

{
\setlength{\itemsep}{0pt}        
\setlength{\parsep}{0pt}
\setlength{\parskip}{0pt}
\renewcommand{\baselinestretch}{0.9}\normalsize
 \begin{figure}[h]
     \centering
\includegraphics[width=0.38\linewidth, keepaspectratio]{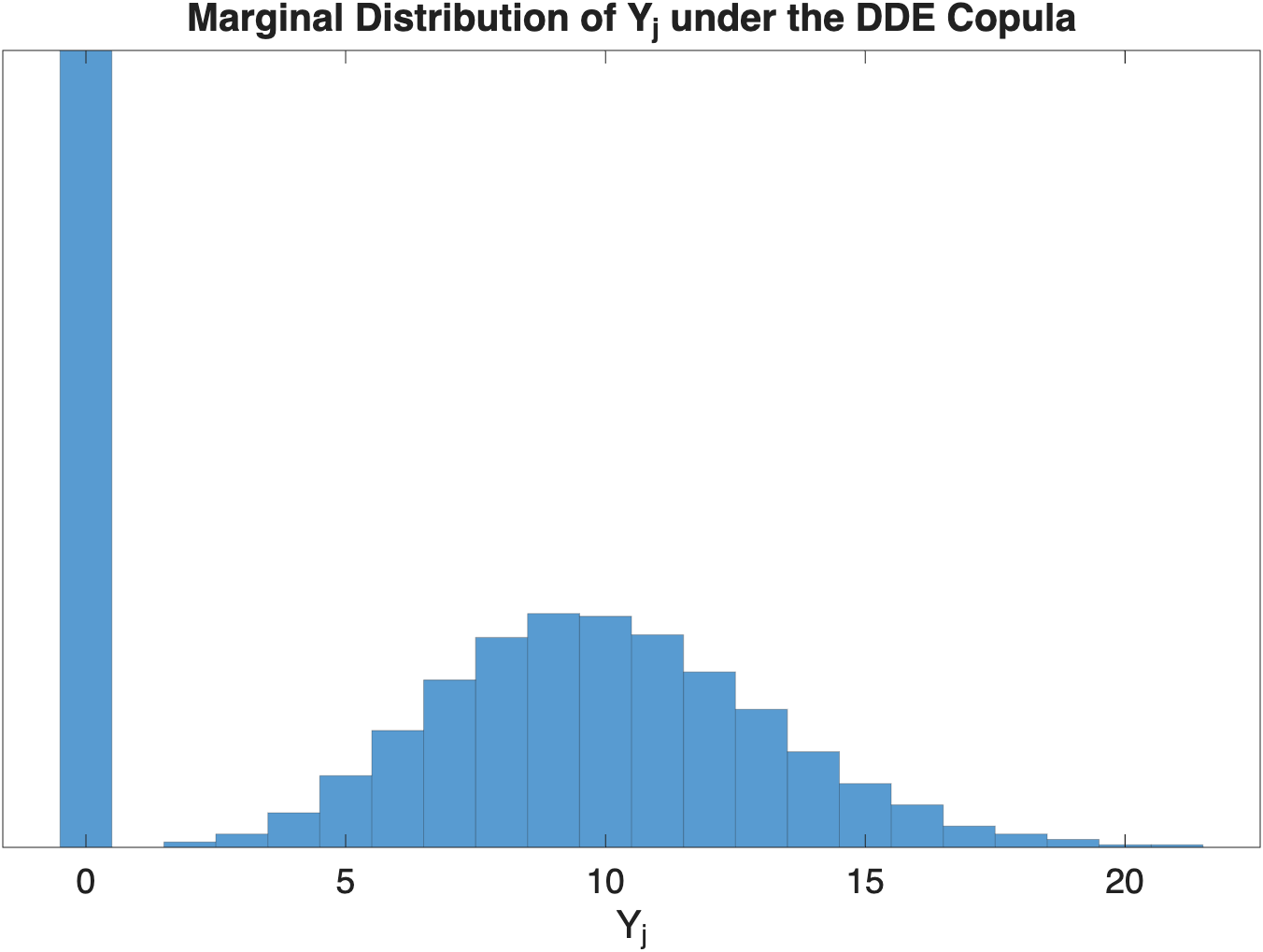}
\includegraphics[width=0.42\linewidth, keepaspectratio]{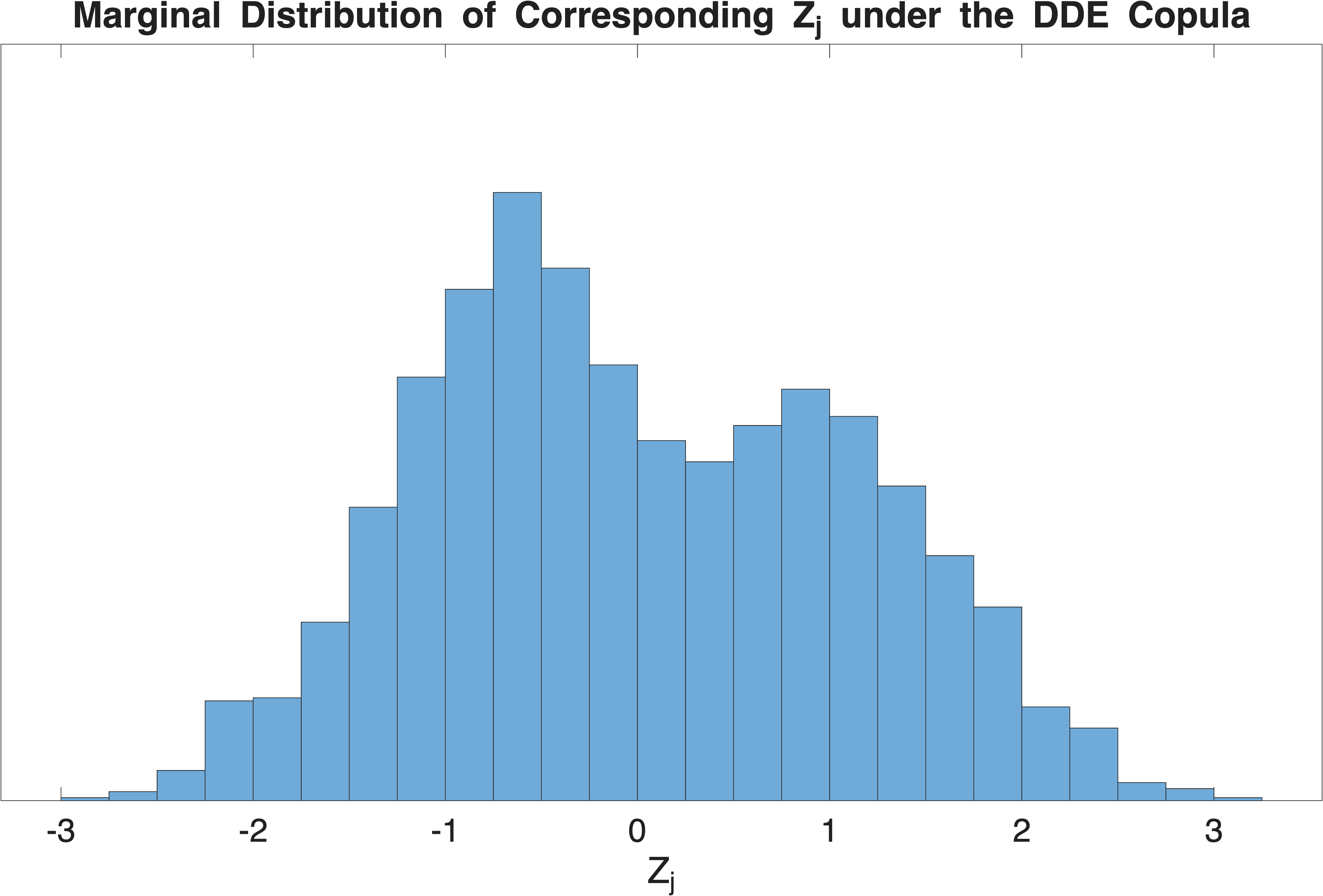}
\includegraphics[width=0.42\linewidth, keepaspectratio]{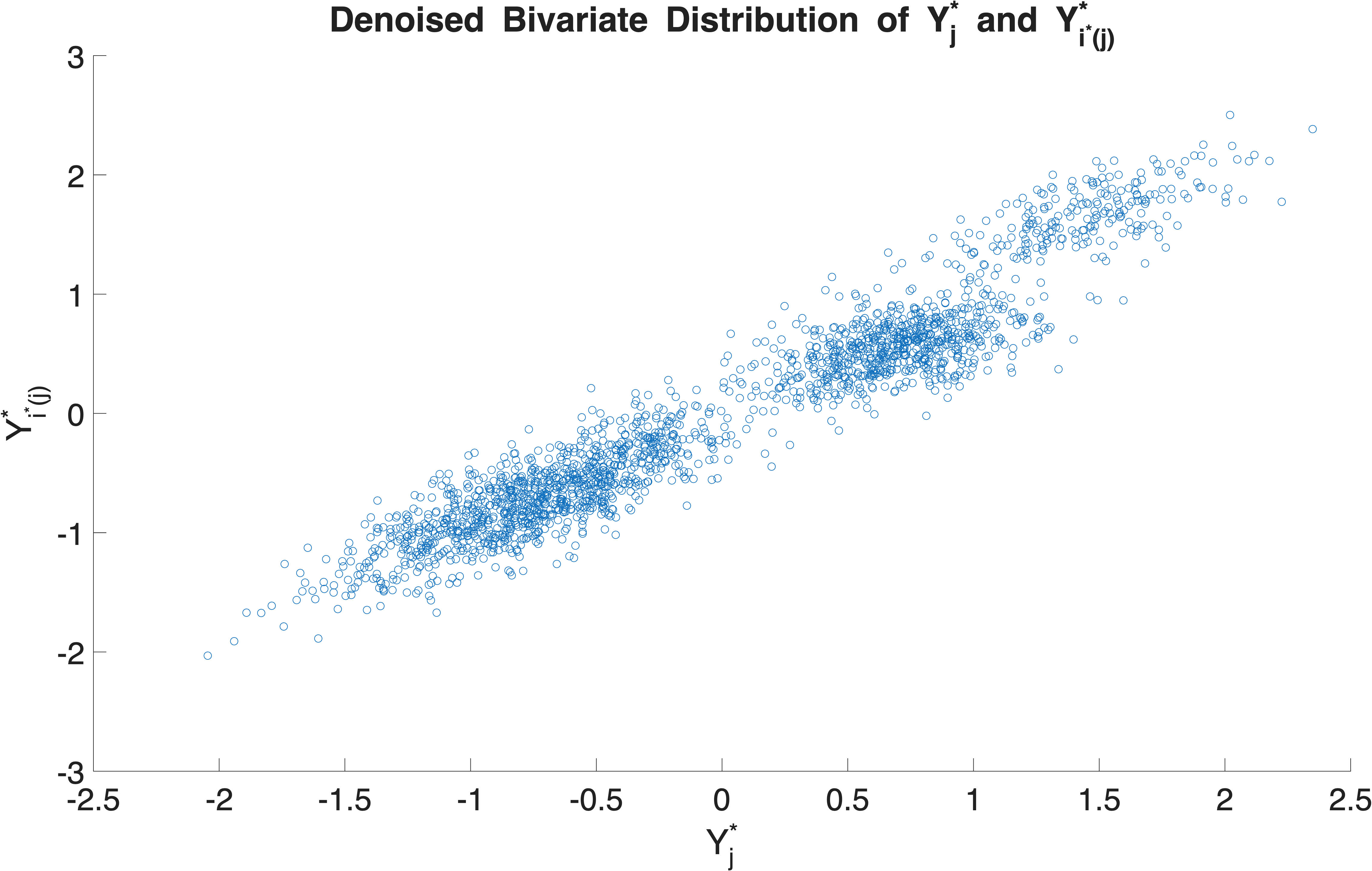}
\includegraphics[width=0.42\linewidth, keepaspectratio]{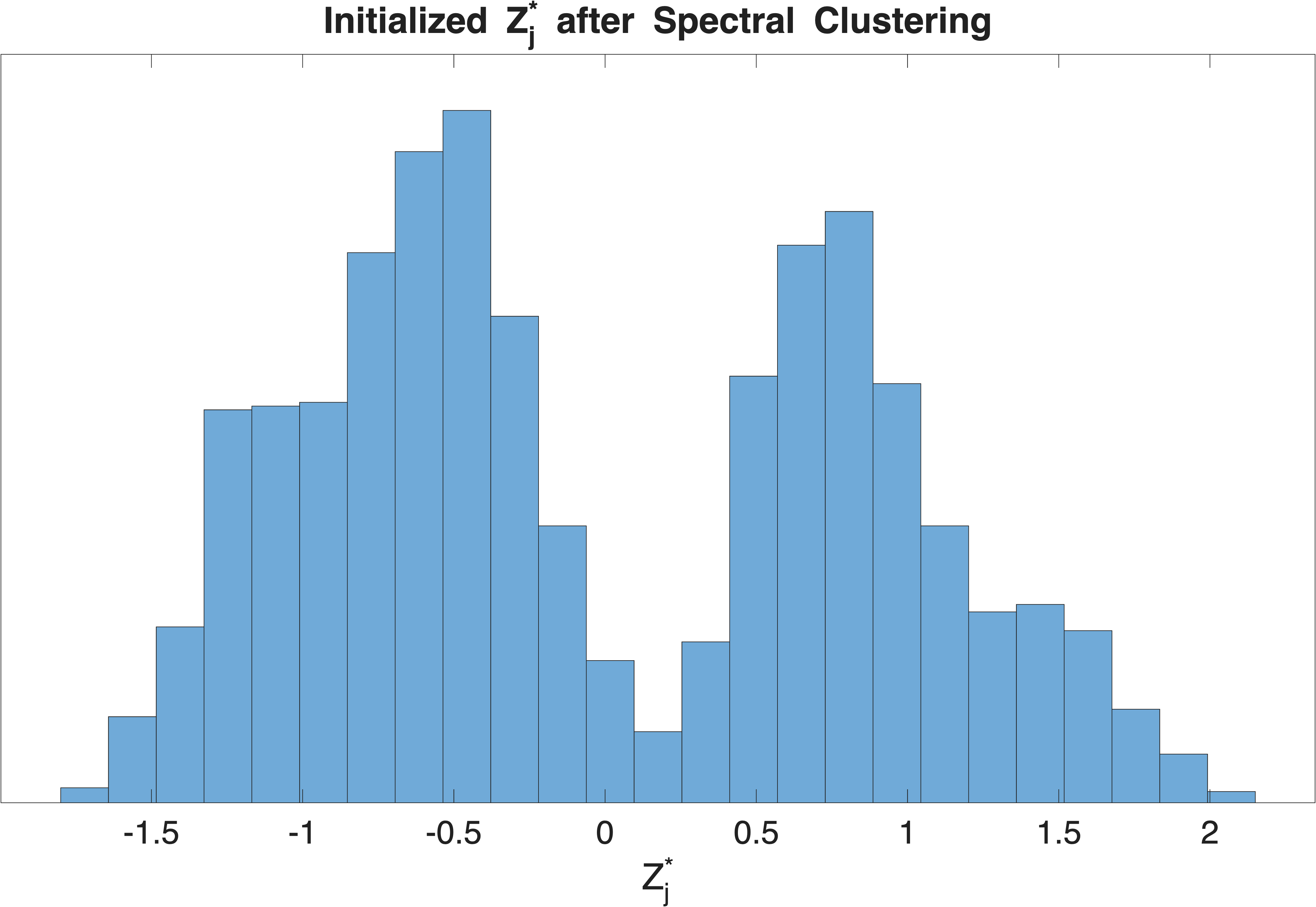}
\caption{\small Top row: Simulated observed $X_j$ (left) and corresponding $Z_j$ (right) under the DDE copula. The marginal distribution of $X_{j}$ is zero-inflated, while $Z_{j}$ is clearly multi-modal. Bottom left: the bivariate distribution of denoised $Y^{*}_{j}$ and $Y^{*}_{i^{*}(j)}$ where $i^{*}(j)$ is the smallest entry of $\boldsymbol{D}$.  Non-overlapping clusters emerge such that when we simulate data from a fitted K-means clustering applied to $(Y_{j}^{*},Y^{*}_{i^{*}(j)})$, the resulting marginal distribution of the simulated data is multi-modal (bottom right).}
     \label{fig:specinit}
 \end{figure}
 }

To address this, we propose a spectral initialization procedure that exploits local dependence structure in $\boldsymbol Y$ to recover latent clustering behavior. The key heuristic is that under the DDE copula,
$Z_j = G_{Z_j}^{-1}\{F_{Y_j}(Y_j)\}$,
so that $Z_j$ is a monotone transformation of $Y_j$. We approximate this transformation by the identity and instead aim to recover latent structure through pairwise relationships in $\boldsymbol Y$.

We begin by computing the singular value decomposition $\boldsymbol Y = \boldsymbol U \boldsymbol \Sigma \boldsymbol V^\top$ and defining a distance matrix $D_{ij} = \|\boldsymbol V_i - \boldsymbol V_j\|_2^2$, where $\boldsymbol V_j$ denotes the $j$th row of $\boldsymbol V$, the right singular vectors of $\boldsymbol Y$. This representation captures dependence between variables through their spectral embeddings. To reduce noise and mitigate issues arising from discrete observations, we construct a low-rank approximation $\boldsymbol Y^* = \boldsymbol U_k \boldsymbol \Sigma_k \boldsymbol V_k^\top$ using the leading $k$ singular values. We select $k$ as the smallest integer satisfying
$
\frac{\sum_{\ell=1}^{k} \Sigma_{\ell\ell}^{2}}{\sum_{\ell=1}^{\min(N,J)} \Sigma_{\ell\ell}^{2}} \geq 0.8,
$
where $\Sigma_{\ell\ell}$ denotes the $\ell$th singular value in the diagonal matrix $\boldsymbol \Sigma$.

For each variable $j$, we identify its nearest neighbor $i^*(j) = \arg\min_{i \neq j} D_{ij}$ and consider the bivariate data $(Y^*_{i^*(j)}, Y^*_j)$. Empirically, these pairs often exhibit clustering structure that reflects shared latent factors; see the bottom-left panel of Figure \ref{fig:specinit}. We apply $k$-means clustering to this bivariate data, selecting the number of clusters via the elbow method, and use the resulting cluster centers and weights to define an isotropic Gaussian mixture model. Drawing $n$ i.i.d.\ samples $(Z^{*}_{i^*(j)},Z^{*}_j)$ from this fitted model yields an approximation to the latent distribution ($Z_{i(j)}, Z_{j})$. Finally, we enforce rank-consistency by matching order statistics, setting $Z_{(r),j} = Z^{*}_{(r),j}$ for $r=1,\dots,n$. This process is outlined in Algorithm 
\ref{alg:spectral_init}.

 {
\setlength{\itemsep}{0pt}        
\setlength{\parsep}{0pt}
\setlength{\parskip}{0pt}
\renewcommand{\baselinestretch}{0.9}\normalsize

\begin{algorithm}[ht]
\caption{Initialization of latent variables rank-consistent $\boldsymbol Z$ using Spectral Clustering}
\label{alg:spectral_init}
\begin{algorithmic}

\REQUIRE Observed data $\boldsymbol Y \in \mathbbm{R}^{N \times J}$, rank $k$

\STATE \textbf{Spectral embedding}

\STATE \hspace{1em} \textbullet\ Compute SVD:
\[
\boldsymbol Y = \boldsymbol U \boldsymbol \Sigma \boldsymbol V^\top
\]

\STATE \hspace{1em} \textbullet\ Form distance matrix:
\[
D_{ij} = \|\boldsymbol V_i - \boldsymbol V_j\|_2^2
\]

\STATE \hspace{1em} \textbullet\ De-noise:
\[
\boldsymbol Y^* = \boldsymbol U_k \boldsymbol \Sigma_k \boldsymbol V_k^\top
\]

\STATE \textbf{Pairwise clustering and latent simulation}

\FOR{$j = 1,\dots,J$}

    \STATE \hspace{1em} \textbullet\ Nearest neighbor:
    \[
    i^*(j) = \arg\min_{i \neq j} D_{ij}
    \]

    \STATE \hspace{1em} \textbullet\ Fit $k$-means to $(Y^*_{i^*(j)}, Y^*_j)$

    \STATE \hspace{1em} \textbullet\ Simulate $(Z^{*}_{i^*(j)},  Z^{*}_j)$ from fitted isotropic Gaussian mixture

    \STATE \hspace{1em} \textbullet\ Rank matching:
    \[
    Z_{(r),j} \leftarrow Z^{*}_{(r),j}, \quad r=1,\dots,N
    \]

\ENDFOR

\RETURN Initialized latent matrix $\boldsymbol Z$

\end{algorithmic}
\end{algorithm}}

This procedure leverages the fact that pairs $(Z_i, Z_j)$ are mixtures of low-dimensional Gaussian components induced by shared latent structure. The spectral embedding identifies strongly dependent variables, while the de-noised representation reveals clustering patterns that are otherwise obscured on the observed scale. Matching order statistics ensures compatibility with the rank likelihood, producing an initialization that respects the marginal ordering constraints. 

Once we have initialized $\boldsymbol Z$, we must then initialize $\boldsymbol{\theta}$ and $\{\boldsymbol{A}^{(d)}\}_{d=1}^{D}$. Because we are effectively estimating a Gaussian DDE on rank-consistent $\boldsymbol Z$,  we  may proceed using the layer-wise Double-SVD initialization strategy in \cite{lee2026dde}. The basic idea is to denoise $\boldsymbol Z$ with a singular value decomposition, which provides initial estimates of $(\hat{\boldsymbol{B}}^{(1)}, \hat{\boldsymbol{A}}^{(1)})$. Then, the algorithm proceeds recursively, treating initialized $\hat{\boldsymbol A}^{(d)}$ as observed data for a denoising process that enables initialization of $(\boldsymbol B^{(d-1)}, \boldsymbol A^{(d-1)})$. Complete details on this process are provided in the Supplementary Material. 

Importantly, the initialization produced by Algorithm~\ref{alg:spectral_init} need not satisfy the normalization $\mathbbm{E}(Z_j)=0$ and $\mathrm{Var}(Z_j)=1$ in Section \ref{sec:DDEcop}. This is inconsequential for estimation, as the DDE copula depends only on the induced rank structure of $\boldsymbol Z$ and is invariant to marginal affine transformations. In practice, we therefore fit the DDE model directly to the unnormalized $\boldsymbol Z$ obtained from the spectral procedure. After fitting, we impose the identification constraint by re-centering and re-scaling each coordinate of $\boldsymbol Z$, with a corresponding adjustment to the first-layer parameters $(\boldsymbol B^{(1)}, \boldsymbol \gamma)$ to preserve the likelihood. This post hoc normalization yields an identified representation without affecting the estimated dependence structure. 

\section{Simulation: Dimension and Parameter Estimation}\label{sec:sim}

We conduct a simulation study to evaluate the performance of the proposed Bayesian extended rank likelihood DDE copula model under layer-wise CSP priors in recovering latent structure and estimating model parameters. The primary goals are to assess: (i) recovery of the latent dimensions $\{K^{(d)}\}_{d=1}^{(D)}$, and (ii) accuracy of the estimated loading matrices $\{\boldsymbol B^{(d)}\}_{d=1}^{D}$.

{
\setlength{\itemsep}{0pt}        
\setlength{\parsep}{0pt}
\setlength{\parskip}{0pt}
\renewcommand{\baselinestretch}{0.9}\normalsize

\begin{figure}
    \centering
    \includegraphics[width=0.45\linewidth]{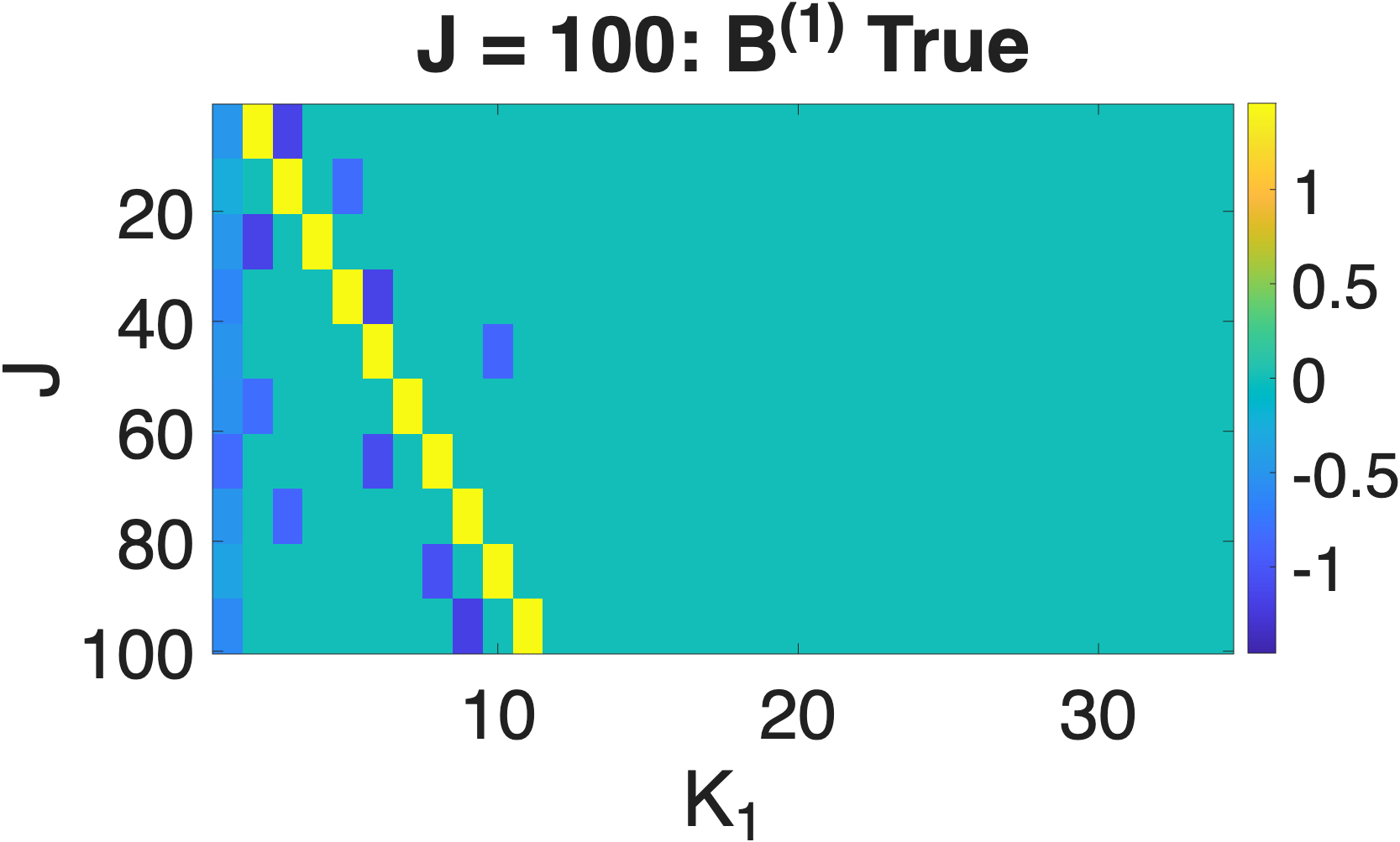}
        \includegraphics[width=0.45\linewidth]{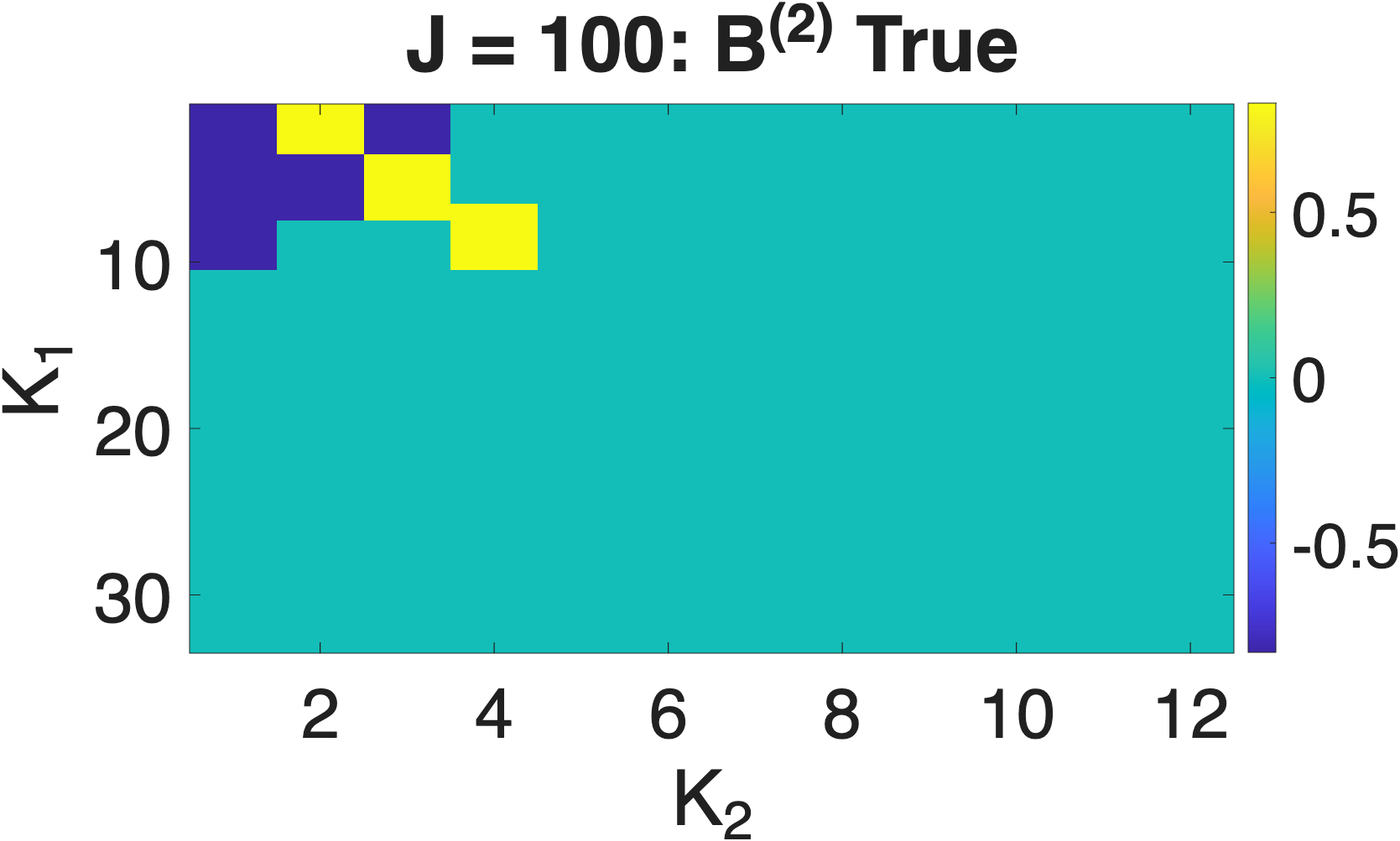}
    \caption{\small Data generating $\{\boldsymbol B^{(d)}\}_{d=1}^{2}$ for $J = 100$}
    \label{fig:DGPsim}
\end{figure}
}
Data are generated from a two-layer DDE copula model with $K^{(1)} = 10$ first-layer latent variables and $K^{(2)} = 3$ second-layer latent variables. The observed dimension $J$ of $\boldsymbol{Y}$ varies across settings, with $J \in \{50, 100, 150\}$. 

{
\setlength{\itemsep}{0pt}        
\setlength{\parsep}{0pt}
\setlength{\parskip}{0pt}
\renewcommand{\baselinestretch}{0.9}\normalsize

\begin{figure}[h]
    \centering
    \includegraphics[width=0.49\linewidth, keepaspectratio]{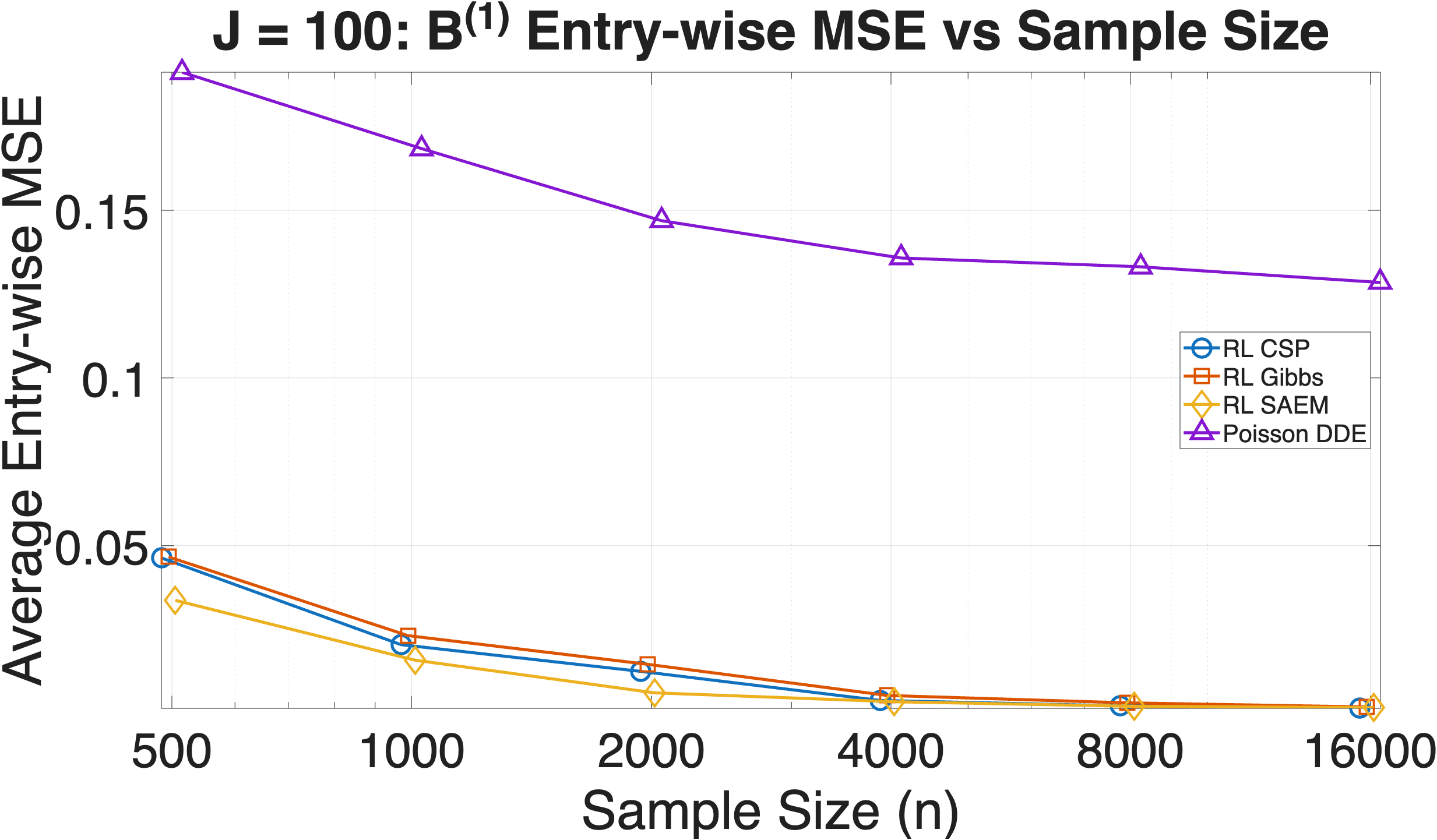}
        \includegraphics[width=0.49\linewidth,keepaspectratio]{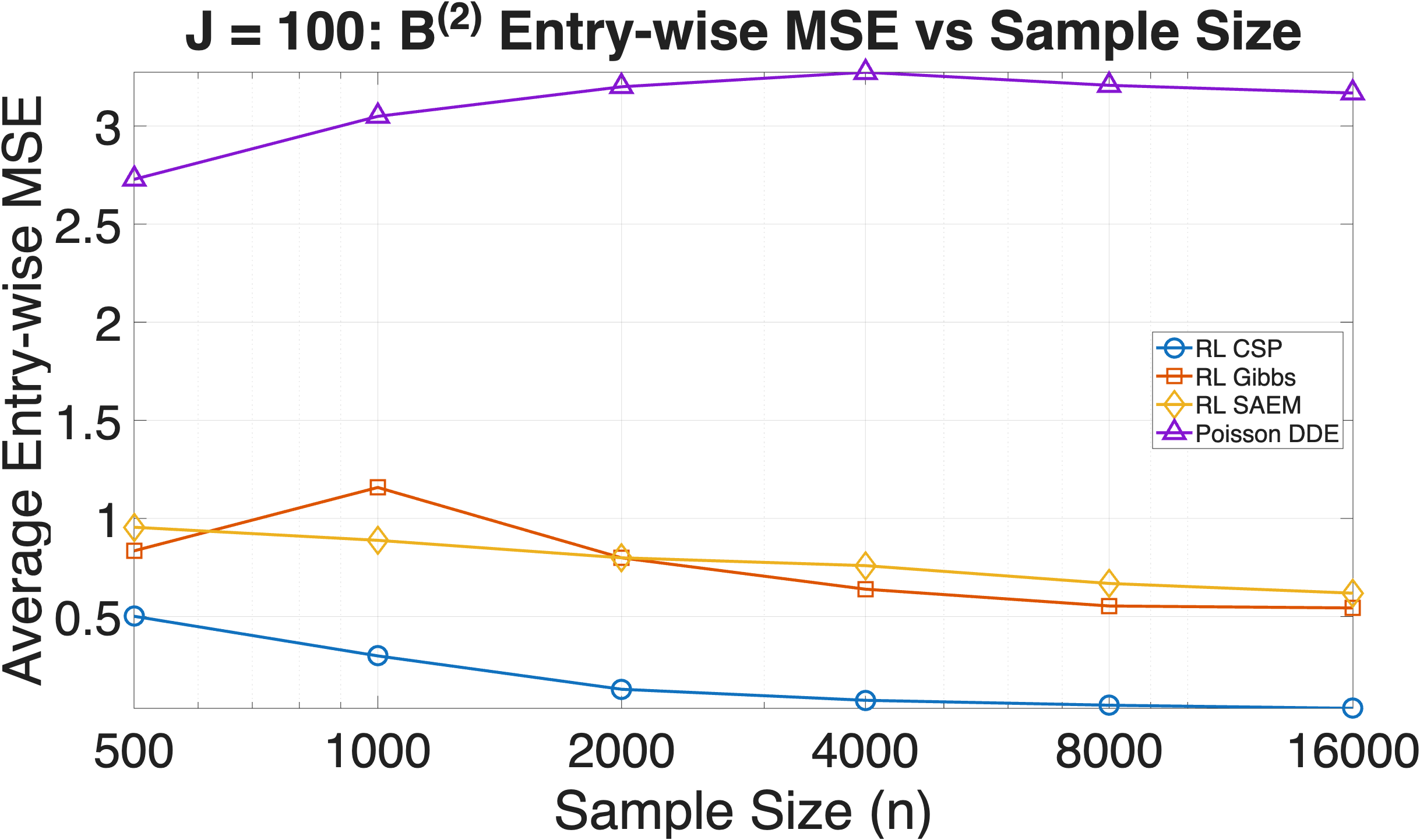}
    \caption{\small $J = 100$: Average entry-wise MSE for estimates of $\boldsymbol{B}^{(1)}$ (top four rows) and $\boldsymbol{B}^{(2)}$ across sample sizes and methods} 
    \label{fig:MSEBsJ100}
\end{figure}
}

The loading matrices are constructed to exhibit structured sparsity, see Figure \ref{fig:DGPsim} for the exact patterns. To mimic a plausible real data setting, groups of variables in $\boldsymbol{Y}$ load strongly (both positively and negatively) on certain shallow-level latent variables, which also share deeper latent features. Crucially, $K^{(d)}<< K_{\text{max}}^{(d)}$ for each $J$, providing a challenging setting for dimension selection and weight matrix estimation.
Given these parameters, data are generated from the DDE copula using the Gaussian latent construction \eqref{condZ}--\eqref{eq::PIT}. To induce heterogeneous marginals, $\{F_{Y_j}\}_{j=1}^{J}$ are assigned cyclically via $\mathrm{mod}(j-1,3)+1$. Let $Q_j = F_{Z_j}(Z_j)$. For type 1 variables, we generate zero-inflated Poisson marginals by setting $X_j=0$ with probability $\pi_0=0.3$, and otherwise applying $X_j = F^{-1}_{\text{Pois}(r_j)}\!\big((Q_j-\pi_0)/(1-\pi_0)\big)$ for $Q_j \ge \pi_0$. For type 2 variables, we use discretized Gamma marginals, $X_j = \mathrm{round}\!\big(F^{-1}_{\Gamma(2,r_j)}(Q_j)\big)$, yielding over-dispersed counts with mean approximately $2r_j$. For type 3 variables, we use standard Poisson marginals, $X_j = F^{-1}_{\text{Pois}(r_j)}(Q_j)$. For each $j$, $r_j$ is drawn uniformly from $\{1,\dots,10\}$. The resulting data exhibit dependence governed by the DDE copula, while retaining heterogeneous marginal scales and shapes.

Across 100 simulated datasets with $n \in \{500,1000,2000,4000,8000,16000\}$, we initialize rank-consistent latent variables $\boldsymbol Z$ via Algorithm~\ref{alg:spectral_init}, and initialize $\boldsymbol \theta$ and $\{\boldsymbol A^{(d)}\}_{d=1}^{2}$ using the Double SVD approach of \cite{lee2026dde} given $\boldsymbol Z$. All weight matrices are initialized at their maximal admissible dimensions under the identifiability conditions of Section~\ref{sec:DDEcop}: $\boldsymbol B^{(1)} \in \mathbb{R}^{J \times (\lfloor J/3 \rfloor + 1)}$ and $\boldsymbol B^{(2)} \in \mathbb{R}^{(\lfloor J/3 \rfloor) \times (\lfloor J/9 \rfloor + 1)}$. As $J$ increases, this induces increasing misspecification in the initial dimensionality.

We evaluate the ability of the Bayesian extended rank likelihood DDE copula with independent layer-wise CSP priors (RL CSP) to recover (i) the true layer dimensions $\{K^{(d)}\}$, (ii) the sparsity structure $\{G^{(d)}\}_{d=1}^{2}$ with $g^{(d)}_{jk} = \mathbb{I}(\beta^{(d)}_{jk} \neq 0)$, and (iii) the weight matrices $\{\boldsymbol B^{(d)}\}_{d=1}^{2}$. Hyperparameter choices are provided in the supplement; in practice, as $J$ increases, we increase $\lambda_{1}^{(d)}$ and decrease the initial temperature to encourage stronger shrinkage in increasingly sparse regimes.

We compare against three alternative approaches. First, RL Gibbs replaces the approximate Gibbs step in Algorithm~\ref{alg:mcem} with exact Gibbs updates for $\boldsymbol A^{(d)}$, holding other steps fixed. Second, RL SAEM replaces the Bayesian CSP prior with a maximum likelihood approach based on a stochastic approximation EM algorithm \citep{lee2026dde}, augmented with rank-likelihood sampling (Algorithm~\ref{alg:rank_aug}) and a truncated lasso penalty \citep{shen2012likelihood} to induce sparsity and estimate latent dimension. Finally, we consider a Poisson DDE that omits the Gaussian copula layer \eqref{condZ}--\eqref{eq::PIT}, assuming conditionally Poisson outcomes in \eqref{outcome} that is compatible with the discrete marginals in $\boldsymbol Y$. All methods are initialized using the same $\boldsymbol Z$ (when applicable) and $\boldsymbol \theta$.

{
\setlength{\itemsep}{0pt}        
\setlength{\parsep}{0pt}
\setlength{\parskip}{0pt}
\renewcommand{\baselinestretch}{0.9}\normalsize

\begin{table}[ht]
\centering
\begin{tabular}{lcccccc}
\toprule
 & \multicolumn{6}{c}{$n$} \\
\cmidrule(lr){2-7}
 & 500 & 1000 & 2000 & 4000 & 8000 & 16000 \\
\midrule

\multicolumn{7}{l}{\textbf{Average recovery of $G_1$}} \\
RL CSP      & 0.977 & 0.994 & 0.997 & 1.000 & 1.000 & 1.000 \\
RL Gibbs    & 0.977 & 0.993 & 0.997 & 1.000 & 1.000 & 1.000 \\
RL SAEM     & 0.977 & 0.989 & 0.995 & 0.998 & 0.999 & 1.000 \\
Poisson DDE & 0.941 & 0.933 & 0.920 & 0.921 & 0.920 & 0.922 \\

\multicolumn{7}{l}{\textbf{Average recovery of $G_2$}} \\
RL CSP      & 0.974 & 0.985 & 0.990 & 1.000 & 1.000 & 1.000 \\
RL Gibbs    & 0.946 & 0.933 & 0.921 & 0.917 & 0.915 & 0.917 \\
RL SAEM     & 0.795 & 0.808 & 0.819 & 0.847 & 0.855 & 0.849 \\
Poisson DDE & 0.691 & 0.683 & 0.650 & 0.655 & 0.663 & 0.662 \\

\multicolumn{7}{l}{\textbf{Average MAP estimate of Layer 1 dimension $K^{(1)}$}} \\
RL CSP      & 8.140 & 9.020 & 9.150 & 9.980 & 10.000 & 10.000 \\
RL Gibbs    & 8.100 & 8.880 & 9.100 & 9.960 & 9.990 & 10.000 \\
RL SAEM     & --    & --    & --    & --    & --    & --    \\
Poisson DDE & --    & --    & --    & --    & --    & --    \\

\multicolumn{7}{l}{\textbf{Average MAP estimate of Layer 2 dimension $K^{(2)}$}} \\
RL CSP      & 3.300 & 2.950 & 2.990 & 3.000 & 3.020 & 3.010 \\
RL Gibbs    & 6.010 & 7.360 & 8.680 & 8.990 & 8.960 & 8.810 \\
RL SAEM     & --    & --    & --    & --    & --    & --    \\
Poisson DDE & --    & --    & --    & --    & --    & --    \\

\bottomrule
\end{tabular}
\caption{\small $J = 100$: Sparsity pattern recovery accuracy and estimated number of active nodes across methods and sample sizes.}
\label{tab:recovery_resultsJ100}
\end{table}
}
We report results for $J=100$ observed variables, and observe qualitatively similar patterns in other settings. Table~\ref{tab:recovery_resultsJ100} summarizes average layer-wise graph recovery and MAP estimates of the latent dimensions across methods. Only RL CSP and RL Gibbs estimate $K^{(d)*}$, as latent dimension is treated as an unknown parameter and updated via Algorithm~\ref{alg:mcem}. RL CSP achieves consistently higher sparsity recovery and near-perfect dimension estimation, whereas competing approaches degrade substantially in deeper layers.

These patterns are reflected in Figures~\ref{fig:MSEBsJ100}-\ref{fig:BsJ100}, which display the average entry-wise MSE between $\boldsymbol B^{(d)}$ and estimates $\hat{\boldsymbol B}^{(d)}$ and the average estimate $\boldsymbol{\hat{B}}^{(2)}$ across sample sizes and methods, respectively. Relative to other rank-likelihood-based estimators, RL CSP yields uniformly lower error, highlighting the combined benefits of the approximate latent-variable sampling step and the CSP shrinkage prior. The Poisson DDE performs poorly across all settings, indicating that parametric DDE formulations—even when aligned with the data type—lack sufficient flexibility to capture complex marginal structure. Additional results in the supplement, including for deeper models ($D > 2)$, show similarly strong performance for RL CSP in both dimension recovery and weight estimation. 

Finally, we note that fitting was fast: for The RL CSP DDE Copula, average run times were 2.83 minutes for $n = 1000$, 4.25 minutes for $n = 4000$, and 26.89 minutes for $n = 16000$ \emph{without} parallel computation (we ran simulation iterations in parallel for each sample size). However, our code implementation is parallelized for the sampling of $\boldsymbol A, \boldsymbol Z$ and optimization of $\boldsymbol B^{(d)}$, which significantly decreases run time in large $n$ and $J$ regimes.


{
\setlength{\itemsep}{0pt}        
\setlength{\parsep}{0pt}
\setlength{\parskip}{0pt}
\renewcommand{\baselinestretch}{0.9}\normalsize

\begin{figure}[h]
    \centering
        \includegraphics[width=.85\linewidth,keepaspectratio]{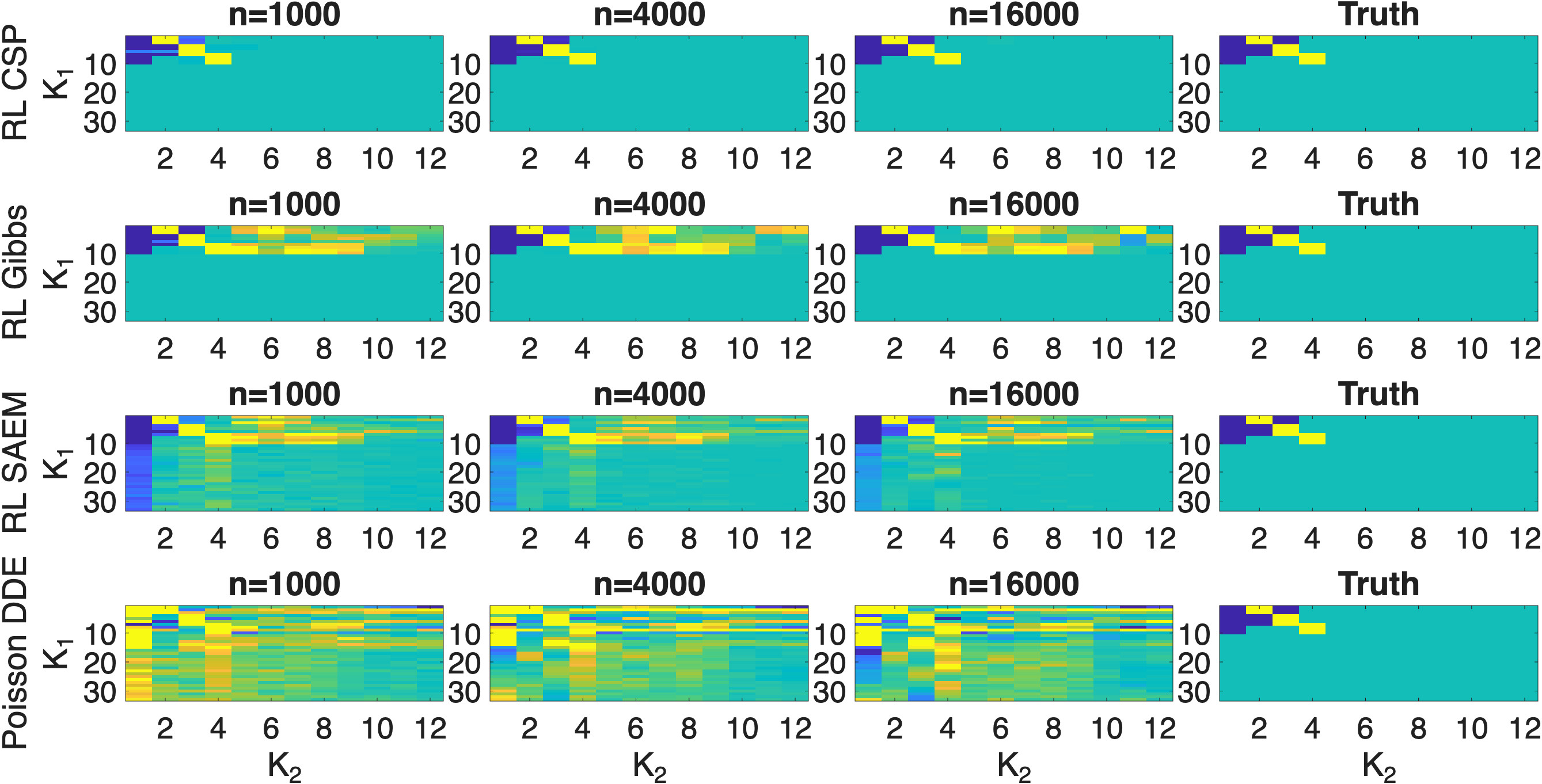}
    \caption{\small$J = 100$: Average point estimates of $\boldsymbol{B}^{(2)}$  across sample sizes and methods compared to the data generating values} 
    \label{fig:BsJ100}
\end{figure}}

\section{Application to Personality and Political Leanings}\label{sec:realdat}

We apply the proposed method to analyze an online survey of 1500 respondents conducted by YouGov linking personality traits, demographic characteristics, and political ideology \citep{montgomery2013computerized}. Respondents were matched to a sampling frame on gender, age, race, and education, which was constructed by stratified sampling from the full 2016 American Community Survey (ACS; \cite{acs2016}). We subsequently refer to these data as the Big5 survey.

Participants were asked to self-rate themselves across 98 items (we exclude two items measuring likelihood  to vote for liberal or conservative political candidates due to our downstream analysis), consistent with the Big Five construct of personality \citep{soto2013five}. Their answers were recorded on an ordinal scale ranging from 1 ("very inaccurate") to 5 ("very accurate"). An example item is displayed in Table~\ref{fig:b5_qs}.

{
\setlength{\itemsep}{0pt}        
\setlength{\parsep}{0pt}
\setlength{\parskip}{0pt}
\renewcommand{\baselinestretch}{0.9}\normalsize

\begin{table}[ht]
\centering

\begin{tabular}{rcl}
\hline
\textbf{Count} & \textbf{Code} & \textbf{Label} \\
\hline
12  & 1 & Very inaccurate \\
67  & 2 & Moderately inaccurate \\
243 & 3 & Neither inaccurate nor accurate \\
694 & 4 & Moderately accurate \\
483 & 5 & Very accurate \\
1   & 8 & Skipped \\
0   & 9 & Not asked \\
\hline
\end{tabular}
\caption{\small Response distribution for Big5 Survey Item: ``Complete tasks successfully.''}\label{fig:b5_qs}
\begin{minipage}{0.72\textwidth}
\footnotesize
\end{minipage}
\end{table}
}
Surveys of personality exhibit complex between- and within-trait dependencies \citep{chen2024idiographic}. In addition to the widely studied Big Five traits (Agreeableness, Openness, Neuroticism, Conscientiousness, and Extraversion), prior work suggests the presence of higher-order latent structure \citep{digman1997higher}. Estimating such structure is challenging, as few methods simultaneously (i) accommodate ordinal discrete data, (ii) capture nonlinear and hierarchical dependence, and (iii) yield interpretable summaries of these relationships.

We fit a two-layer RL CSP DDE copula to the Big5 survey, with $K^{(1)}_{\max} = \lfloor 98/3 \rfloor$ and $K^{(2)}_{\max} = \lfloor K^{(1)}_{\max}/3 \rfloor$. Model and hyperparameter selections were based on predictive evaluation. For each candidate configuration of hyperparameters, we estimated $\hat{\boldsymbol{\theta}}$ and generated $50$ synthetic datasets of size $n = 1500$ via sequential sampling of $\boldsymbol{\tilde{A}}_i^{(2), \ell}$, $\boldsymbol{\tilde{A}}_i^{(1), \ell}$, and $\boldsymbol{\tilde{Z}}_i^{\ell}$ for $\ell = 1,\dots,50$, followed by $\tilde{\boldsymbol{Y}}_i^{\ell} = \hat{F}_{Y_j}\{F_{Z_j}(\tilde{Z}_{ij}^{\ell})\}$, where $\hat{F}_{Y_j}$ denotes the empirical distribution function of $Y_j$.

We then combined the observed and synthetic data and trained a Random Forest classifier \citep{breiman2001random} using five-fold cross-validation. For each hold-out set, we computed the propensity score mean squared error (pMSE), $\text{pMSE}_k^{\ell} = n_{\text{test}}^{-1} \sum_{i=1}^{n_{\text{test}}} (\hat{p}_i - 0.5)^2$ \citep{snoke2018general}, and averaged across folds to obtain $\overline{\text{pMSE}}^{\ell} = 5^{-1} \sum_{k=1}^{5} \text{pMSE}_k^{\ell}$. The pMSE quantifies joint distributional similarity between real and synthetic data; with equal proportions of each, a strong generative model produces synthetic data that is indistinguishable from the real observations and so the classifier assigns probabilities near 0.5. Thus, the DDE copula with the lowest average pMSE best approximates the Big5 joint distribution. Additional univariate and multivariate diagnostics further support the selected model.
 We selected the model minimizing $50^{-1} \sum_{\ell=1}^{50} \overline{\text{pMSE}}^{\ell}$, yielding $K^{(1)*} = 9$ and $K^{(2)*} = 1$.

To aid interpretation, we identified, for each active shallow-layer latent variable $k = 1,\dots,9$, the three items with strongest relative loadings using $\max\{\min_{l \neq k} (\beta_{jk} - \beta_{jl}),\, 0\}$. The resulting items (Table~\ref{tab:keyitems}) are non-overlapping and thematically coherent within each node, suggesting a structured latent representation beyond the classical Big Five. The supplement includes detailed comparisons to an identifiable variational autoencoder (iVAE; \cite{ivae}), constructed to mirror the discovered latent dimension of the DDE copula. The neural network weights for both the encoder and decoder in iVAE are not nearly as interpretable as those discovered using our method.

{
\setlength{\itemsep}{0pt}
\setlength{\parsep}{0pt}
\setlength{\parskip}{0pt}
\renewcommand{\baselinestretch}{0.9}\normalsize
\begin{table}[ht]
\centering
\scriptsize
\setlength{\tabcolsep}{2pt}
\renewcommand{\arraystretch}{1.15}
\resizebox{\textwidth}{!}{%
\begin{tabular}{p{3cm} *{9}{>{\raggedright\arraybackslash}p{1.4cm}}}
\hline
& \multicolumn{9}{c}{\textbf{Binary latent feature}} \\
\cline{2-10}
& $A_1^{(1)}$ & $A_2^{(1)}$ & $A_3^{(1)}$ & $A_4^{(1)}$ & $A_5^{(1)}$ & $A_6^{(1)}$ & $A_7^{(1)}$ & $A_8^{(1)}$ & $A_9^{(1)}$ \\
\hline

\textbf{Key Item 1}
& Keep others at a distance
& Talk to a lot of different people at parties
& Follow through with my plans
& Do not like art
& Am very pleased with myself
& Do not like to draw attention to myself
& Worry about things
& Have a good word for everyone
& Shirk my duties
\\ \hline

\textbf{Key Item 2}
& Am hard to get to know
& Am the life of the party
& Carry out my plans
& Do not enjoy going to art museums
& Feel comfortable with myself
& Keep in the background
& Panic easily
& Enjoy hearing new ideas
& Do just enough work to get by
\\ \hline

\textbf{Key Item 3}
& Avoid contact with others
& Make friends easily
& Do things according to a plan
& Do not like poetry
& Seldom feel blue
& Don't talk a lot
& Get stressed out easily
& Sympathize with other's feelings
& Don't see things through
\\ \hline

\end{tabular}}
    \caption{\small Key items for each binary latent in the shallow layer of the estimated DDE copula.}
    \label{tab:keyitems}
\end{table}
}
At the deeper layer, $A_1^{(2)}$ partitions the shallow-layer factors into two groups: $(A_1^{(1)}, A_4^{(1)}, A_6^{(1)}, A_7^{(1)}, A_9^{(1)})$ load negatively, while $(A_2^{(1)}, A_3^{(1)}, A_5^{(1)}, A_8^{(1)})$ load positively in $\hat{\boldsymbol{B}}^{(2)}$. The associated items suggest a dichotomy between more withdrawn and reactive traits and those reflecting greater social engagement and self-regulation. 

Finally, we examined the relationship between latent representations and political ideology. After convergence, we extracted $\{(\boldsymbol{\hat{A}}_i^{(1)}, \boldsymbol{\hat{A}}_i^{(2)})\}_{i=1}^{1500}$ and focused on 389 individuals identifying as ``Very Liberal'' or ``Very Conservative'' (48.59\% and 51.41\%, respectively).
{
\setlength{\itemsep}{0pt}
\setlength{\parsep}{0pt}
\setlength{\parskip}{0pt}
\renewcommand{\baselinestretch}{0.9}\normalsize
\begin{figure}[h]
    \centering
    \includegraphics[width=.75\linewidth, keepaspectratio]{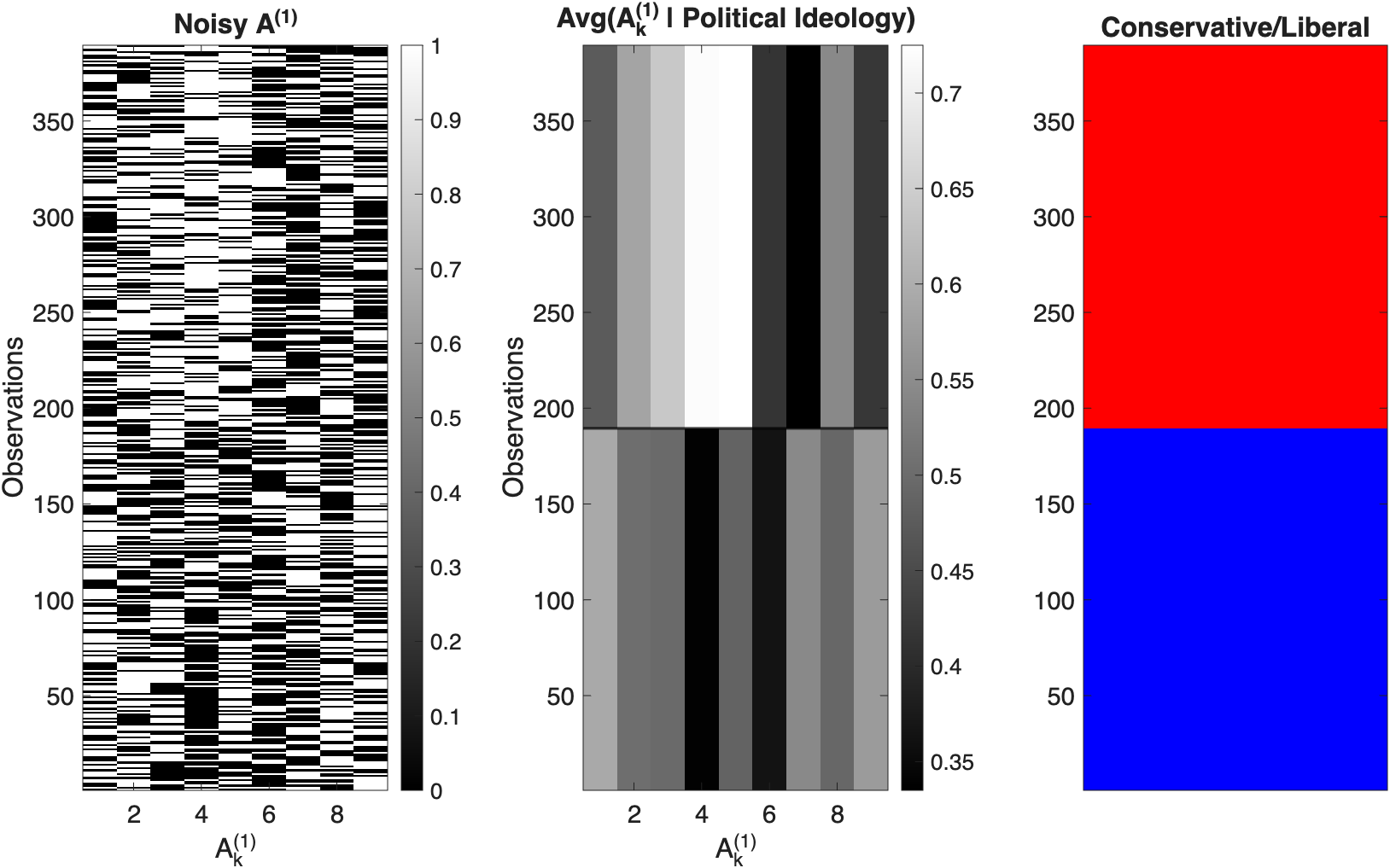}
    \caption{\small Estimated shallow latent representation $\boldsymbol A^{(1)}$ among individuals who self-report as very conservative or very liberal (left) and empirical averages of $A^{(1)}_{k}$ (middle) partitioned by political ideology (right).}
    \label{fig:ideology}
\end{figure}
}

We computed group-specific averages $\bar{A}_{k,\text{ideology}} = |n_{\text{ideology}}|^{-1} \sum_{i \in \text{ideology}} \hat{A}_{ik}^{(1)}, \ \text{ideology} \in \{\text{Very Conservative}, \text{Very Liberal}\}$ and compared these across groups (Figure~\ref{fig:ideology}). Several factors, notably $(A_4^{(1)}, A_5^{(1)}, A_7^{(1)}, A_9^{(1)})$, exhibit systematic differences by ideology. These correspond to variation in traits related to artistic inclination, reactivity, self-confidence, and work ethic. Additionally, 46\% of Very Liberal individuals have $\hat{A}_1^{(2)} = 1$, compared to 56\% among Very Conservative individuals.

\section{Discussion}\label{sec:disc}

We introduce the Bayesian Deep Discrete Encoder Copula, a generative model that accommodates arbitrary and mixed data types. By incorporating layer-wise cumulative shrinkage priors, the framework enables principled estimation of unknown layer widths within the graphical model architecture. We develop scalable initialization and MAP estimation procedures, leveraging a Monte Carlo EM algorithm with variationally motivated approximate sampling steps for latent variables, yielding tractable maximization for the DDE parameters. 

Theoretically, we establish identifiability of the canonical DDE copula from the copula law under transparent layerwise pure-child and separation conditions. The argument combines
the tensor-identifiability machinery for DDEs with a new recovery step for the normalized
Gaussian copula layer. {The rank-likelihood results then clarify the large-sample inferential
target: exact-rank posterior consistency on the quotient space for continuous margins, and
generalized-posterior concentration for tied margins under the supplementary extended-rank
contrast condition.}

In a comprehensive simulation study, the proposed method demonstrates improved recovery of layer-specific weight matrices and latent dimensions relative to existing parametric DDE approaches across a range of challenging settings. Gains are particularly pronounced in deeper layers, underscoring the importance of the copula construction, approximate Gibbs updates, and Bayesian shrinkage. We further apply the model to a moderately sized real-world dataset of personality survey responses measured on discrete ordinal scales. The analysis reveals fine-grained personality structure that departs from classical theory while also recovering coarser higher-order organization. These findings are supported by strong associations between learned binary embeddings and political ideology, with clear contrasts emerging between liberal and conservative respondents.

This work contributes to the growing literature on interpretable unsupervised learning by extending identifiable deep generative models to settings with arbitrary marginal distributions via rank likelihoods. 
This extension is more than a marginal modeling device: it preserves the interpretation of
DDE parameters as features of multivariate dependence after arbitrary monotone marginal
transformations. Consequently, the learned binary embeddings and directed edges can be
viewed as summaries of dependence rather than artifacts of a chosen parametric model for
each observed variable.
The proposed framework also provides scalable estimation procedures compatible with modern Bayesian shrinkage priors.

Several directions for future research remain. To enable uncertainty quantification without full MCMC, one approach is to pre-train the binary latent variables $\boldsymbol{A}$ via a MAP procedure \citep{mauri2025factor}, and then condition on $\boldsymbol{\hat{A}}$ when estimating $\boldsymbol{\theta}$. This suggests the approximation 
$p\{\boldsymbol \theta \mid \boldsymbol Z \in \mathcal{R}(\boldsymbol Y)\}
\approx
p\{\boldsymbol \theta \mid \boldsymbol Z \in \mathcal{R}(\boldsymbol Y), \boldsymbol{\hat{A}}\},$
which may be theoretically justified in high dimensions. More broadly, the flexibility of the DDE copula makes it well suited for extracting interpretable structure from multi-modal data. Finally, our analysis of Big Five survey data suggests that personality may exhibit richer structure than traditionally assumed; future work can investigate whether similar patterns arise across other psychological instruments.

\paragraph*{Data Availability}
The data that support the findings of this study are available from the corresponding author, J.F., upon reasonable request.
\paragraph*{Supplementary Material.}
The Supplementary Material contains proofs of the theoretical results, computational details, and additional numerical results.
\paragraph*{Funding}
Yuqi Gu's research is partially supported by NSF grant DMS-2210796

\clearpage
\setcounter{section}{0}
\setcounter{subsection}{0}
\setcounter{figure}{0}
\setcounter{table}{0}
\setcounter{equation}{0}
\setcounter{assumption}{0}
\setcounter{theorem}{0}
\setcounter{lemma}{0}
\renewcommand{\thetable}{S.\arabic{table}}
\renewcommand{\thefigure}{S.\arabic{figure}}
\renewcommand{\theequation}{S.\arabic{equation}}
\renewcommand{\theassumption}{S.\arabic{assumption}}
\renewcommand{\thetheorem}{S.\arabic{theorem}}
\renewcommand{\thelemma}{S.\arabic{lemma}}
\counterwithin{figure}{section}
\counterwithin{equation}{section}

\begingroup
\spacingset{1}
\begin{center}
{\Large\bfseries Supplementary Material for\par}
\vspace{0.8em}
{\Large\bfseries ``Identifiable Bayesian Deep Generative Copulas with Unknown Layer Widths for Data with Arbitrary Marginal Distributions''\par}
\vspace{1.5em}
\ifblinded
\else
{\large Joseph Feldman and Yuqi Gu\par}
\fi
\end{center}
\endgroup

\vspace{1.5em}
\normalsize
\doublespacing

\appendix
\noindent In Section~\ref{sec:proof}, we present the proofs of the theoretical results.
In Section~\ref{sec:alg}, we include detailed descriptions of our proposed coordinate ascent Monte Carlo EM algorithm, as well as details about the SAEM RL algorithm, hyperparameter selections, and methods for addressing latent variable permutations in our simulations.  Section \ref{sec:addsim} contains additional simulation results, completing the details excluded from the Simulation studies in the main paper, and extending the DDE copula to deeper architectures. In Section \ref{sec:iVAE}, we include results and comparisons from fitting identifiable variational autoencoders to the Big5 survey data.

\section{Proofs of Theoretical Results}\label{sec:proof}
\subsection{Proof of Theorem 1.}
The proof proceeds in four steps. First, on the copula scale, the distribution of
\(\boldsymbol U=F_{\boldsymbol Z}(\boldsymbol Z)\) is a one-layer saturated latent class model with latent variable \(\boldsymbol A^{(1)}\).
Second, a tensor-decomposition argument identifies the first-layer latent mixing proportions
\(\boldsymbol \eta^{(1)}\) and the conditional copula-scale distribution functions
\[
    H_{j\boldsymbol \alpha}(u)=P(U_j\le u\mid \boldsymbol A^{(1)}=\boldsymbol \alpha)
\]
up to the unavoidable relabeling of the binary latent coordinates. Third, the special Gaussian
structure of the first layer identifies the normalized Gaussian parameters
\((\boldsymbol B^{(1)}, \boldsymbol\gamma)\) from the functions \(H_{j\boldsymbol \alpha}\). Fourth, once the distribution of
\(\boldsymbol A^{(1)}\) is identified, the same one-layer argument is applied recursively to identify
\((\boldsymbol B^{(2)},\boldsymbol G^{(2)}),\ldots,(\boldsymbol B^{(D)},\boldsymbol G^{(D)})\) and the top-layer Bernoulli probabilities \(\pi_1,\ldots,\pi_{K^{(D)}}\).

\begin{lemma}[Copula-scale latent class representation]
\label{lem:copula-latent-class}
Let \(U_j=F_{Z_j}(Z_j)\). Then, conditional on \(\boldsymbol A^{(1)}=\boldsymbol \alpha\), the coordinates
\(U_1,\ldots,U_J\) are mutually independent, and
\begin{equation}
    C_{\boldsymbol{\Theta}}(u_1,\ldots,u_J)
    =
    \sum_{\boldsymbol \alpha\in\{0,1\}^{K^{(1)}}}
    \eta^{(1)}_{\boldsymbol \alpha}
    \prod_{j=1}^J H_{j\boldsymbol \alpha}(u_j),
    \label{eq:copula-latent-class}
\end{equation}
where
\begin{equation}
    H_{j\boldsymbol \alpha}(u)
    =
    \Phi\left\{
        \frac{Q_j(u)-\mu_{j\boldsymbol \alpha}}{\sqrt{\gamma_j}}
    \right\},
    \qquad
    Q_j(u)=F_{Z_j}^{-1}(u),
    \qquad
    0<u<1 .
    \label{eq:Hja-definition}
\end{equation}
\end{lemma}

\begin{proof}
Since \(\gamma_j>0\) and \(Z_j\) is a finite mixture of normal distributions with positive
component variances, \(F_{Z_j}\) is continuous and strictly increasing. Hence
\(Q_j=F_{Z_j}^{-1}\) is well-defined on \((0,1)\). Conditional on \(\boldsymbol A^{(1)}=\boldsymbol \alpha\),
the variables \(Z_1,\ldots,Z_J\) are independent normal variables. Therefore,
\[
    P(U_j\le u\mid \boldsymbol A^{(1)}=\boldsymbol \alpha)
    =
    P(Z_j\le Q_j(u)\mid \boldsymbol  A^{(1)}=\boldsymbol \alpha)
    =
    \Phi\left\{
        \frac{Q_j(u)-\mu_{j\alpha}}{\sqrt{\gamma_j}}
    \right\}.
\]
Taking the product over \(j\) conditional on \(\boldsymbol A^{(1)}=\boldsymbol \alpha\), and then mixing over
\(\boldsymbol \alpha\), gives \eqref{eq:copula-latent-class}.
\end{proof}

\begin{lemma}[Tensor identification of the shallowest latent layer]
\label{lem:tensor-identification-first-layer}
Under Assumptions~\ref{ass:known-widths-positivity}--\ref{ass:pure-separation}, the
copula law \(C_{\boldsymbol{\Theta}}\) identifies
\[
    \boldsymbol \eta^{(1)}
    =
    \{\eta^{(1)}_{\boldsymbol \alpha}: \boldsymbol \alpha\in\{0,1\}^{K^{(1)}}\}
\]
and the conditional distribution functions
\[
    \{H_{j\boldsymbol \alpha}:j=1,\ldots,J,\ \boldsymbol \alpha\in\{0,1\}^{K^{(1)}}\},
\]
up to a permutation and possible \(0/1\) relabeling of the coordinates of \(\boldsymbol A^{(1)}\).
The orientation condition \eqref{eq:orientation-condition}, together with
Lemma~\ref{lem:gaussian-contrast-recovery}, removes the \(0/1\) relabeling.
\end{lemma}

\begin{proof}
Let \(K=K^{(1)}\). For each \(j\), choose a finite collection of measurable subsets of
\((0,1)\),
\[
    \mathcal S_j=\{S_{j1},\ldots,S_{j\kappa_j}\},
    \qquad S_{j\kappa_j}=(0,1),
\]
and define
\[
    M_j(s,\alpha)
    =
    P(U_j\in S_{js}\mid A^{(1)}=\alpha).
\]
For a pure child \(j=r^{(1)}_{k,a}\) of \(A_k^{(1)}\), the conditional distribution of \(U_j\)
depends only on \(\alpha_k\). Since \(\beta^{(1)}_{jk}\ne0\), the two conditional distributions
corresponding to \(\alpha_k=0\) and \(\alpha_k=1\) are distinct. Hence there exists a measurable
set \(S_j\subset(0,1)\) such that
\[
    P(U_j\in S_j\mid \alpha_k=0)
    \ne
    P(U_j\in S_j\mid \alpha_k=1).
\]
Using the two rows \(S_j\) and \((0,1)\), the associated \(2\times2\) conditional-probability
matrix has rank two.

Let
\[
    I_1=\{r^{(1)}_{k,1}:k=1,\ldots,K\},
    \qquad
    I_2=\{r^{(1)}_{k,2}:k=1,\ldots,K\},
    \qquad
    I_3=\{1,\ldots,J\}\setminus(I_1\cup I_2).
\]
Construct conditional-probability matrices \(N_1,N_2,N_3\) whose columns are indexed by
\(\alpha\in\{0,1\}^K\) and whose rows are indexed by all Cartesian products of the finite
sets \(\mathcal S_j\) over \(j\in I_a\), \(a=1,2,3\). For example,
\[
    N_1(\xi,\boldsymbol\alpha)
    =
    P(U_{I_1}\in S_\xi\mid \boldsymbol A^{(1)}=\boldsymbol \alpha)
    =
    \prod_{j\in I_1} P(U_j\in S_{\xi_j}\mid \boldsymbol A^{(1)}=\boldsymbol \alpha),
\]
where the product follows from conditional independence.

Because each \(I_a\), \(a=1,2\), contains one pure child for each latent coordinate, \(N_1\)
and \(N_2\) are Kronecker products, up to column permutation, of \(K\) rank-two matrices.
Therefore
\[
    \operatorname{rank}(N_1)=\operatorname{rank}(N_2)=2^K.
\]
Since each \(N_a\) is a square \(2^K\times 2^K\) matrix of full column rank, every subset of
at most \(2^K\) columns is linearly independent, so the Kruskal rank equals the matrix rank:
\(k_{N_1}=k_{N_2}=2^K\).

Next, for \(d=1\) we have \(K^{(0)}=J\) and \(\mathcal P_1=I_1\cup I_2\), so
\(\{1,\ldots,K^{(0)}\}\setminus\mathcal P_1 = I_3\); the separation condition in
Assumption~\ref{ass:pure-separation} therefore applies to variables in \(I_3\).
For every \(\boldsymbol \alpha\ne\boldsymbol \alpha'\), some \(r\in I_3\) satisfies
\(\mu_{r\boldsymbol\alpha}\ne\mu_{r\boldsymbol\alpha'}\), i.e., the conditional means of \(Z_r\)
differ. Since \(\mathcal{N}(\mu_{r\boldsymbol\alpha},\gamma_r)\ne\mathcal{N}(\mu_{r\boldsymbol\alpha'},\gamma_r)\)
whenever \(\mu_{r\boldsymbol\alpha}\ne\mu_{r\boldsymbol\alpha'}\), this gives distinct
conditional distributions for that coordinate. By adding finitely many separating measurable
sets to the collections \(\mathcal S_j\), we can ensure that the corresponding columns of
\(N_3\) are distinct for every pair \(\boldsymbol \alpha\ne \boldsymbol \alpha'\). Because the last row of \(N_3\) is the
all-one row, any two distinct columns of \(N_3\) are linearly independent. Hence the Kruskal
rank of \(N_3\) is at least two.

The joint probabilities of the discretized copula-scale variables form a three-way tensor
\[
    \mathcal T
    =
    [\boldsymbol N_1\operatorname{diag}(\boldsymbol \eta^{(1)}),\,\boldsymbol N_2,\,\boldsymbol N_3]
    =
    \sum_{\boldsymbol \alpha\in\{0,1\}^K}
    \eta^{(1)}_{\boldsymbol \alpha}
    \boldsymbol N_{1,\boldsymbol \alpha}\circ \boldsymbol N_{2,\boldsymbol \alpha}\circ \boldsymbol N_{3,\boldsymbol \alpha},
\]
where \(\boldsymbol N_{a,\alpha}\) denotes the \(\boldsymbol \alpha\)-th column of \(\boldsymbol N_a\), and \(\circ\) denotes the
outer product. The Kruskal ranks satisfy
\[
    k_{N_1}+k_{N_2}+k_{N_3}
    \ge
    2^K+2^K+2
    =
    2(2^K)+2 .
\]
By Kruskal's uniqueness theorem for three-way tensor decompositions \citep{kruskal1977three},
this decomposition is unique up to a simultaneous permutation and scaling of the \(2^K\) components. The
scaling indeterminacy is removed by the fact that the conditional-probability matrices have
a row equal to one; therefore \(\eta^{(1)}\) and all finite-dimensional conditional probabilities
are identified up to a common permutation of the latent configurations.

Varying the finite collections \(\mathcal S_j\) over a countable generating class of Borel
sets in \((0,1)\) identifies the full conditional distribution functions \(H_{j\boldsymbol \alpha}\). Finally,
the two pure-child groups identify the \(K\) binary coordinate bipartitions of the \(2^K\)
latent configurations. Thus the remaining ambiguity is only a permutation of latent coordinates
and an independent \(0/1\) relabeling of each coordinate.
\end{proof}

\begin{lemma}[Recovery of the normalized Gaussian first layer]
\label{lem:gaussian-contrast-recovery}
Suppose \(\eta^{(1)}\) and \(\{H_{j\alpha}\}\) are known with the latent configurations indexed
up to coordinate permutation. Under the canonical normalization
\eqref{eq:canonical-normalization}, the first-layer parameters $\boldsymbol B^{(1)}$, $\boldsymbol  G^{(1)}$, $\boldsymbol \gamma$
are identified up to the same coordinate permutation.
\end{lemma}

\begin{proof}
Fix \(j\). Let \(\sigma_j=\sqrt{\gamma_j}\). From \eqref{eq:Hja-definition},
\[
    \Phi^{-1}\{H_{j\alpha}(u)\}
    =
    \frac{Q_j(u)-\mu_{j\alpha}}{\sigma_j}.
\]
Taking the difference between \(\boldsymbol \alpha=\boldsymbol 0_K\) and a general \(\boldsymbol \alpha\) gives
\begin{equation}
    \Phi^{-1}\{H_{j\boldsymbol 0}(u)\}
    -
    \Phi^{-1}\{H_{j\boldsymbol \alpha}(u)\}
    =
    \frac{Q_j(u)-\mu_{j\boldsymbol 0}}{\sigma_j}
    -
    \frac{Q_j(u)-\mu_{j\boldsymbol \alpha}}{\sigma_j}
    =
    \frac{\mu_{j\boldsymbol \alpha}-\mu_{j\boldsymbol 0}}{\sigma_j}
    =
    \sum_{k=1}^{K^{(1)}} \frac{\beta^{(1)}_{jk}}{\sigma_j}\alpha_k .
    \label{eq:standardized-contrast}
\end{equation}
The \(Q_j(u)\) terms cancel, so the left-hand side is constant in \(u\) and known
(since \(\eta^{(1)}\) and \(\{H_{j\boldsymbol\alpha}\}\) are identified by Lemma~\ref{lem:tensor-identification-first-layer}).
In particular, taking
\(\boldsymbol \alpha=\boldsymbol e_k\), the \(k\)-th canonical basis vector, identifies
\[
    \widetilde\beta^{(1)}_{jk}
    :=
    \frac{\beta^{(1)}_{jk}}{\sigma_j},
    \qquad k=1,\ldots,K^{(1)} .
\]
Let
\[
    \widetilde m_j(\boldsymbol A^{(1)})
    =
    \sum_{k=1}^{K^{(1)}}\widetilde\beta^{(1)}_{jk} A_k^{(1)},
    \qquad
    \bar m_j
    =
    E_{\eta^{(1)}}\{\widetilde m_j(\boldsymbol A^{(1)})\},
    \qquad
    v_j
    =
    \operatorname{Var}_{\eta^{(1)}}\{\widetilde m_j(\boldsymbol A^{(1)})\}.
\]
Because \(\mu_{j\alpha}=\beta^{(1)}_{j0}+\sigma_j\widetilde m_j(\alpha)\), the canonical mean
constraint \(E(Z_j)=0\) gives
\[
    \beta^{(1)}_{j0}
    =
    -\sigma_j \bar m_j .
\]
The canonical variance constraint gives, by the law of total variance,
\[
    1
    =
    \operatorname{Var}(Z_j)
    =
    \underbrace{\gamma_j}_{=\,E[\operatorname{Var}(Z_j\mid A^{(1)})]}
    +
    \underbrace{\operatorname{Var}\{\mu_{j,A^{(1)}}\}}_{=\,\operatorname{Var}(E[Z_j\mid A^{(1)}])}
    =
    \sigma_j^2
    +
    \sigma_j^2 v_j
    =
    \sigma_j^2(1+v_j).
\]
Thus
\[
    \sigma_j^2=\gamma_j=\frac{1}{1+v_j},
    \qquad
    \beta^{(1)}_{jk}
    =
    \sigma_j\widetilde\beta^{(1)}_{jk},
    \qquad
    \beta^{(1)}_{j0}
    =
    -\sigma_j\bar m_j .
\]
Therefore \(\boldsymbol B^{(1)}\) and \(\gamma\) are identified. Faithfulness then identifies
\[
    g^{(1)}_{jk}=1\{\beta^{(1)}_{jk}\ne0\}.
\]
If a coordinate has been relabeled by \(A_k^{(1)}\mapsto 1-A_k^{(1)}\), the recovered
coefficients in the corresponding column change sign, with an accompanying intercept
adjustment. The orientation condition \eqref{eq:orientation-condition} selects the unique
orientation for which the column sum is positive. This removes the \(0/1\) relabeling
ambiguity.
\end{proof}

\begin{lemma}[Layerwise recovery of the deeper DDE]
\label{lem:layerwise-dde-recovery}
Suppose that the marginal distribution
\[
    \eta^{(d-1)}_{\boldsymbol \alpha}
    =
    P(\boldsymbol A^{(d-1)}=\boldsymbol \alpha),
    \qquad
    \boldsymbol \alpha\in\{0,1\}^{K^{(d-1)}},
\]
has been identified, with the coordinates of \(\boldsymbol A^{(d-1)}\) labeled up to permutation. If
Assumptions~\ref{ass:known-widths-positivity}--\ref{ass:pure-separation} hold for layer
\(d\), then \(\boldsymbol B^{(d)}\), \(\boldsymbol G^{(d)}\), and the marginal distribution \(\boldsymbol \eta^{(d)}\) of \(\boldsymbol A^{(d)}\)
are identified up to permutation of the coordinates of \(\boldsymbol A^{(d)}\). For \(d=D\), this also
identifies the top-layer Bernoulli probabilities \(p\).
\end{lemma}

\begin{proof}
Treat \(\boldsymbol A^{(d-1)}\) as the observed binary vector and \(\boldsymbol A^{(d)}\) as the latent vector in a
one-layer saturated latent class model:
\[
    P(\boldsymbol A^{(d-1)}=\boldsymbol x)
    =
    \sum_{\boldsymbol \alpha\in\{0,1\}^{K^{(d)}}}
    \eta^{(d)}_{\boldsymbol \alpha}
    \prod_{r=1}^{K^{(d-1)}}
    P(A^{(d-1)}_r=x_r\mid \boldsymbol A^{(d)}=\boldsymbol \alpha).
\]
For each child \(r\),
\[
    P(A^{(d-1)}_r=1\mid \boldsymbol A^{(d)}=\boldsymbol \alpha)
    =
    \operatorname{logit}^{-1}
    \left(
        \beta^{(d)}_{r0}
        +
        \sum_{k=1}^{K^{(d)}}\beta^{(d)}_{rk}\alpha_k
    \right).
\]
The same tensor argument as in Lemma~\ref{lem:tensor-identification-first-layer} applies.
The two pure-child groups give two full-column-rank matrices, and the separation condition
gives a third matrix with Kruskal rank at least two. Hence the mixing proportions
\(\eta^{(d)}\) and all conditional Bernoulli probabilities are identified up to latent-coordinate
permutation and \(0/1\) relabeling.

Once the conditional probabilities are identified, the coefficients are recovered by logits:
\[
    \operatorname{logit}
    P(A^{(d-1)}_r=1\mid \boldsymbol A^{(d)}=\boldsymbol \alpha)
    =
    \beta^{(d)}_{r0}
    +
    \sum_{k=1}^{K^{(d)}}\beta^{(d)}_{rk}\alpha_k.
\]
Taking \(\alpha=0\) identifies \(\beta^{(d)}_{r0}\), and taking \(\alpha=e_k\) identifies
\[
    \beta^{(d)}_{rk}
    =
    \operatorname{logit}
    P(A^{(d-1)}_r=1\mid \boldsymbol A^{(d)}=\boldsymbol e_k)
    -
    \operatorname{logit}
    P(A^{(d-1)}_r=1\mid \boldsymbol A^{(d)}=\boldsymbol 0).
\]
Faithfulness identifies \(\boldsymbol G^{(d)}\), and the orientation condition removes the \(0/1\) relabeling.
When \(d=D\), the top layer has independent Bernoulli coordinates, so after identifying
\(\boldsymbol \eta^{(D)}\),
\[
    p_k
    =
    P(A^{(D)}_k=1)
    =
    \sum_{\boldsymbol \alpha:\alpha_k=1}\eta^{(D)}_{\boldsymbol \alpha},
    \qquad k=1,\ldots,K^{(D)} .
\]
\end{proof}

\begin{proof}[Proof of Theorem~\ref{thm:strict-identifiability-copuladde}]
By Lemma~\ref{lem:copula-latent-class}, the copula law \(C_{\boldsymbol{\Theta}}\) is a one-layer saturated
latent class model on the copula scale, with latent variable \(\boldsymbol A^{(1)}\). By
Lemma~\ref{lem:tensor-identification-first-layer}, \(C_{\boldsymbol{\Theta}}\) identifies
\(\boldsymbol \eta^{(1)}\) and \(\{H_{j\boldsymbol \alpha}\}\), up to latent-coordinate relabeling. By
Lemma~\ref{lem:gaussian-contrast-recovery}, the canonical Gaussian first-layer parameters
\(\boldsymbol B^{(1)}\), \(\boldsymbol G^{(1)}\), and \(\boldsymbol \gamma\) are identified up to permutation of the coordinates of
\(\boldsymbol A^{(1)}\).

Now \(\boldsymbol \eta^{(1)}\), the full marginal distribution of \(\boldsymbol A^{(1)}\), is known. Applying
Lemma~\ref{lem:layerwise-dde-recovery} with \(d=2\) identifies \(\boldsymbol B^{(2)}\), \(\boldsymbol G^{(2)}\), and
\(\boldsymbol \eta^{(2)}\), up to permutation of the coordinates of \(\boldsymbol A^{(2)}\). Repeating this argument
for \(d=3,\ldots,D\) identifies every deeper coefficient matrix and graphical matrix. At the
top layer, Lemma~\ref{lem:layerwise-dde-recovery} identifies the independent Bernoulli
probabilities \(\boldsymbol p\). Therefore, any two canonical parameters inducing the same copula law
must differ only by latent-variable permutations within layers; that is,
\(\boldsymbol{\Theta}\sim_K\widetilde{\boldsymbol{\Theta}}\).
\end{proof}

\subsection{Proof of Theorem 2}
The proof is organized through four lemmas.

\begin{lemma}[Nuisance-free exact rank likelihood under continuous margins]
\label{lem:exact-rank-likelihood-continuous}
If each \(F_{Y_j}\) is continuous, then \(L_n^R(\boldsymbol{\Theta};\boldsymbol Y^{(n)})\)
is the exact likelihood of the observed coordinatewise ranks \(R_n\). Moreover, it depends
on \(\boldsymbol{\Theta}\) only through the copula \(C_{\boldsymbol{\Theta}}\).
\end{lemma}

\begin{proof}
Since each \(F_{Y_j}\) is continuous, the variables \(Y_{1j},\ldots,Y_{nj}\) have no ties
almost surely. Therefore, for each coordinate \(j\), the observed data determine a unique
permutation \(\rho_j\) such that
\[
    y_{\rho_j(1),j}<y_{\rho_j(2),j}<\cdots<y_{\rho_j(n),j}.
\]
Because \(Y_{ij}=F_{Y_j}^{-1}(U_{ij})\) and continuous margins have no atoms, the same
ordering is equivalent almost surely to
\[
    u_{\rho_j(1),j}<u_{\rho_j(2),j}<\cdots<u_{\rho_j(n),j}.
\]
Thus \(\mathcal D_n(\boldsymbol Y^{(n)})\) is precisely the rank cell associated with the
observed rank array \(R_n\). The collection of such rank cells forms a finite measurable
partition of \((0,1)^{n\times J}\). Hence
\[
    P_{\boldsymbol{\Theta}}\{\boldsymbol U^{(n)}\in\mathcal D_n(\boldsymbol Y^{(n)})\}
    =
    P_{\boldsymbol{\Theta}}\{R_n(\boldsymbol U^{(n)})=R_n(\boldsymbol Y^{(n)})\},
\]
which is the exact likelihood of \(R_n\). Since
\(\boldsymbol U_1,\ldots,\boldsymbol U_n\) are i.i.d. from \(C_{\boldsymbol{\Theta}}\), this
probability depends on \(\boldsymbol{\Theta}\) only through \(C_{\boldsymbol{\Theta}}\).
\end{proof}

\begin{lemma}[Empirical ranks recover the copula]
\label{lem:empirical-ranks-recover-copula}
Assume the margins are continuous. Define
\[
    \widehat U_{nij}
    =
    \frac{1}{n}\sum_{\ell=1}^n 1(Y_{\ell j}\le Y_{ij}),
    \qquad
    \widehat{\boldsymbol U}_{ni}
    =
    (\widehat U_{ni1},\ldots,\widehat U_{niJ}),
\]
and
\[
    \widehat C_n
    =
    \frac{1}{n}\sum_{i=1}^n\delta_{\widehat{\boldsymbol U}_{ni}}.
\]
Then \(\widehat C_n\) is \(\mathscr R_n\)-measurable, and
\[
    d_{\mathrm{BL}}(\widehat C_n,C_{\psi_0})
    \longrightarrow 0
    \qquad
    P_{\psi_0}^{\infty}\text{-almost surely}.
\]
\end{lemma}

\begin{proof}
Let
\[
    \widehat F_{nj}(y)=\frac{1}{n}\sum_{\ell=1}^n1(Y_{\ell j}\le y)
\]
be the empirical CDF of the \(j\)th margin. Then
\(\widehat U_{nij}=\widehat F_{nj}(Y_{ij})\). This quantity is determined by the ordering of
\(Y_{1j},\ldots,Y_{nj}\), and is therefore \(\mathscr R_n\)-measurable.

Since \(F_{Y_j}\) is continuous and \(Y_{ij}=F_{Y_j}^{-1}(U_{ij})\),
\[
    F_{Y_j}(Y_{ij})=U_{ij}
    \qquad\text{almost surely}.
\]
By the Glivenko--Cantelli theorem,
\[
    \sup_y|\widehat F_{nj}(y)-F_{Y_j}(y)|\longrightarrow0
    \qquad\text{almost surely}.
\]
Therefore
\[
    \max_{1\le i\le n}|\widehat U_{nij}-U_{ij}|
    \le
    \sup_y|\widehat F_{nj}(y)-F_{Y_j}(y)|
    \longrightarrow0
    \qquad\text{almost surely}.
\]
Let
\[
    C_n^*=\frac{1}{n}\sum_{i=1}^n\delta_{\boldsymbol U_i}.
\]
For every bounded Lipschitz function \(f:[0,1]^J\to\mathbb R\) with Lipschitz constant at
most one,
\[
\begin{aligned}
    \left|
        \int f\,d\widehat C_n-
        \int f\,dC_n^*
    \right|
    &\le
    \frac{1}{n}\sum_{i=1}^n
    \|\widehat{\boldsymbol U}_{ni}-\boldsymbol U_i\|_1 \\
    &\le
    \sum_{j=1}^J\sup_y|\widehat F_{nj}(y)-F_{Y_j}(y)|
    \longrightarrow0
\end{aligned}
\]
almost surely. Hence \(d_{\mathrm{BL}}(\widehat C_n,C_n^*)\to0\) almost surely. Since
\(\boldsymbol U_1,\boldsymbol U_2,\ldots\) are i.i.d. from \(C_{\psi_0}\), the strong law for
empirical measures under the bounded-Lipschitz metric gives
\[
    d_{\mathrm{BL}}(C_n^*,C_{\psi_0})\to0
    \qquad\text{almost surely}.
\]
The triangle inequality gives $d_{\mathrm{BL}}(\hat C_n,C_{\psi_0})\to0$ and completes the proof.
\end{proof}

\begin{lemma}[A measurable recovery map from infinite ranks]
\label{lem:continuous-rank-recovery-map}
Under the assumptions of Theorem~\ref{thm:continuous-rank-posterior-consistency}, there
exists an \(\mathscr R_\infty\)-measurable map \(h:\mathscr R_\infty\to\Psi_K\) such that,
for every \(\psi_0\in\Psi_K\),
\[
    h(\mathscr R_\infty)=\psi_0
    \qquad
    P_{\psi_0}^{\infty}\text{-almost surely}.
\]
\end{lemma}

\begin{proof}
We first verify continuity of \(\psi\mapsto C_\psi\). Let \(\psi_m\to\psi\) in \(\Psi_K\),
and choose representatives \(\boldsymbol{\Theta}_m\to\boldsymbol{\Theta}\) in the canonical
parameter space after applying within-layer permutations. Since \(\mathcal T_K\) is compact
inside the nondegenerate region, the Gaussian variances are bounded away from zero and
all coefficients are bounded. The joint density of \(\boldsymbol Z\) under
\(\boldsymbol{\Theta}_m\) is a finite mixture of product Gaussian densities whose mixing
probabilities and component parameters vary continuously in \(\boldsymbol{\Theta}_m\).
Hence the joint distribution functions of \(\boldsymbol Z\) converge pointwise at every
continuity point, and the marginal distribution functions converge uniformly on
\(\mathbb R\) to their limits. Since each limiting marginal distribution is continuous and
strictly increasing, the corresponding marginal quantile functions converge uniformly on
compact subsets of \((0,1)\). Therefore, for every
\(\boldsymbol u\in(0,1)^J\),
\[
    C_{\psi_m}(\boldsymbol u)
    =
    F_{\boldsymbol{\Theta}_m}\!\left(
        F_{\boldsymbol{\Theta}_m,1}^{-1}(u_1),\ldots,
        F_{\boldsymbol{\Theta}_m,J}^{-1}(u_J)
    \right)
    \longrightarrow
    F_{\boldsymbol{\Theta}}\!\left(
        F_{\boldsymbol{\Theta},1}^{-1}(u_1),\ldots,
        F_{\boldsymbol{\Theta},J}^{-1}(u_J)
    \right)
    =
    C_{\psi}(\boldsymbol u).
\]
Because all \(C_\psi\) have continuous margins, pointwise convergence on \((0,1)^J\)
implies weak convergence, equivalently convergence under \(d_{\mathrm{BL}}\). Thus
\(\psi\mapsto C_\psi\) is continuous.

By Theorem~\ref{thm:strict-identifiability-copuladde}, the map
\(\psi\mapsto C_\psi\) is injective. Since \(\Psi_K\) is compact and
\(d_{\mathrm{BL}}\) is a metric, injectivity and continuity imply separation: for every
\(\epsilon>0\),
\[
    \Delta_\epsilon(\psi_0)
    :=
    \inf_{\psi:\,d_K(\psi,\psi_0)\ge\epsilon}
    d_{\mathrm{BL}}(C_\psi,C_{\psi_0})
    >0.
    \label{eq:continuous-separation}
\]

Let \(\{\psi_m:m\ge1\}\) be a countable dense subset of \(\Psi_K\). For each \(n\), define
\(\widehat\psi_n\) to be the first element \(\psi_m\) satisfying
\[
    d_{\mathrm{BL}}(\widehat C_n,C_{\psi_m})
    \le
    \inf_{r\ge1}d_{\mathrm{BL}}(\widehat C_n,C_{\psi_r})+\frac1n.
\]
This construction makes \(\widehat\psi_n\) measurable with respect to \(\mathscr R_n\). By
density of \(\{\psi_m\}\) and continuity of \(\psi\mapsto C_\psi\),
\[
    \inf_{r\ge1}d_{\mathrm{BL}}(\widehat C_n,C_{\psi_r})
    =
    \inf_{\psi\in\Psi_K}d_{\mathrm{BL}}(\widehat C_n,C_\psi)
    \le
    d_{\mathrm{BL}}(\widehat C_n,C_{\psi_0}).
\]
Therefore
\[
\begin{aligned}
    d_{\mathrm{BL}}(C_{\widehat\psi_n},C_{\psi_0})
    &\le
    d_{\mathrm{BL}}(C_{\widehat\psi_n},\widehat C_n)
    +
    d_{\mathrm{BL}}(\widehat C_n,C_{\psi_0}) \\
    &\le
    2d_{\mathrm{BL}}(\widehat C_n,C_{\psi_0})+\frac1n
    \longrightarrow0
\end{aligned}
\]
almost surely by Lemma~\ref{lem:empirical-ranks-recover-copula}. The separation
\eqref{eq:continuous-separation} implies
\[
    d_K(\widehat\psi_n,\psi_0)\to0
    \qquad\text{almost surely}.
\]
Define \(h(\mathscr R_\infty)=\lim_{n\to\infty}\widehat\psi_n\) on the almost-sure event
where the limit exists, and define it arbitrarily elsewhere. Since \(\widehat\psi_n\) is
\(\mathscr R_n\)-measurable and \(\mathscr R_n\subseteq\mathscr R_\infty\), the limit is
\(\mathscr R_\infty\)-measurable. This proves the claim.
\end{proof}

\begin{lemma}[Doob consistency for the exact rank posterior]
\label{lem:doob-exact-rank}
Let \(\psi=[\boldsymbol{\Theta}]\in\Psi_K\). Suppose there exists an
\(\mathscr R_\infty\)-measurable map \(h:\mathscr R_\infty\to\Psi_K\) such that, under the
prior predictive law
\[
    \int P_\psi^\infty(\cdot)\,\bar\Pi(d\psi),
\]
we have \(h(\mathscr R_\infty)=\psi\) almost surely. Then there exists a set
\(\Psi^\star\subseteq\Psi_K\) with \(\bar\Pi(\Psi^\star)=1\) such that, for every
\(\psi_0\in\Psi^\star\) and every open neighborhood \(O\) of \(\psi_0\),
\[
    \bar\Pi(\psi\in O\mid\mathscr R_n)\longrightarrow1
\]
\(P_{\psi_0}^{\infty}\)-almost surely.
\end{lemma}

\begin{proof}
Because \(\Psi_K\) is compact metric, it has a countable base \(\mathcal B\). Work under
the joint prior predictive probability measure for \((\psi,\mathscr R_\infty)\), and let
\(E\) denote expectation under this joint law. For every Borel set \(B\subseteq\Psi_K\),
\[
    \bar\Pi(\psi\in B\mid\mathscr R_n)
    =
    E\{1(\psi\in B)\mid\mathscr R_n\}.
\]
For \(B\in\mathcal B\), the sequence
\[
    E\{1(\psi\in B)\mid\mathscr R_n\},
    \qquad n=1,2,\ldots,
\]
is a bounded martingale with respect to the increasing sigma-fields \(\mathscr R_n\). By the
martingale convergence theorem,
\[
    E\{1(\psi\in B)\mid\mathscr R_n\}
    \longrightarrow
    E\{1(\psi\in B)\mid\mathscr R_\infty\}
\]
almost surely under the prior predictive law. Since \(h(\mathscr R_\infty)=\psi\) almost
surely,
\[
    E\{1(\psi\in B)\mid\mathscr R_\infty\}
    =
    1\{h(\mathscr R_\infty)\in B\}.
\]
Intersecting over the countable base \(\mathcal B\) gives a prior-predictive probability-one
event on which this convergence holds for every \(B\in\mathcal B\).

By Fubini's theorem, there exists a set \(\Psi^\star\subseteq\Psi_K\) with
\(\bar\Pi(\Psi^\star)=1\) such that, for every \(\psi_0\in\Psi^\star\), the preceding
convergence holds \(P_{\psi_0}^{\infty}\)-almost surely. Fix such a \(\psi_0\), and let
\(O\) be an open neighborhood of \(\psi_0\). Choose \(B\in\mathcal B\) such that
\(\psi_0\in B\subseteq O\). On the event \(h(\mathscr R_\infty)=\psi_0\),
\[
    \bar\Pi(\psi\in B\mid\mathscr R_n)\to1.
\]
Since \(B\subseteq O\),
\[
    \bar\Pi(\psi\in O\mid\mathscr R_n)
    \ge
    \bar\Pi(\psi\in B\mid\mathscr R_n)
    \to1.
\]
This completes the proof.
\end{proof}

\begin{proof}[Proof of Theorem~\ref{thm:continuous-rank-posterior-consistency}]
Lemma~\ref{lem:exact-rank-likelihood-continuous} shows that, under continuous margins,
the rank likelihood is the exact likelihood of the rank data \(\mathscr R_n\). Hence the
quotient-space rank-likelihood posterior is the regular conditional posterior
\[
    \bar\Pi_n^R(\cdot\mid\mathscr R_n)
    =
    \bar\Pi(\cdot\mid\mathscr R_n).
\]
Lemma~\ref{lem:continuous-rank-recovery-map} constructs an
\(\mathscr R_\infty\)-measurable map \(h\) satisfying
\(h(\mathscr R_\infty)=\psi_0\) almost surely under every fixed
\(P_{\psi_0}^{\infty}\). In particular, \(h(\mathscr R_\infty)=\psi\) almost surely under the
prior predictive law. Applying Lemma~\ref{lem:doob-exact-rank} gives posterior consistency
for \(\bar\Pi\)-almost every \(\psi_0\).
\end{proof}

\subsection{Proofs and Additional Details for Section~\ref{subsubsec:rl-mixed-margins}}
\label{sec:supp-mixed-rank-details}

This section gives the details omitted from Section~\ref{subsubsec:rl-mixed-margins}: first,
why the extended rank likelihood is not an exact likelihood when ties occur; second, what
infinite weak-rank/tie information reveals; and third, the proof of
Theorem~\ref{thm:mixed-rank-generalized}.

\subsubsection{The extended rank likelihood is not an exact likelihood with ties}
\label{sec:supp-erl-not-exact}

Let
\[
    \boldsymbol U_i=(U_{i1},\ldots,U_{iJ}),
    \qquad i=1,\ldots,n,
\]
be i.i.d. from the copula \(C_\psi\). For a realized data array \(\boldsymbol Y^{(n)}\), the
extended rank set is
\[
    \mathcal D_n(\boldsymbol Y^{(n)})
    =
    \Bigl\{
        \boldsymbol u\in(0,1)^{n\times J}:
        y_{ij}<y_{\ell j}\Rightarrow u_{ij}<u_{\ell j},
        \quad i,\ell=1,\ldots,n,
        \ j=1,\ldots,J
    \Bigr\}.
\]
The marginal extended rank likelihood is
\[
    L_n^R(\psi;\boldsymbol Y^{(n)})
    =
    P_\psi\{\boldsymbol U^{(n)}\in\mathcal D_n(\boldsymbol Y^{(n)})\}.
\]
When all margins are continuous, ties occur with probability zero and the sets
\(\mathcal D_n(\boldsymbol Y^{(n)})\) are exactly the rank cells corresponding to the observed
coordinatewise ranks. These rank cells form a measurable partition of \((0,1)^{nJ}\), so
\(L_n^R\) is the exact likelihood of the observed ranks.

When ties can occur, tied observations impose no strict ordering constraints. Consequently,
\(\mathcal D_n(\boldsymbol Y^{(n)})\) is a rank-compatible superset rather than the event
corresponding exactly to the observed weak-rank/tie pattern. The sets associated with
different weak-rank/tie patterns are no longer disjoint and therefore cannot define a
probability mass function for those patterns.

\paragraph{Example S.1: one margin and two observations.}
Take \(n=2\) and \(J=1\). There are three weak-rank patterns:
\[
    Y_1<Y_2,
    \qquad
    Y_2<Y_1,
    \qquad
    Y_1=Y_2.
\]
The corresponding extended rank sets are
\[
    \mathcal D_< = \{(u_1,u_2):u_1<u_2\},
    \qquad
    \mathcal D_> = \{(u_1,u_2):u_2<u_1\},
\]
and
\[
    \mathcal D_= = (0,1)^2,
\]
because a tie imposes no strict inequality constraint. Since \(U_1\) and \(U_2\) are i.i.d.
uniform random variables,
\[
    P(\mathcal D_<)=\frac12,
    \qquad
    P(\mathcal D_>)=\frac12,
    \qquad
    P(\mathcal D_=)=1.
\]
These three numbers cannot be probabilities of the three mutually exclusive weak-rank
patterns, since they sum to \(2\), not \(1\).

The difference from the exact likelihood is already visible for a binary margin. Suppose
\[
    Y_i=1\{U_i>\tau\},
    \qquad 0<\tau<1.
\]
Then the exact probabilities of the three weak patterns are
\[
    P(Y_1<Y_2)=\tau(1-\tau),
    \qquad
    P(Y_2<Y_1)=\tau(1-\tau),
\]
and
\[
    P(Y_1=Y_2)=\tau^2+(1-\tau)^2.
\]
These probabilities depend on the marginal threshold \(\tau\), whereas the corresponding
extended-rank values are \(1/2\), \(1/2\), and \(1\). The extended rank likelihood is therefore
margin-free, but it is not an exact likelihood for the weak-rank/tie statistic.

\paragraph{Example S.2: two binary margins and two observations.}
Take \(n=2\) and \(J=2\), with binary margins
\[
    Y_{ij}=1\{U_{ij}>\tau_j\},
    \qquad j=1,2.
\]
Suppose the observed data are
\[
    \boldsymbol Y_1=(0,0),
    \qquad
    \boldsymbol Y_2=(1,0).
\]
The first margin implies \(Y_{11}<Y_{21}\), while the second margin is tied,
\(Y_{12}=Y_{22}\). Hence
\[
    \mathcal D_n(\boldsymbol Y^{(n)})
    =
    \{\boldsymbol u:u_{11}<u_{21}\},
\]
with no constraint on \(u_{12}\) or \(u_{22}\). Therefore
\[
    L_n^R(\psi;\boldsymbol Y^{(n)})
    =
    P_\psi(U_{11}<U_{21})
    =
    \frac12,
\]
because \(U_{11}\) and \(U_{21}\) are independent uniforms from two independent subjects.
By contrast, the exact probability of the observed binary data under fixed thresholds is
\[
\begin{aligned}
    &P_{\psi,F}\{\boldsymbol Y_1=(0,0),\boldsymbol Y_2=(1,0)\}  \\
    &\qquad =
    P_\psi(U_{11}\le\tau_1,U_{12}\le\tau_2)
    P_\psi(U_{21}>\tau_1,U_{22}\le\tau_2) \\
    &\qquad =
    C_\psi(\tau_1,\tau_2)
    \{\tau_2-C_\psi(\tau_1,\tau_2)\}.
\end{aligned}
\]
This exact probability depends on both the marginal thresholds and the copula. The extended
rank likelihood deliberately avoids the unknown thresholds, but it no longer coincides with
the exact likelihood of the observed data or of the observed weak-rank/tie pattern.

\subsubsection{Coarsened-rank identifiability}
\label{sec:supp-coarsened-identifiability}

For a fixed true margin vector \(F=(F_{Y_1},\ldots,F_{Y_J})\), define
\[
    T_{F_j}(u)=F_{Y_j}\{F_{Y_j}^{-1}(u)\},
    \qquad
    T_F(u_1,\ldots,u_J)=\{T_{F_1}(u_1),\ldots,T_{F_J}(u_J)\},
\]
and
\[
    Q_{\psi,F}=\mathcal L_\psi\{T_F(\boldsymbol U)\},
    \qquad
    \boldsymbol U\sim C_\psi.
\]
For a measurable map \(T\) and probability measure \(\mu\), let \(T\#\mu\) denote the
pushforward measure. Thus \(Q_{\psi,F}=T_F\# C_\psi\).

\begin{lemma}[Continuity of the coarsened copula law]
\label{lem:supp-coarsened-continuity}
For any fixed margin vector \(F\), the map
\[
    \psi\mapsto Q_{\psi,F}
\]
is continuous from \(\Psi_K\) to the space of probability measures on \([0,1]^J\) equipped
with the bounded-Lipschitz metric.
\end{lemma}

\begin{proof}
Let \(\psi_m\to\psi\) in \(\Psi_K\). By the continuity of the CopulaDDE map established in
Section~\ref{subsec:rl-identifiability},
\[
    C_{\psi_m}\Rightarrow C_\psi.
\]
The ordinary continuous mapping theorem cannot be applied directly because \(T_F\) may be
discontinuous when some margins have atoms. We verify the condition for the extended
continuous mapping theorem.

For margin \(j\), let
\[
    \mathcal B_j
    =
    \{F_{Y_j}(c^-),F_{Y_j}(c): F_{Y_j}(c)-F_{Y_j}(c^-)>0\}
\]
be the set of left and right probability boundaries associated with atoms of \(F_{Y_j}\). This
set is countable. The discontinuity set of \(T_F\) is contained in
\[
    \operatorname{Disc}(T_F)
    \subseteq
    \bigcup_{j=1}^J
    \{\boldsymbol u\in[0,1]^J:u_j\in\mathcal B_j\}.
\]
Since \(C_\psi\) is a copula, each coordinate \(U_j\) is uniform on \([0,1]\). Therefore, for
any \(b\in[0,1]\),
\[
    C_\psi\{\boldsymbol u:u_j=b\}=P_\psi(U_j=b)=0.
\]
By countability of \(\mathcal B_j\) and finiteness of \(J\),
$C_\psi\{\operatorname{Disc}(T_F)\}=0.$
The extended continuous mapping theorem then gives
$T_F\# C_{\psi_m}\Rightarrow T_F\# C_\psi,$
which is exactly
$Q_{\psi_m,F}\Rightarrow Q_{\psi,F}.$
On the compact space \([0,1]^J\), weak convergence is equivalent to convergence in the
bounded-Lipschitz metric. Hence \(d_{\mathrm{BL}}(Q_{\psi_m,F},Q_{\psi,F})\to0\).
\end{proof}

\begin{assumption}[Coarsened-rank separation]
\label{ass:supp-coarsened-rank-separation}
For the true margin vector \(F\), the map
$\psi\mapsto Q_{\psi,F}$
is injective on \(\Psi_K\).
\end{assumption}

\begin{theorem}[Infinite-rank identifiability under coarsening]
\label{thm:supp-coarsened-rank-identifiability}
Suppose Assumption~\ref{ass:rank-compact-prior} and
Assumption~\ref{ass:supp-coarsened-rank-separation} hold. Let \((\psi_0,F)\) be the true
quotient parameter and margin pair. Then the infinite weak-rank/tie information identifies
\(\psi_0\) conditional on the true marginal coarsening \(F\): there exists an
\(\mathscr R_\infty\)-measurable map
$h_F:\mathscr R_\infty\to\Psi_K$
such that
$h_F(\mathscr R_\infty)=\psi_0$, $P_{\psi_0,F}^{\infty}$-almost surely.
\end{theorem}

\begin{proof}
Define the empirical marginal probability labels
\[
    \widehat W_{nij}
    =
    \frac{1}{n}\sum_{\ell=1}^n1(Y_{\ell j}\le Y_{ij}),
    \qquad
    W_{ij}=F_{Y_j}(Y_{ij})=T_{F_j}(U_{ij}),
\]
and
\[
    \widehat{\boldsymbol W}_{ni}
    =
    (\widehat W_{ni1},\ldots,\widehat W_{niJ}),
    \qquad
    \boldsymbol W_i=(W_{i1},\ldots,W_{iJ}).
\]
The quantities \(\widehat W_{nij}\) are measurable with respect to the weak-rank/tie
sigma-field \(\mathscr R_n\), because they are empirical CDF values at the observed data
points.

Let
\[
    \widehat Q_n
    =
    \frac{1}{n}\sum_{i=1}^n\delta_{\widehat{\boldsymbol W}_{ni}},
    \qquad
    Q_n^*
    =
    \frac{1}{n}\sum_{i=1}^n\delta_{\boldsymbol W_i}.
\]
For each margin \(j\), the Glivenko--Cantelli theorem gives
\[
    \sup_y |\widehat F_{nj}(y)-F_{Y_j}(y)|\to0
    \qquad\text{almost surely},
\]
where \(\widehat F_{nj}(y)=n^{-1}\sum_{\ell=1}^n1(Y_{\ell j}\le y)\). Hence
\[
    \max_{1\le i\le n}|\widehat W_{nij}-W_{ij}|
    \le
    \sup_y |\widehat F_{nj}(y)-F_{Y_j}(y)|
    \to0
    \qquad\text{almost surely}.
\]
For any bounded Lipschitz function \(f:[0,1]^J\to\mathbb R\) with Lipschitz constant at most
one,
\[
\begin{aligned}
    \left|\int f\,d\widehat Q_n-\int f\,dQ_n^*\right|
    &\le
    \frac{1}{n}\sum_{i=1}^n
    \|\widehat{\boldsymbol W}_{ni}-\boldsymbol W_i\|_1 \\
    &\le
    \sum_{j=1}^J
    \sup_y |\widehat F_{nj}(y)-F_{Y_j}(y)|
    \to0
\end{aligned}
\]
almost surely. Therefore
\[
    d_{\mathrm{BL}}(\widehat Q_n,Q_n^*)\to0
    \qquad\text{almost surely}.
\]
Since \(\boldsymbol W_i=T_F(\boldsymbol U_i)\) are i.i.d. from \(Q_{\psi_0,F}\), the strong
law for empirical measures gives
\[
    d_{\mathrm{BL}}(Q_n^*,Q_{\psi_0,F})\to0
    \qquad\text{almost surely}.
\]
Thus
\[
    d_{\mathrm{BL}}(\widehat Q_n,Q_{\psi_0,F})\to0
    \qquad\text{almost surely}.
\]

We now construct a measurable minimum-distance recovery map. Since \(\Psi_K\) is compact
metric, let \(\{\psi_m:m\ge1\}\) be a countable dense subset. For each \(n\), define
\(\widehat\psi_n\) to be the first \(\psi_m\) satisfying
\[
    d_{\mathrm{BL}}(\widehat Q_n,Q_{\psi_m,F})
    \le
    \inf_{r\ge1}d_{\mathrm{BL}}(\widehat Q_n,Q_{\psi_r,F})+\frac{1}{n}.
\]
Then \(\widehat\psi_n\) is \(\mathscr R_n\)-measurable. By density of
\(\{\psi_m\}\) and Lemma~\ref{lem:supp-coarsened-continuity},
\[
    \inf_{r\ge1}d_{\mathrm{BL}}(\widehat Q_n,Q_{\psi_r,F})
    =
    \inf_{\psi\in\Psi_K}d_{\mathrm{BL}}(\widehat Q_n,Q_{\psi,F})
    \le
    d_{\mathrm{BL}}(\widehat Q_n,Q_{\psi_0,F}).
\]
Therefore
\[
\begin{aligned}
    d_{\mathrm{BL}}(Q_{\widehat\psi_n,F},Q_{\psi_0,F})
    &\le
    d_{\mathrm{BL}}(Q_{\widehat\psi_n,F},\widehat Q_n)
    +d_{\mathrm{BL}}(\widehat Q_n,Q_{\psi_0,F}) \\
    &\le
    2d_{\mathrm{BL}}(\widehat Q_n,Q_{\psi_0,F})+\frac{1}{n}
    \to0
\end{aligned}
\]
almost surely.

By Lemma~\ref{lem:supp-coarsened-continuity} and
Assumption~\ref{ass:supp-coarsened-rank-separation}, the continuous map
\(\psi\mapsto Q_{\psi,F}\) is injective on compact \(\Psi_K\). Hence, for every
\(\epsilon>0\),
\[
    \Delta_\epsilon(\psi_0)
    :=
    \inf_{\psi:\,d_K(\psi,\psi_0)\ge\epsilon}
    d_{\mathrm{BL}}(Q_{\psi,F},Q_{\psi_0,F})
    >0.
\]
The convergence
\(d_{\mathrm{BL}}(Q_{\widehat\psi_n,F},Q_{\psi_0,F})\to0\) therefore implies
\[
    d_K(\widehat\psi_n,\psi_0)\to0
    \qquad\text{almost surely}.
\]
Define
\[
    h_F(\mathscr R_\infty)=\lim_{n\to\infty}\widehat\psi_n
\]
on the almost-sure event where the limit exists, and define it arbitrarily elsewhere. Since
each \(\widehat\psi_n\) is \(\mathscr R_n\)-measurable and
\(\mathscr R_n\subseteq\mathscr R_\infty\), the limit is
\(\mathscr R_\infty\)-measurable. This proves the theorem.
\end{proof}

\begin{assumption}[Extended-rank contrast consistency]
\label{ass:supp-erl-contrast}
Fix the true pair \((\psi_0,F)\), where \(\psi_0=[\boldsymbol{\Theta}_0]\in\Psi_K\) is
the true canonical DDE copula equivalence class and \(F\) is the true margin vector. Let
\[
    \ell_n^R(\psi)=\log L_n^R(\psi;\boldsymbol Y^{(n)}),
    \qquad \psi\in\Psi_K .
\]
There exists a finite continuous function
$M_{\psi_0,F}:\Psi_K\to\mathbb R$
such that:
\begin{enumerate}
\item[(a)] Uniform contrast convergence holds:
\[
    \sup_{\psi\in\Psi_K}
    \left|
        \frac1n\{\ell_n^R(\psi)-\ell_n^R(\psi_0)\}
        -
        \{M_{\psi_0,F}(\psi)-M_{\psi_0,F}(\psi_0)\}
    \right|
    \longrightarrow 0
\]
almost surely under \(P^\infty_{\psi_0,F}\).

\item[(b)] The limiting contrast has a unique maximizer:
\[
    M_{\psi_0,F}(\psi)<M_{\psi_0,F}(\psi_0)
    \qquad
    \text{for every } \psi\ne\psi_0 .
\]
\end{enumerate}
\end{assumption}

\begin{lemma}
\label{lem:supp-erl-contrast-implies-separation}
Suppose Assumption~\ref{ass:rank-compact-prior} and
Assumption~\ref{ass:supp-erl-contrast} hold. Then, for every \(\epsilon>0\),
\[
    \sup_{\psi:\ d_K(\psi,\psi_0)\ge\epsilon}
    \{M_{\psi_0,F}(\psi)-M_{\psi_0,F}(\psi_0)\}<0 .
\]
Moreover, for every \(\eta>0\),
$\bar\Pi\!\left(
        \psi:
        M_{\psi_0,F}(\psi)-M_{\psi_0,F}(\psi_0)>-\eta
    \right)>0$.
\end{lemma}

\begin{proof}
For fixed \(\epsilon>0\), let
\[
    N_\epsilon=\{\psi\in\Psi_K:d_K(\psi,\psi_0)\ge\epsilon\}.
\]
If \(N_\epsilon\) is empty, the first assertion in the lemma is immediate. Otherwise, \(N_\epsilon\) is
compact because \(\Psi_K\) is compact. Since \(M_{\psi_0,F}\) is continuous, it attains its
maximum on \(N_\epsilon\). By the unique maximizer condition in
Assumption~\ref{ass:supp-erl-contrast}, no point in \(N_\epsilon\) can attain the value
\(M_{\psi_0,F}(\psi_0)\). Hence the strict separation inequality holds.

For the prior-mass statement, continuity of \(M_{\psi_0,F}\) at \(\psi_0\) implies that for
every \(\eta>0\) there exists \(\delta>0\) such that
\[
    d_K(\psi,\psi_0)<\delta
    \quad\Longrightarrow\quad
    M_{\psi_0,F}(\psi)-M_{\psi_0,F}(\psi_0)>-\eta .
\]
By the prior-support condition in Assumption~\ref{ass:rank-compact-prior}, every nonempty
open ball in \(\Psi_K\) has positive \(\bar\Pi\)-probability. Therefore the near-maximizer
set has positive prior mass.
\end{proof}

\begin{remark}[Interpretation of the contrast condition]
Assumption~\ref{ass:supp-erl-contrast} requires that, asymptotically, the extended rank
likelihood assigns a strictly better normalized log score to the true DDE copula equivalence
class than to any separated alternative. This is stronger than identifying \(\psi_0\) from
the coarsened law \(Q_{\psi,F}\). The latter says that infinite weak ranks contain enough
information to distinguish the parameter; the former says that the particular generalized
likelihood used for inference exploits this information in a uniformly separating way.
\end{remark}

\begin{remark}[Relation to probit latent-factor identifiability]
Identifiability results for probit latent-factor models provide useful context for
mixed-margin DDE copulas. For dichotomous probit bifactor models, existing work shows that
thresholds and tetrachoric correlations can carry information for identifying latent loading
structures under suitable item-replication and separation conditions
\citep{fang2021identifiability}. These results do not directly verify
Assumption~\ref{ass:supp-erl-contrast}, because DDE copula uses a multilayer discrete latent
hierarchy and the extended rank likelihood is a generalized likelihood. They nevertheless
support the broader principle that coarse ordinal or binary observations can retain enough
information to identify latent dependence parameters when the measurement graph contains
sufficient replication and separation.
\end{remark}


\subsubsection{Proof of Theorem~\ref{thm:mixed-rank-generalized}}
\label{sec:supp-proof-mixed-rank-generalized}

\begin{proof}[Proof of Theorem~\ref{thm:mixed-rank-generalized}]
Let
\[
    N_\epsilon=\{\psi\in\Psi_K:d_K(\psi,\psi_0)\ge\epsilon\}.
\]
If \(N_\epsilon\) is empty, the result is immediate. Otherwise, Lemma~\ref{lem:supp-erl-contrast-implies-separation}
gives \(c_\epsilon>0\) such that
\[
    \sup_{\psi\in N_\epsilon}
    \{M_{\psi_0,F}(\psi)-M_{\psi_0,F}(\psi_0)\}
    \le -c_\epsilon.
\]
Choose \(\eta<c_\epsilon/4\), and define
\[
    V_\eta
    =
    \{\psi\in\Psi_K:
      M_{\psi_0,F}(\psi)-M_{\psi_0,F}(\psi_0)>-\eta\}.
\]
By Lemma~\ref{lem:supp-erl-contrast-implies-separation}, \(\bar\Pi(V_\eta)>0\).

By Assumption~\ref{ass:supp-erl-contrast}(a), almost surely, for all sufficiently large
\(n\),
\[
    \sup_{\psi\in N_\epsilon}
    \frac{1}{n}\{\ell_n^R(\psi)-\ell_n^R(\psi_0)\}
    \le
    -\frac{3c_\epsilon}{4},
\]
and
\[
    \inf_{\psi\in V_\eta}
    \frac{1}{n}\{\ell_n^R(\psi)-\ell_n^R(\psi_0)\}
    \ge
    -2\eta.
\]
Therefore the posterior numerator outside the \(\epsilon\)-ball satisfies
\[
    \int_{N_\epsilon}\exp\{\ell_n^R(\psi)\}\,\bar\Pi(d\psi)
    \le
    \exp\{\ell_n^R(\psi_0)-3nc_\epsilon/4\},
\]
where we used \(\bar\Pi(N_\epsilon)\le1\). The denominator satisfies
\[
\begin{aligned}
    \int_{\Psi_K}\exp\{\ell_n^R(\psi)\}\,\bar\Pi(d\psi)
    &\ge
    \int_{V_\eta}\exp\{\ell_n^R(\psi)\}\,\bar\Pi(d\psi) \\
    &\ge
    \bar\Pi(V_\eta)
    \exp\{\ell_n^R(\psi_0)-2n\eta\}.
\end{aligned}
\]
Consequently,
\[
    \bar\Pi_n^R(N_\epsilon\mid\boldsymbol Y^{(n)})
    \le
    \bar\Pi(V_\eta)^{-1}
    \exp\{-n(3c_\epsilon/4-2\eta)\}.
\]
Since \(\eta<c_\epsilon/4\), the exponent is strictly negative, and the right-hand side
converges to zero. Hence
\[
    \bar\Pi_n^R
    \{\psi:d_K(\psi,\psi_0)<\epsilon\mid\boldsymbol Y^{(n)}\}
    \to1
\]
almost surely under \(P_{\psi_0,F}^{\infty}\).
\end{proof}

\section{Detailed Algorithms}\label{sec:alg}
\subsection{Tempering to Reduce Multimodality}

We first review how to apply tempering, which is mentioned in Section \ref{sec:EM}. For the latent binary updates, all log-odds $\Delta_{ik}^{(d),t}$ are rescaled as $\tilde{\Delta}_{ik}^{(d),t} \;\leftarrow\; \tau \, \Delta_{ik}^{(d),t}$, which shrinks probabilities toward $1/2$ for $\tau < 1$ and promotes exploration across configurations.  For the Gaussian latent variables under the rank likelihood, tempering corresponds to inflating the conditional variances,
$Z_{ij} \mid (\text{--}) \sim \mathrm{TN}(\mu_{ij}, \gamma_j / \tau; L_{ij}, U_{ij}),$
which further smooths the latent landscape while preserving rank constraints. 

In the maximization step, tempering rescales the contribution of the likelihood in $\widehat Q(\boldsymbol\theta \mid \boldsymbol\theta^t)$ by $\tau$, which is equivalent to applying the same scaling within the regression-based updates for $\boldsymbol B^{(d)}$ and $\boldsymbol \gamma$.

We employ a deterministic schedule $\{\tau_t\}_{t \geq 0}$ that gradually increases from a moderate value (e.g., $\tau_0 \approx 0.7$) to $\tau=1$, at which point the algorithm targets the approximation to the RL posterior provided in Algorithm \ref{alg:mcem}. Early iterations thus emphasize exploration of the posterior landscape, while later iterations recover the untempered objective for accurate estimation. Empirically, this strategy substantially improves convergence and reduces sensitivity to initialization in deeper DDE models.
\ref{sec:sim}. 
\subsection{Coordinate Ascent Monte Carlo EM for the DDE Copula}
We provide detailed updates and sampling steps for Algorithm \ref{alg:mcem} in the main paper, as well as the rank likelihood SAEM algorithm described in the simulations of Section 
Here, we describe the algorithm for a $D = 2$ layer DDE copula. The steps may be generalized for $D > 2$. Let
\begin{equation}
\boldsymbol{\theta}
=
\Bigl(
\boldsymbol{\eta,\,
B^{(1)},\,
B^{(2)},\,
\gamma}, \{c_{k}^{(1)}\}_{k=1}^{K_{\text{max}}^{(1)}},\,
\{c_{k}^{(2)}\}_{k=1}^{K_{\text{max}}^{(2)}}\Bigr)
\end{equation}
denote the collection of model parameters. At iteration $t$, given the current state
$\boldsymbol{\theta}^{(t)}$, the coordinate-ascent Monte Carlo EM algorithm alternates between a Monte Carlo E-step and a conditional maximization step. For each subject $i=1,\dots,N$, let
\[
\boldsymbol{A}_i^{(2)} \in \{0,1\}^{K_{\text{max}}^{(1)}},
\qquad
\boldsymbol{A}_i^{(1)} \in \{0,1\}^{K_{\text{max}}^{(2)}},
\]
denote the two layers of binary latent variables. We augment each layer with an intercept entry equal to one:
\[
\widetilde{\boldsymbol{A}}_i^{(1)} = (1, \boldsymbol{A}_i^{(1)}),
\qquad
\widetilde{\boldsymbol{A}}_i^{(2)} = (1, \boldsymbol{A}_i^{(2)}).
\]
The latent Gaussian copula variable for observed feature $j$ is denoted $Z_{ij}$, and the residual variance for feature $j$ is $\gamma_j$.

\bigskip

\noindent
\textbf{Monte Carlo E-step.}
Given $\boldsymbol \theta^{(t)}$, perform the following updates. For all of our implementations, we draw $C = 1$ samples in each step.

\begin{itemize}
    \item \textbf{Sample the binary latent variables} $\{\boldsymbol {A}_{i}^{(1)},\boldsymbol{A}_{i}^{(2)}\}_{i=1}^{N}$ 
        \begin{itemize}
    \item \textbf{Top layer ($d = D$):}
    \begin{align}
    \Delta_{ik}^{(D),t} &=
    \log \frac{\pi_k^t}{1-\pi_k^t}
    +
    \log p\!\left(\boldsymbol A_i^{(D-1),t-1} \mid A_{ik}^{(D)} = 1, \boldsymbol A_{i,-k}^{(D),t-1}, \boldsymbol \theta^t\right)\\
    &-
    \log p\!\left(\boldsymbol A_i^{(D-1),t} \mid A_{ik}^{(D)} = 0, \boldsymbol A_{i,-k}^{(D),t-1}, \boldsymbol \theta^t\right).
    \end{align}

    \item \textbf{Intermediate layers ($2 \le d \le D-1$):}
    \[
    \Delta_{ik}^{(d),t} =
    \log p\!\left(A_{ik}^{(d)} = 1 \mid \boldsymbol A_i^{(d+1),t-1}, \boldsymbol \theta^t\right)
    -
    \log p\!\left(A_{ik}^{(d)} = 0 \mid \boldsymbol A_i^{(d+1),t-1}, \boldsymbol \theta^t\right)
    \]
    \[
    \quad +
    \log p\!\left(\boldsymbol A_i^{(d-1),t-1} \mid A_{ik}^{(d)} = 1, \boldsymbol A_{i,-k}^{(d),t-1}, \boldsymbol \theta^t\right)
    -
    \log p\!\left(\boldsymbol A_i^{(d-1),t-1} \mid A_{ik}^{(d)} = 0, \boldsymbol A_{i,-k}^{(d),t-1}, \boldsymbol \theta^t\right).
    \]

    \item \textbf{Bottom layer ($d = 1$):}
    \[
    \Delta_{ik}^{(1),t} =
    \log p\!\left(A_{ik}^{(1)} = 1 \mid \boldsymbol A_i^{(2),t-1}, \boldsymbol \theta^t\right)
    -
    \log p\!\left(A_{ik}^{(1)} = 0 \mid \boldsymbol A_i^{(2),t-1}, \boldsymbol \theta^t\right)
    \]
    \[
    \quad +
    \log p\!\left(\boldsymbol Z_i^{t} \mid A_{ik}^{(1)} = 1, \boldsymbol A_{i,-k}^{(1),t-1}, \boldsymbol \theta^t\right)
    -
    \log p\!\left(\boldsymbol Z_i^{t} \mid A_{ik}^{(1)} = 0, \boldsymbol A_{i,-k}^{(1),t-1}, \boldsymbol \theta^t\right).
    \]
\end{itemize}
    \item  \textbf{Sample the rank likelihood latent Gaussian variables} $\boldsymbol Z$ via Algorithm \ref{alg:rank_aug}.

    \item \textbf{Update the CUSP column-allocation variables for layer 2.}  
    
    For each column $h=1,\dots,K^{(2)}_{\text{max}}$, let $x_k^{(2),t}=B_{\cdot,h+1}^{(2),t}$ denote the $h$th non-intercept column of $\boldsymbol{B}^{(2),t}$. Define
    \[
    \log p_{\mathrm{spike}}^{(2)}(x_k^{(2),t})
    =
    -K^{(1)}_{\text{max}} \log(2\lambda_{1}^{(2)})
    -
    \frac{\|x_k^{(2),t}\|_1}{\lambda_{1}^{(2)}},
    \]
    and
    \[
    \log p_{\mathrm{slab}}^{(2)}(x_k^{(2),t})
    =
    \log 5
    -
    K^{(1), t}_{\text{active}}\log 2
    +
    \log\Gamma(1+K^{(1),t}_{\text{active}})
    -
    \log\Gamma(1)
    -
    (1+K^{(1),t}_{\text{active}})\log\!\bigl(5+\|x_k^{(2),t}\|_1\bigr).
    \]
    This derived by marginalizing the Laplace slab density with rate $\lambda^{(2)}_{0k}$, where $\lambda^{(2)}_{0k}\sim\text{Gamma}(5,1)$. $\Gamma(.)$ is the gamma function. Then for each allocation state $k=1,\dots,K^{(2)}_{\text{max}}+1$,
    \[
    \log \Pr(c_k^{(2),t}=\ell \mid \boldsymbol{B}^{(2),t},\omega^{(2),t})
    \propto
    \log \omega_k^{(2),t}
    +
    \begin{cases}
    \log p_{\mathrm{spike}}^{(2)}(x_k^{(2),t}), & \ell \le k,\\[0.3em]
    \log p_{\mathrm{slab}}^{(2)}(x_k^{(2),t}), & \ell > k.
    \end{cases}
    \]
    The implementation uses the MAP update
    \[
    c_k^{(2),t}
    =
    \mbox{argmax}_{\ell \in\{1,\dots,K^{(2)}_{\text{max}} + 1\}}
    \Pr(c_k^{(2),t}=\ell \mid \boldsymbol{B}^{(2),t},\omega^{(2),t}).
    \]

    \item \textbf{Update the CUSP stick-breaking weights for layer 2.}  
    Let
    \[
    n_\ell^{(2),t} = \sum_{h=1}^{K_{\text{max}}^{(2)}}\mathbf 1\{c_k^{(2),t}=\ell\},
    \qquad
    n_{>\ell}^{(2),t} = \sum_{h=1}^{K_{\text{max}}^{(2)}}\mathbf 1\{c_k^{(2),t}>\ell\}.
    \]
    Then for $\ell=1,\dots,K^{(2)}_{\text{max}}$,
    \[
    v_\ell^{(2),t}
    \sim
    \mathrm{Beta}\!\bigl(1+n_\ell^{(2),t},\, \alpha^{(2)} + n_{>\ell}^{(2),t}\bigr),
    \qquad \alpha^{(2)} = K^{(2)}_{\text{max}},
    \]
    and set $v_{K^{(2)}_{\text{max}} + 1}^{(2),t}=1$. The corresponding mixture weights are
    \[
    \omega_1^{(2),t}=v_1^{(2),t},
    \qquad
    \omega_\ell^{(2),t}
    =
    v_\ell^{(2),t}\prod_{m=1}^{\ell-1}(1-v_m^{(2),t}),
    \quad \ell=2,\dots,K^{(2)}_{\text{max}}+1.
    \]

    \item \textbf{Update the slab scales for layer 2.}  
    For each $h=1,\dots,K^{(2)}_{\text{max}}$ such that $c_k^{(2),t}>h$, sample
    \[
    \lambda_{1,k}^{(2),t}
    \sim
    \mathrm{Gamma}\!\left(
    1+K^{(1),t}_{\text{active}},\,
    5+\|\boldsymbol{B}_{\cdot,h+1}^{(2),t}\|_1
    \right),
    \]
    where the second argument is the rate parameter. Then define the penalty matrix
    \[
    \lambda_{kh}^{(2),t}
    =
    \begin{cases}
    \lambda_{1}^{(2)}, & c_k^{(2),t}\le h,\\
    \lambda_{1,k}^{(2),t}, & c_k^{(2),t}>h.
    \end{cases}
    \]

    \item \textbf{Sample the CUSP column-allocation variables for layer 1.}  
    For each column $k=1,\dots,K^{(1)}_{\text{max}}$, let $x_k^{(1),t}=\boldsymbol{B}_{\cdot,h+1}^{(1),t}$. Define
    \[
    \log p_{\mathrm{spike}}^{(1)}(x_k^{(1),t})
    =
    -J \log(2\lambda_{1}^{(1)})
    -
    \frac{\|x_k^{(1),t}\|_1}{\lambda_{1}^{(1)}},
    \]
    and
    \[
    \log p_{\mathrm{slab}}^{(1)}(x_k^{(1),t})
    =
    \log 5
    -
    J\log 2
    +
    \log\Gamma(1+J)
    -
    \log\Gamma(1)
    -
    (1+J)\log\!\bigl(5+\|x_k^{(1),t}\|_1\bigr).
    \]
    Then for $\ell=1,\dots,K^{(1)}_{\text{max}}+1$,
    \[
    \log \Pr(c_k^{(1),t}=\ell \mid \boldsymbol{B}^{(1),t},\omega^{(1),t})
    \propto
    \log \omega_\ell^{(1),t}
    +
    \begin{cases}
    \log p_{\mathrm{spike}}^{(1)}(x_k^{(1),t}), & \ell \le k,\\[0.3em]
    \log p_{\mathrm{slab}}^{(1)}(x_k^{(1),t}), & \ell > k.
    \end{cases}
    \]
    Again, the implementation uses the MAP choice
    \[
    c_k^{(1),t}
    =
    \arg\max_{\ell\in\{1,\dots,K^{(1)}_{\text{max}}+1\}}
    \Pr(c_k^{(1),t}=\ell \mid \boldsymbol{B}^{(1),t},\omega^{(1),t}).
    \]

    \item \textbf{Update the CUSP stick-breaking weights and slab scales for layer 1.}  
    Let
    \[
    n_\ell^{(1),t} = \sum_{h=1}^{K_{\text{max}}^{(1)}}\mathbf 1\{c_k^{(1),t}=\ell\},
    \qquad
    n_{>\ell}^{(1),t} = \sum_{h=1}^{K_{\text{max}}^{(1)}}\mathbf 1\{c_k^{(1),t}>\ell\}.
    \]
    Then for $\ell=1,\dots,K^{(1)}_{\text{max}}$,
    \[
    v_\ell^{(1),t}
    \sim
    \mathrm{Beta}\!\bigl(1+n_\ell^{(1),t},\, \alpha^{(1)} + n_{>\ell}^{(1),t}\bigr),
    \qquad \alpha^{(1)} = K^{(1)}_{\text{max}},
    \]
    with $v_{K^{(1)}_{\text{max}}+1}^{(1),t}=1$, and
    \[
    \omega_1^{(1),t}=v_1^{(1),t},
    \qquad
    \omega_\ell^{(1),t}
    =
    v_\ell^{(1),t}\prod_{m=1}^{\ell-1}(1-v_m^{(1),t}),
    \quad \ell=2,\dots,K^{(1)}_{\text{max}}+1.
    \]
    For each active slab column $h$ such that $c_k^{(1),t}>h$,
    \[
    \lambda_{0,k}^{(1),t}
    \sim
    \mathrm{Gamma}\!\left(
    1+J,\,
    5+\|\boldsymbol{B}_{\cdot,k+1}^{(1),t}\|_1
    \right),
    \]
    and
    \[
    \lambda_{jk}^{(1),t}
    =
    \begin{cases}
    \lambda_{1}^{(1)}, & c_k^{(1),t}\le k,\\
    \lambda_{0,k}^{(1),t}, & c_k^{(1),t}>k.
    \end{cases}
    \]
\end{itemize}

\bigskip

\noindent
\textbf{Conditional M-step.}
Conditional on the Monte Carlo draws from the E-step, maximize the surrogate objective $\widehat Q(\theta\mid \theta^{(t)})$ blockwise.

\begin{itemize}
    \item \textbf{Update the top-layer Bernoulli probabilities.}  
    The complete-data conditional log-likelihood for $\pi=(\pi_1,\dots,\pi_{K_{\text{max}}}^{(2)})$ is
    \[
    \sum_{i=1}^N \sum_{l=1}^{K_{\text{max}}^{(2)}}
    \Bigl[
    A_{il}^{(2)} \log \pi_l
    +
    (1-A_{il}^{(2)})\log(1-\pi_l)
    \Bigr].
    \]
    Replacing $A_{il}^{(2)}$ by its Monte Carlo expectation $q_{il}^{(t)}$ yields
    \[
    \pi_l^{(t+1)}
    =
    \arg\max_{\pi_l\in(0,1)}
    \sum_{i=1}^N
    \Bigl[
    q_{il}^{(t)} \log \pi_l
    +
    (1-q_{il}^{(t)})\log(1-\pi_l)
    \Bigr]
    =
    \frac{1}{N}\sum_{i=1}^N q_{il}^{(t)}.
    \]

    \item \textbf{Update the first-layer loading matrix $B^{(1)}$.}  
    For each observed variable $j=1,\dots,J$, define the Monte Carlo Gaussian objective for $C$ simulations from the conditional posterior in the MC E-step as
    \[
    \ell_j^{(1)}(b)
    =
    \frac{1}{2\gamma_j^{(t)}}
    \sum_{i=1}^N \sum_{c=1}^C
    \left(
    Z_{ij}^{(t)}
    -
    \widetilde{\boldsymbol{A}}_i^{(1),(c),t} \boldsymbol{B}_{j.}^{(1),t^{\top}}
    \right)^2.
    \]
    The adaptive weighted $\ell_1$ penalty is
    \[
    \mathcal P_j^{(1)}(b)
    =
    \sum_{k=1}^{K_{\text{max}}^{(1)}}
    |\boldsymbol{B}_{j k+1}|    1/\lambda_{jk}^{(1),t}.
    \]
    Thus the row update is
    \[
    \boldsymbol{B}_{j,\cdot}^{(1),t+1}
    =
    \arg\min_{b\in\mathbb R^{K^{(1)}_{\text{max}}+1}}
    \left\{
    \ell_j^{(1)}(b)
    +
    \mathcal P_j^{(1)}(b)
    \right\}.
    \]

    \item \textbf{Update the residual variances $\gamma_j$.}  
    Given the updated row $\boldsymbol{B}_{j,\cdot}^{(1),t+1}$, define
    \[
    \eta_{ij}^{(c),t+1}
    =
    \widetilde A_i^{(1),(c),t}
    \boldsymbol{B}_{j,\cdot}^{(1),t+1\top}.
    \]
    With  $\gamma_{j} \sim \mathrm{Gamma}(1,1)$, the conditional MAP update is
    \[
    \gamma_j^{(t+1)}
    =
    \frac{
    1 + \frac12 \sum_{i=1}^N \sum_{c=1}^C
    \bigl(Z_{ij}^{(t)}-\eta_{ij}^{(c),t+1}\bigr)^2
    }{
    1 + \frac12 NC + 1
    }.
    \]

    \item \textbf{Update the second-layer loading matrix $B^{(2)}$.}  
    For each hidden unit $k=1,\dots,K^{(1)}_{\text{max}}$, define the Monte Carlo logistic objective
    \[
    \ell_k^{(2)}(b)
    =
    -
    \sum_{i=1}^N \sum_{c=1}^C
    \left[
    A_{ik}^{(1),(c),t}
    \,\widetilde A_i^{(2),(c),t}B^{(2),t^{\top}}_{k.}
    -
\log\!\left(1+\exp\bigl(\widetilde A_i^{(2),(c),t} \boldsymbol{B}_{k\cdot}^{(2),t\top}\bigr)\right)
    \right].
    \]
    The adaptive weighted $\ell_1$ penalty is
    \[
    \mathcal P_k^{(2)}(b)
    =
    \sum_{l=1}^{K_{\text{max}}^{(2)}}
    |B^{(2)}_{kl+1}|1/\lambda_{kl}^{(2),t}.
    \]
    Hence
    \[
    \boldsymbol{B}_{k,\cdot}^{(2),t+1}
    =
    \arg\min_{b\in\mathbb R^{K^{(2)}_{\text{max}}+1}}
    \left\{
    \ell_k^{(2)}(b)
    +
    \mathcal P_k^{(2)}(b)
    \right\},
    \]
    again solved numerically with \texttt{fmincon}.

    \item \textbf{Threshold small coefficients.}  
    After the numerical updates, apply elementwise thresholding to the non-intercept coefficients:
    \[
    \boldsymbol{B}_{\cdot,2:(K^{(1)}_{\text{max}}+1)}^{(1),t+1}
    \leftarrow
    \operatorname{thres}\!\left(
    \boldsymbol{B}_{\cdot,2:(K^{(1)}_{\text{max}}+1)}^{(1),t+1},\, \xi
    \right),
    \]
    \[
    \boldsymbol{B}_{\cdot,2:(K^{(2)}_{\text{max}}+1)}^{(2),t+1}
    \leftarrow
    \operatorname{thres}\!\left(
    \boldsymbol{B}_{\cdot,2:(K^{(2)}_{\text{max}}+1)}^{(2),t+1},\, \xi
    \right).
    \]
Here, the $ \operatorname{thres}(.)$ function maps any entry of the matrix less than $\xi$ to zero, which speeds up convergence of the algorithm. We find this step necessary in order to minimize the noise injection from the Monte Carlo E-step. In practice, we tune $\xi$ based on the predictive evaluation procedure outlined in Section \ref{sec:realdat}. See Section \ref{sec:hyp} for suggested default values. 
    \item \textbf{Enforce hierarchical gating implied by the CSP activations in layer 1.}  
    Let
    \[
    \mathcal I_2^{(t)} = \{k : c_k^{(1),t}\le h\},
    \] Then,
    $\boldsymbol{B}_{k,.}^{(2),t+1}=0
    \quad \text{for } k\in\mathcal I_2^{(t)}.$

    \item \textbf{Update binary point estimates.}  
    After drawing $C$ Monte Carlo replicates, the binary point estimates are updated by majority vote:
    \[
    A_{ik}^{(1),t+1}
    =
    \mathbf 1\!\left(
    \frac{1}{C}\sum_{c=1}^C A_{ik}^{(1),(c),t} > \frac12
    \right),
    \qquad
    A_{il}^{(2),t+1}
    =
    \mathbf 1\!\left(
    \frac{1}{C}\sum_{c=1}^C A_{il}^{(2),(c),t} > \frac12
    \right).
    \]

    \item \textbf{Update the effective number of active nodes.}  
    The active widths of the two latent layers are
    \[
    K^{(1)^{*,t+1}_{\text{max}}}
    =
    \sum_{h=1}^{K_{\text{max}}^{(1)}}\mathbf 1\{c_k^{(1),t}>h\},
    \qquad
    K^{(2)^{*,t+1}}_{\text{max}}
    =
    \sum_{h=1}^{K_{\text{max}}^{(2)}}\mathbf 1\{c_k^{(2),t}>h\}.
    \]
\end{itemize}

\paragraph{Convergence and temperature schedule.}
The algorithm monitors the relative Frobenius changes
\[
\frac{\|\boldsymbol{B}^{(1),t+1}-\boldsymbol{B}^{(1),t}\|_F}{\|\boldsymbol{B}^{(1),t}\|_F+\varepsilon},
\qquad
\frac{\|\boldsymbol{B}^{(2),t+1}-\boldsymbol{B}^{(2),t}\|_F}{\|\boldsymbol{B}^{(2),t}\|_F+\varepsilon},
\]
and stops when their rolling-window variability is sufficiently small. Here $\epsilon$ is a constant so that the quantity remains finite.  After burn-in, the temperature is annealed upward according to
\[
\tau^{(t+1)}
=
\min\!\left\{1,\,
\tau^{(t)} + 0.01\bigl((t-\mathrm{burn})-1\bigr)
\right\}.
\]
Section \ref{sec:hyp} has recommendations default starting values.
\subsection{Rank likelihood 
stochastic approximate EM}
\cite{lee2026dde} introduce a stochastic approximate EM (SAEM) algorithm for estimation of parametric DDEs, which is adopted for DDE copula estimation in Section \ref{sec:sim} (referred to as SAEM RL in the main text).
We refer readers to their Algorithm 2 in the main text supplement S.3.2 for detail.  To extend their algorithm to the maximum likelihood estimation under the rank likelihood DDE copula requires only one additional step. Namely, in addition to sampling binary latent variables $\{\boldsymbol A^{(d)}\}_{d=1}^{D}$, we also sample rank-consistent $\boldsymbol Z$ using Algorithm \ref{alg:rank_aug}. Then, conditional on $\boldsymbol Z \in \mathcal{R}(\boldsymbol Y)$, the subsequent M-step, which maximizes a stochastically averaged objective, is identical to the M-step for a Gaussian DDE presented in \cite{lee2026dde}. The procedure is outlined in Algorithm \ref{alg:saem_dde}.

In this case, the DDE copula parameters are $\boldsymbol \theta  = (\{\boldsymbol B^{(d)}\}, \boldsymbol \gamma)$, since dimension is inferred through the sparsity pattern in the weights. We use the default hyperparameter selections for Normal DDEs recommended by \cite{lee2026dde}, Supplement S.3.4, which include thresholding values limits for small entries in the weight matrices, as well as penalty parameters for the truncated LASSO penalty function \citep{shen2012likelihood} used in their optimization.

\begin{algorithm}[h]
\caption{Penalized SAEM algorithm for the two-latent-layer DDE}
\label{alg:saem_dde}
\begin{algorithmic}

\REQUIRE $\boldsymbol Y$

\STATE Initialize $\boldsymbol A^{(1),0}, \boldsymbol A^{(2),0}$, $\boldsymbol Z^{0}$ (via Algorithm \ref{fig:specinit}), and $\boldsymbol \theta^{0}$ (DDE parameters).

\WHILE{$\|\boldsymbol \theta^{t} - \boldsymbol \theta^{t-1}\|$ is larger than a threshold}

\STATE \textbf{Iteration $t$}

\STATE \textit{// Simulation step}

\STATE Sample $\boldsymbol A^{(1),t+1}_{i,k}$ and $\boldsymbol A^{(2),t+1}_{i,k}$ from the complete conditionals using previous parameter estimates $\boldsymbol \theta^{t}$ and $\boldsymbol A^{[t]}$.
\STATE Sample $\boldsymbol{Z} \in \mathcal{R}(\boldsymbol Y)$

\STATE \textit{// Stochastic approximation M-step}

\STATE Update parameters $\boldsymbol \theta^{[t+1]}$ by maximizing the stochastic averaged objective for a Gaussian DDE

\ENDWHILE

\STATE \textbf{Output:} Estimated continuous parameters $\widehat{\boldsymbol \theta}$
\end{algorithmic}
\end{algorithm}

\subsection{Addressing latent variable permutations for simulations}

As mentioned in \cite{lee2026dde}, resolving latent variable permutation is necessary to accurately compute errors between estimated weight matrices and data generating weight matrices. Given an estimated weight matrix $\hat{\boldsymbol{B}}^{(d)}$, we resolve the permutation ambiguity by rearranging orderings of latent variables (and thus the rows and columns of each layer-specific weight matrix) and solving an assignment problem for each layer in a bottom-up fashion.

First, we construct a $K^{(1)}_{\text{max}} \times K^{(1)}_{\text{max}}$ cost matrix, where each entry is the squared $L^{2}$ norm between corresponding column vectors of $\boldsymbol B^{(1)}$ and $\hat{\boldsymbol {B}}^{(1)}$. Next, using the Hungarian algorithm of \cite{kuhn2005hungarian}, we find the optimal permutation of the columns of $\hat{\boldsymbol{B}}^{(1)}$ that minimizes assignment cost. In a recursive manner, we then permute the rows of $\hat{\boldsymbol{ B}}^{(2)}$ based on the optimal permutation of the columns in $\hat{\boldsymbol B}^{(1)}$, and solve the assignment problem Hungarian for the optimal permutation of the columns of $\hat{\boldsymbol{B}}^{(2)}$. This procedure may be generalized for deeper models by solving the assignment problem for $\hat{\boldsymbol B}^{(d)}$, permuting the rows of $\hat{\boldsymbol B}^{(d+1)}$ based on the solution in the previous layer, and then again solving the assignment problem.
\subsection{Hyperparameter selections}\label{sec:hyp}

In our simulations and real data analysis, we find that there are four parameters that meaningfully effect DDE Copula estimation: the layer-specific rates of the spike-Laplace kernel in the CSP prior $\{\lambda_{1}^{(d)}\}_{d=1}^{D}$, the temperature $\tau$, and  $\xi$, the truncation limit for thresolding  small coefficients.

\paragraph{- $\boldsymbol{\{\lambda_{1}^{(d)}\}_{d=1}^{D}}$} For 2-layer DDE copulas, we generally specify $\lambda_{1}^{(1)} <\lambda_{1}^{(2)}$, with the of each depending on the level of sparsity. For larger $\lambda_{1}^{(d)}$, there is higher prior probability that a given column will come from the spike. Thus, we encourage more shrinkage in layer 2, which is affected by noise propagation from layer 1 latent variables. For $J = 50$, we set $\lambda_{1}^{(1)} = .02$ and $\lambda_{1}^{(2)} = .04$, since there were fewer irrelevant columns in each weight matrix. For $J = 100,150$, we set $\lambda_{1}^{(1)} = .05$ and $\lambda_{1}^{(2)} = .1$, as stronger shrinkage was required. For deeper DDE copula  we recommend the ascending pattern for the penalty parameters for the first two layers, but suggest $\lambda_{1}^{(d)} < .01, d>2$.  We find that, while noise propagation through the network can create false positives in shallow layers, as the DDE becomes deeper, it is increasingly difficult to detect signal. In  our deeper simulation (Section \ref{sec:deep}), we set $\lambda_{1}^{(1)} = .01, \lambda_{1}^{(2)} = .03,\lambda_{1}^{(3)} = .005$

\paragraph{- $\boldsymbol \tau$} For the temperature, lower values will encourage more sparsity. The initial temperature should also balance the depth of the DDE. In our two-layer simulations we use $\tau = .7$, which is subsequently raised in each iteration of the fitting algorithm. We note that we also apply tempering to the RL SAEM algorithm. For the deeper simulation in Section \ref{sec:deep}, the initial temperature of .7 prevented us from estimating any non-zero entries in $\boldsymbol{B}^{(3)}$, so we used $\tau = .9$ for those simulations.  

\paragraph{- $\boldsymbol \xi$}
We follow the recommendations of \cite{lee2026dde} and use $\xi = \max(.3, 3*N^{-.3})$ for all simulations and real data analyses
\section{Additional Simulation Results}\label{sec:addsim}
We complete the results presented in Section \ref{sec:sim} by including the simulations for $J = 50$ and $J = 150$. We note largely consistent results to what is presented in the main text: RL CSP provides more accurate weight, sparsity structure, and dimension estimation, regardless of the dimension of $\boldsymbol Y$. First we include the average estimate of $\boldsymbol B^{(1)}$ across sample sizes and methods for $J = 100$ in Figure \ref{fig:B1s_J100}, analogous to what we presented in Figure \ref{fig:BsJ100} in the main text.

\begin{figure}[h]
    \centering
    \includegraphics[width=\linewidth, keepaspectratio]{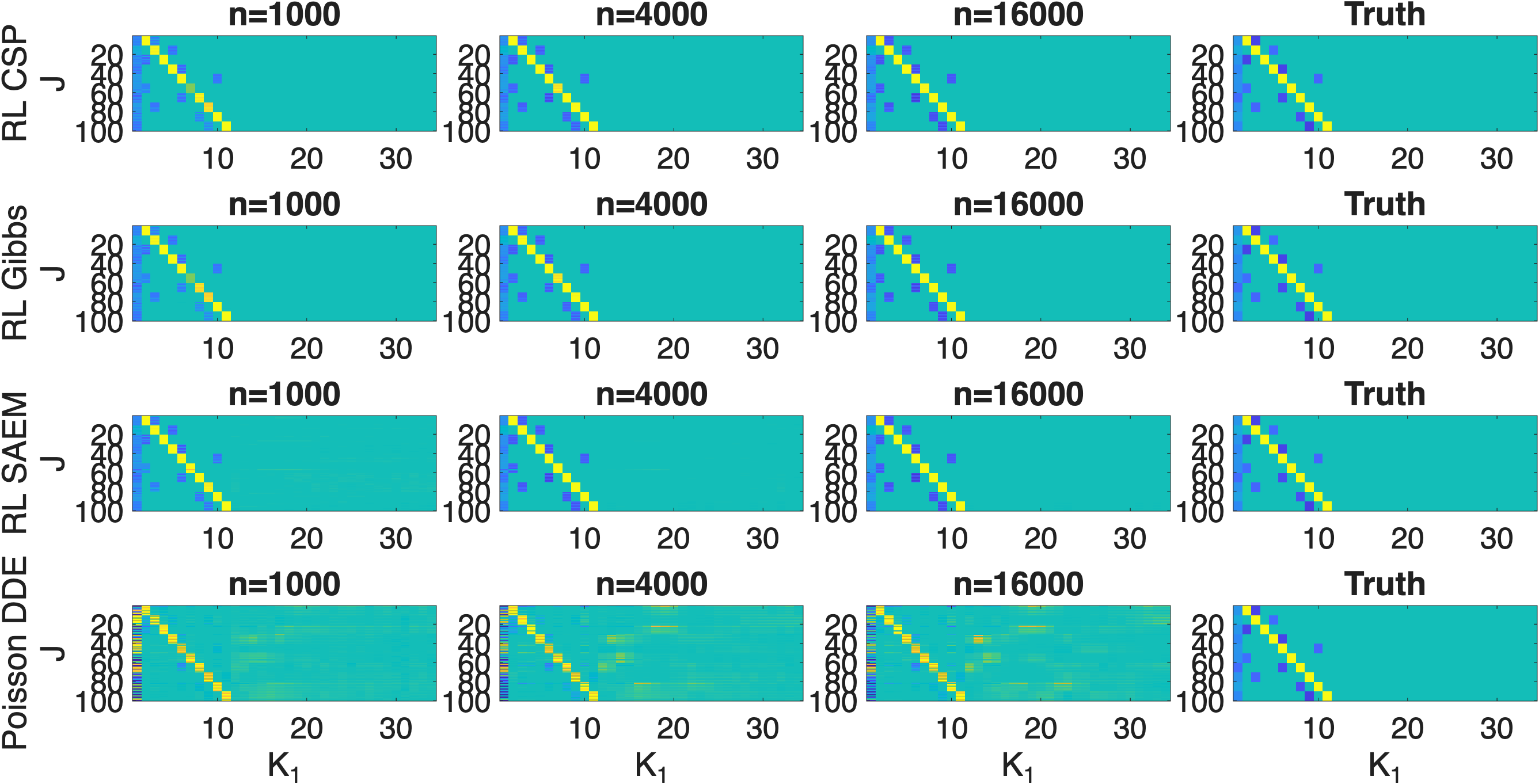}
    \caption{$J = 100$: Average estimate of $\boldsymbol{B}^{(1)}$ across methods and sample sizes }
    \label{fig:B1s_J100}
\end{figure}

We note that all rank-likelihood methods estimate $\boldsymbol{B}^{(1)}$ accurately in increasing dimensions, but that the Poisson DDE performs substantially worse.  Next we provide complete results for the remaining sample sizes.
\subsection{J = 50}
Analogous to the reuslts in the main text, in table Figure \ref{fig:DGPsim}, we include the data generating weight matrices for the $J = 50$ case. In Table \ref{tab:recovery_resultsJ50}, we provide dimension and sparsity structure estimation results, while in Figures \ref{fig:MSEBsJ50}-\ref{fig:B2sJ50} we provide weight matrix estimation statistics and visualizations.

\begin{figure}[h]
    \centering
    \includegraphics[width=0.45\linewidth]{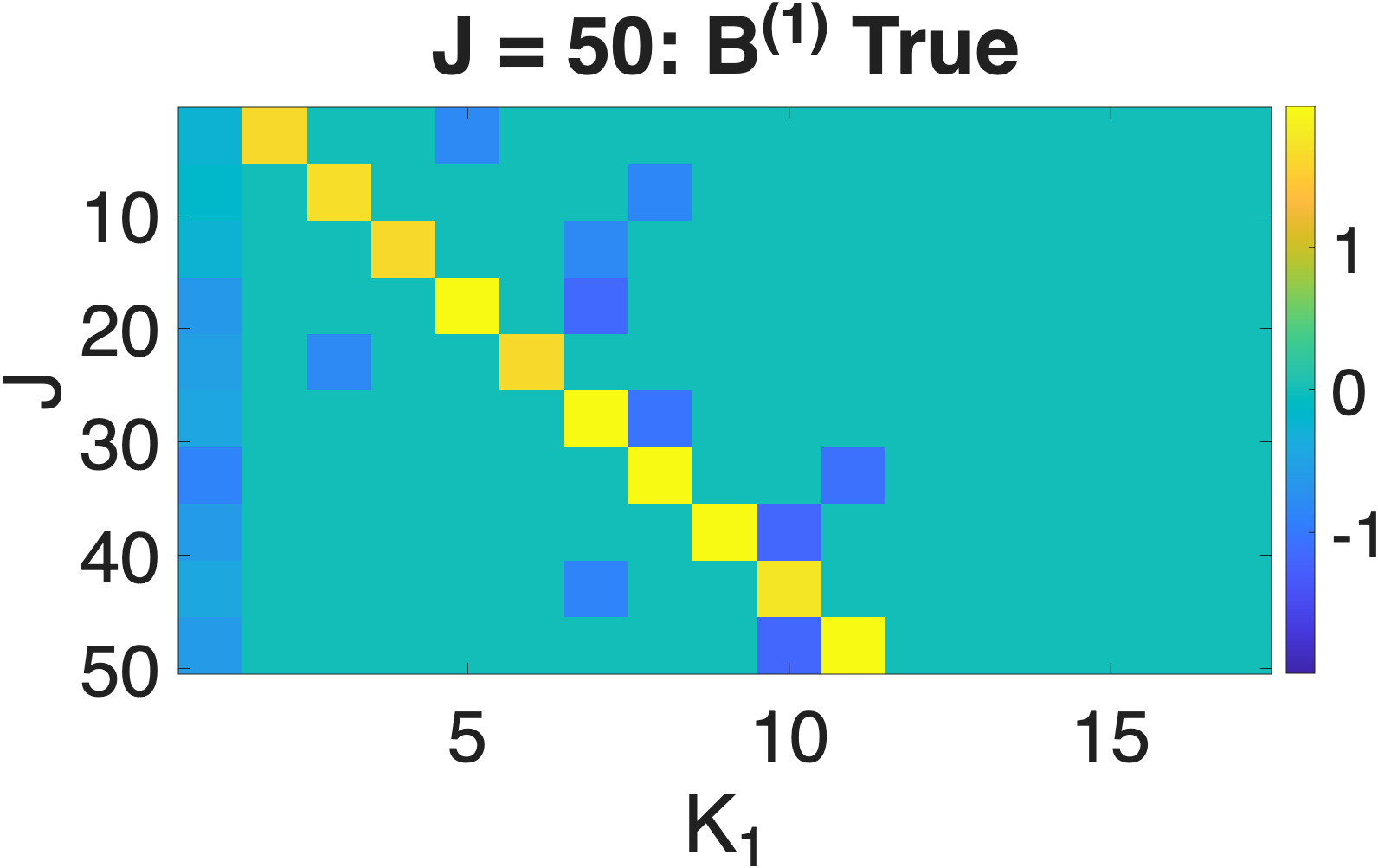}
        \includegraphics[width=0.45\linewidth]{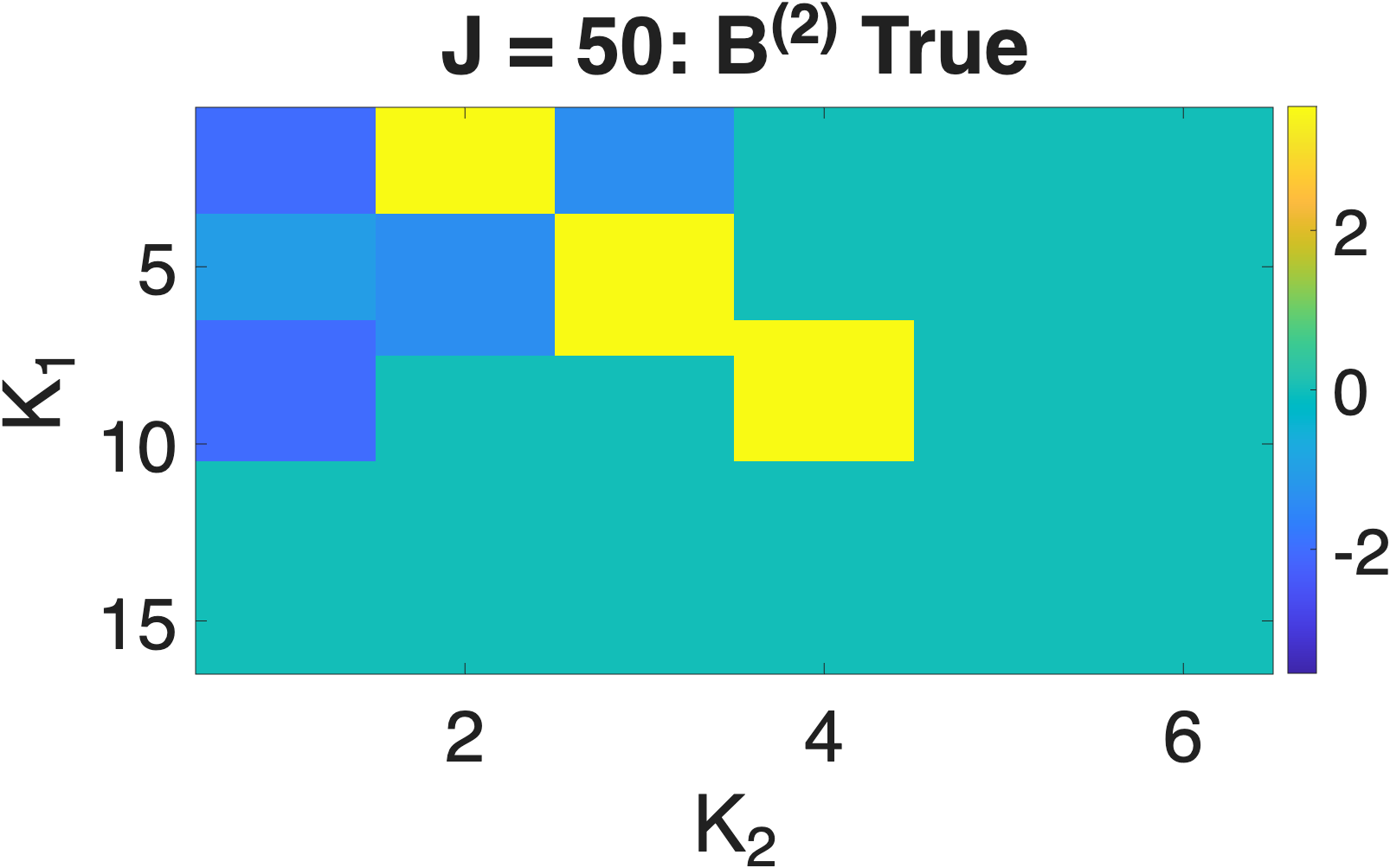}
    \caption{Data generating $\{\boldsymbol B^{(d)}\}_{d=1}^{2}$ for $J = 50$}
    \label{fig:DGPsim}
\end{figure}

\begin{table}[h]
\centering
\begin{tabular}{lcccccc}
\toprule
 & \multicolumn{6}{c}{$n$} \\
\cmidrule(lr){2-7}
 & $500$ & $1000$ & $2000$ & $4000$ & $8000$ & $16000$ \\
\midrule

\multicolumn{7}{l}{\textbf{Average recovery of $G_1$}} \\
RL CSP     & 0.976 & 0.992 & 0.998 & 1.000 & 1.000 & 1.000 \\
RL Gibbs   & 0.976 & 0.991 & 0.998 & 0.999 & 1.000 & 0.999 \\
RL SAEM    & 0.982 & 0.989 & 0.991 & 0.993 & 0.995 & 0.995 \\
Poisson DDE& 0.909 & 0.888 & 0.861 & 0.857 & 0.855 & 0.849 \\

\multicolumn{7}{l}{\textbf{Average recovery of $G_2$}} \\
RL CSP     & 0.926 & 0.957 & 0.976 & 0.995 & 0.997 & 0.999 \\
RL Gibbs   & 0.896 & 0.880 & 0.875 & 0.887 & 0.892 & 0.900 \\
RL SAEM    & 0.790 & 0.776 & 0.755 & 0.768 & 0.776 & 0.785 \\
Poisson DDE& 0.713 & 0.668 & 0.622 & 0.615 & 0.605 & 0.613 \\

\multicolumn{7}{l}{\textbf{Average MAP estimate of Layer 1 dimension $K^{(1)^{*}}$}} \\
RL CSP     & 10.020 & 10.060 & 10.090 & 10.020 & 10.000 & 10.000 \\
RL Gibbs   & 10.050 & 10.130 & 10.100 & 10.060 & 10.000 & 10.000 \\
RL SAEM    & --     & --     & --     & --     & --     & --     \\
Poisson DDE& --     & --     & --     & --     & --     & --     \\

\multicolumn{7}{l}{\textbf{Average MAP estimate of Layer 2 dimension $K^{(2)^{*}}$}} \\
RL CSP     & 3.550 & 3.130 & 3.030 & 3.010 & 3.010 & 3.010 \\
RL Gibbs   & 4.240 & 4.770 & 4.960 & 4.990 & 4.990 & 5.000 \\
RL SAEM    & --    & --    & --    & --    & --    & --    \\
Poisson DDE& --    & --    & --    & --    & --    & --    \\

\bottomrule
\end{tabular}
\caption{$J = 50$: Sparsity pattern recovery accuracy and estimated number of active nodes across methods and sample sizes.}
\label{tab:recovery_resultsJ50}
\end{table}

\begin{figure}[h]
    \centering
    \includegraphics[width=0.7\linewidth, keepaspectratio]{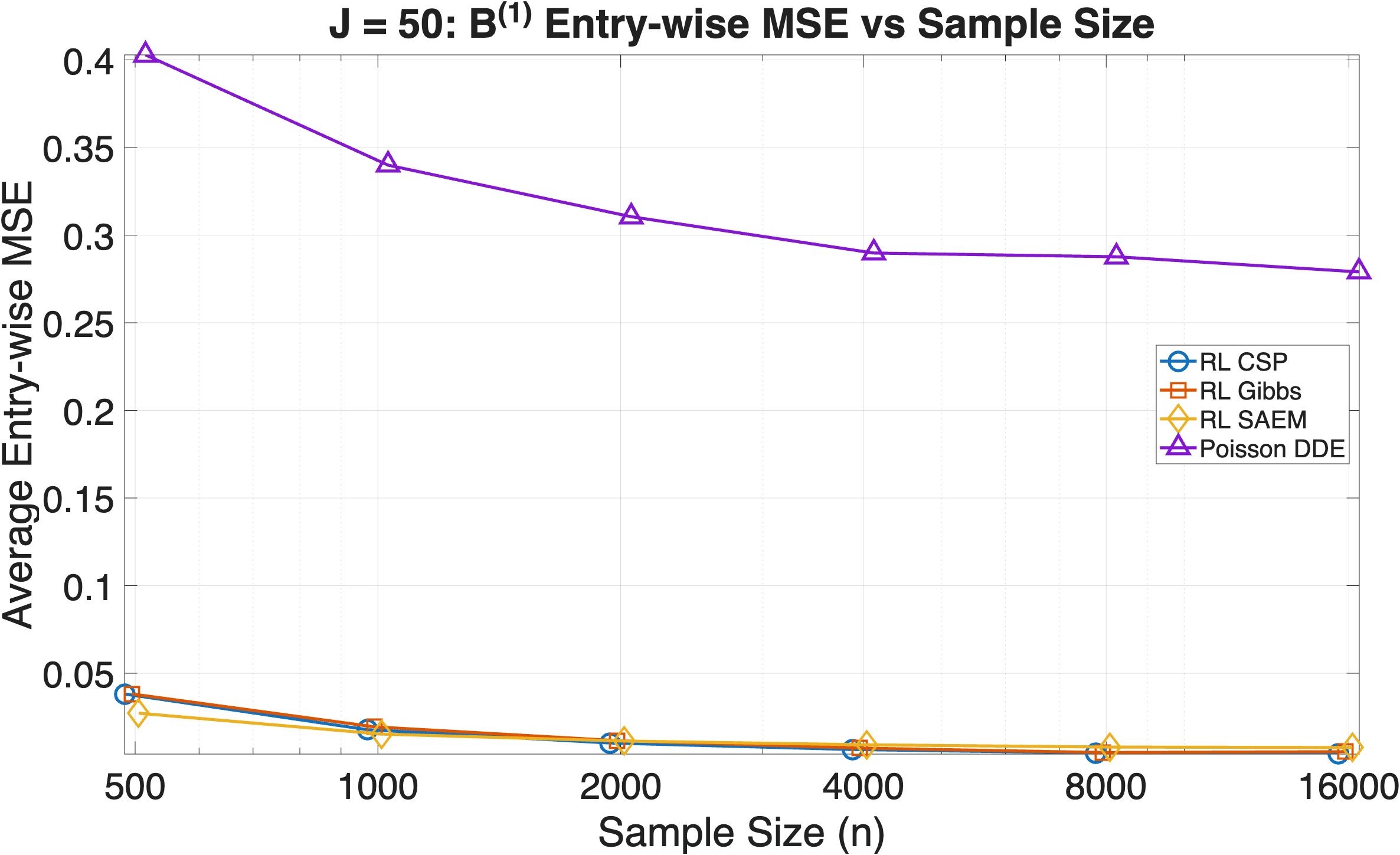}
        \includegraphics[width=0.7\linewidth,keepaspectratio]{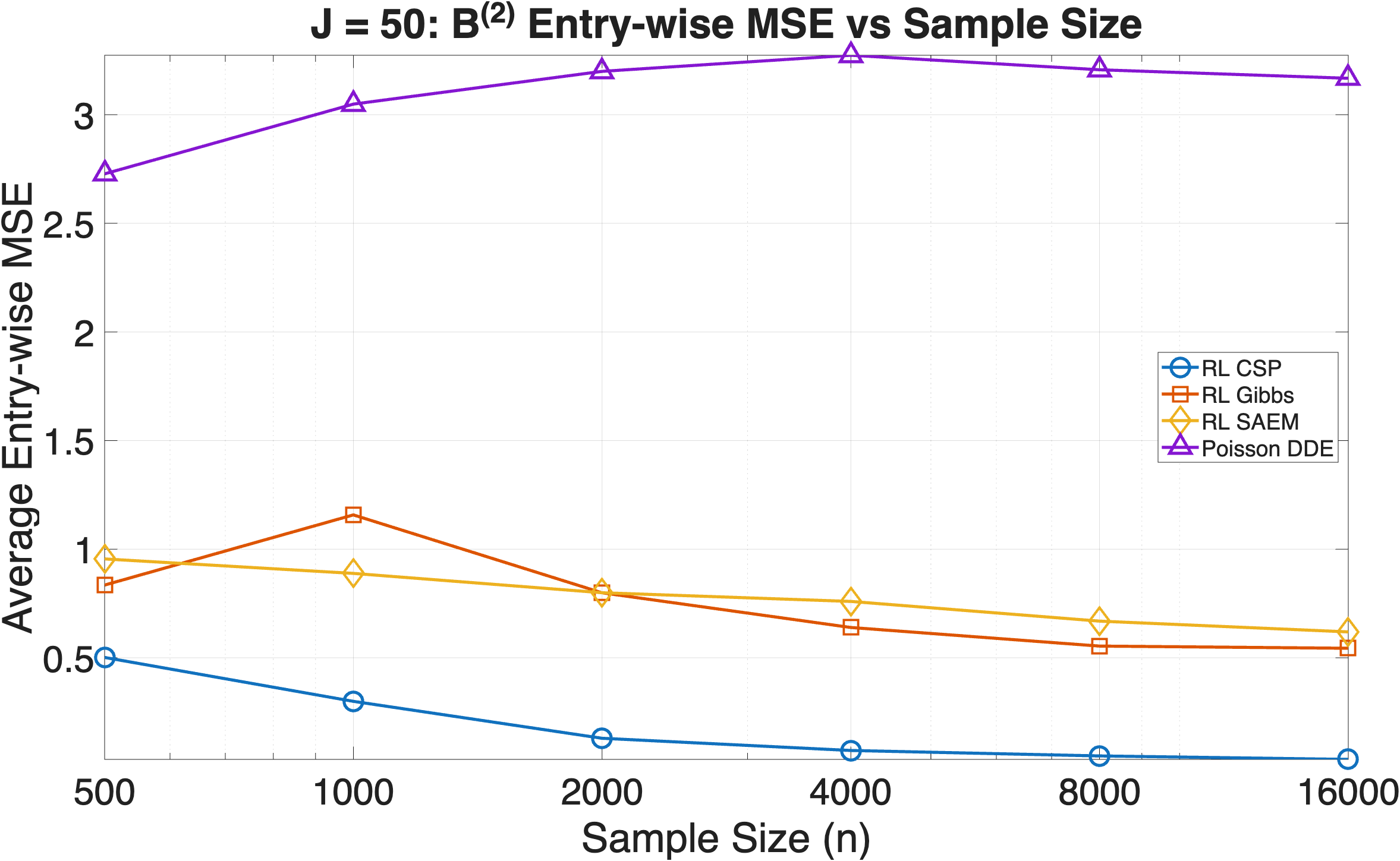}
    \caption{$J = 100$: Average entry-wise MSE for estimates of $\boldsymbol{B}^{(1)}$ (top four rows) and $\boldsymbol{B}^{(2)}$ across sample sizes and methods} 
    \label{fig:MSEBsJ50}
\end{figure}

\begin{figure}[h]
    \centering
        \includegraphics[width=\linewidth,keepaspectratio]{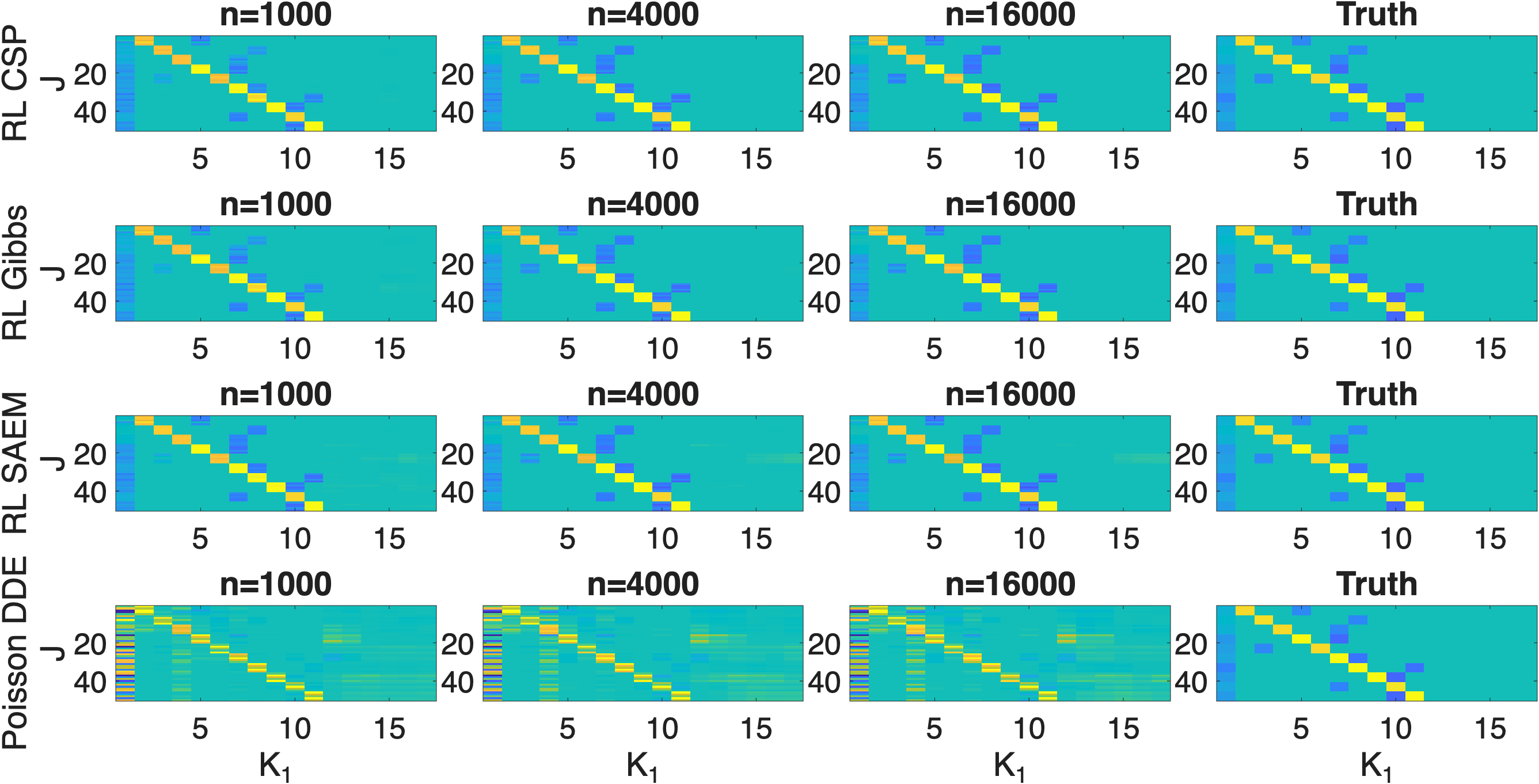}
    \caption{$J = 50$: Average point estimates of $\boldsymbol{B}^{(1)}$  across sample sizes and methods compared to the data generating values} 
    \label{fig:B1sJ50}
\end{figure}

\begin{figure}[h]
    \centering
        \includegraphics[width=\linewidth,keepaspectratio]{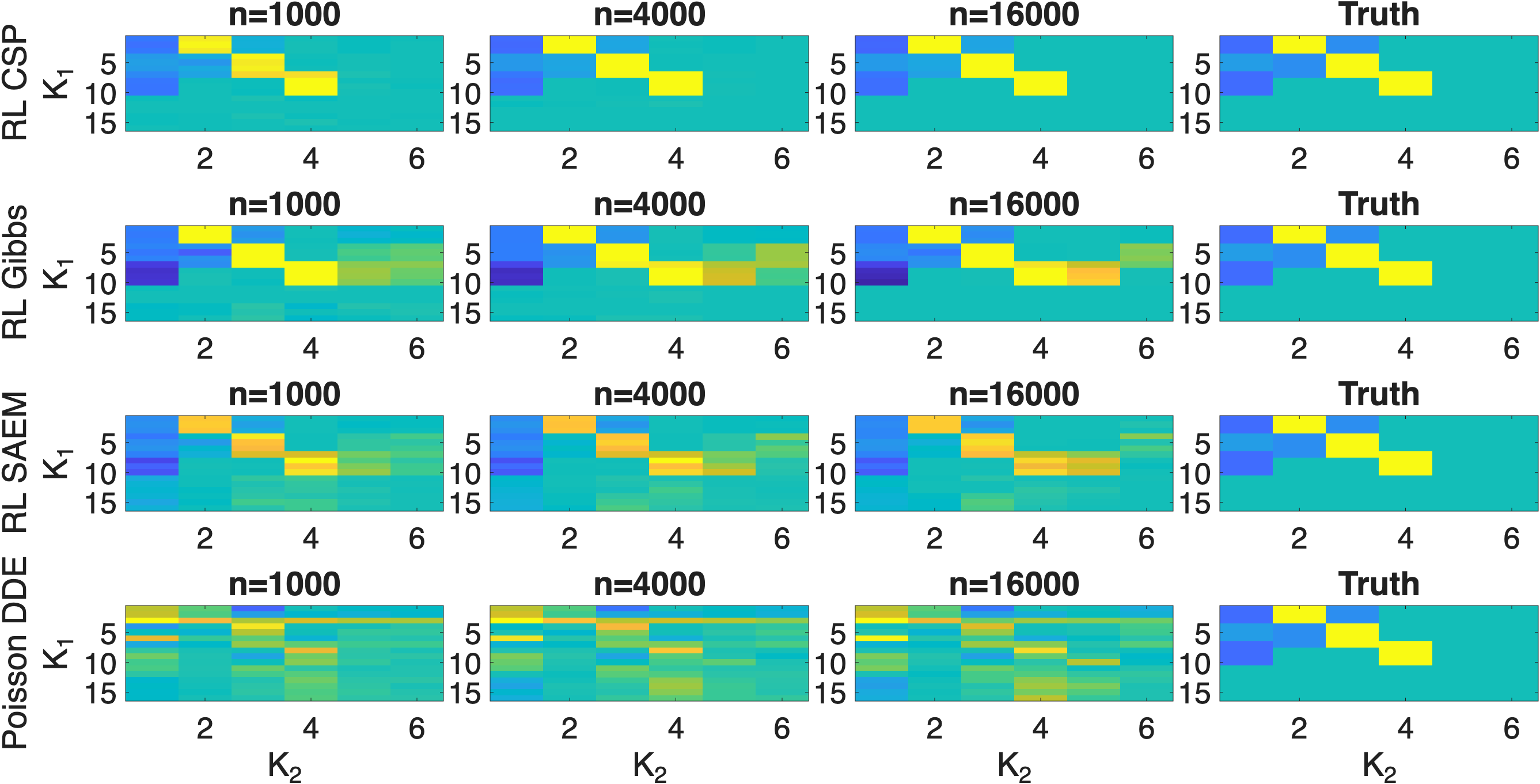}
    \caption{$J = 50$: Average point estimates of $\boldsymbol{B}^{(1)}$  across sample sizes and methods compared to the data generating values} 
    \label{fig:B2sJ50}
\end{figure}

\subsection{J = 150}
Analogous to the reuslts in the main text, in table Figure \ref{fig:DGPsim150}, we include the data generating weight matrices for the $J = 50$ case. In Table \ref{tab:recovery_resultsJ150}, we provide dimension and sparsity structure estimation results, while in Figures \ref{fig:B1sJ150}-\ref{fig:MSEBsJ150}we provide weight matrix estimation statistics and visualizations.
\begin{figure}[h]
    \centering
    \includegraphics[width=0.45\linewidth]{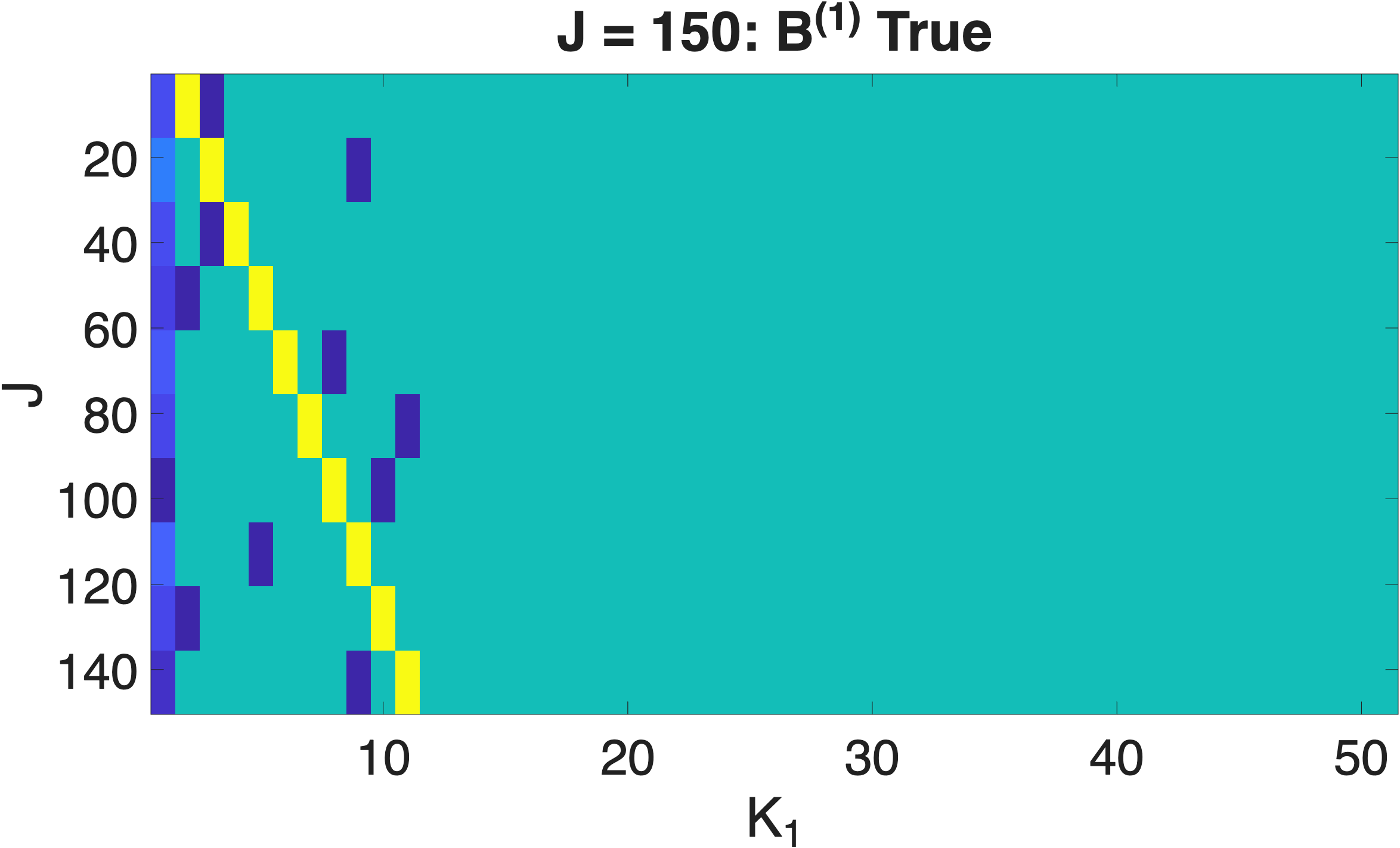}
        \includegraphics[width=0.45\linewidth]{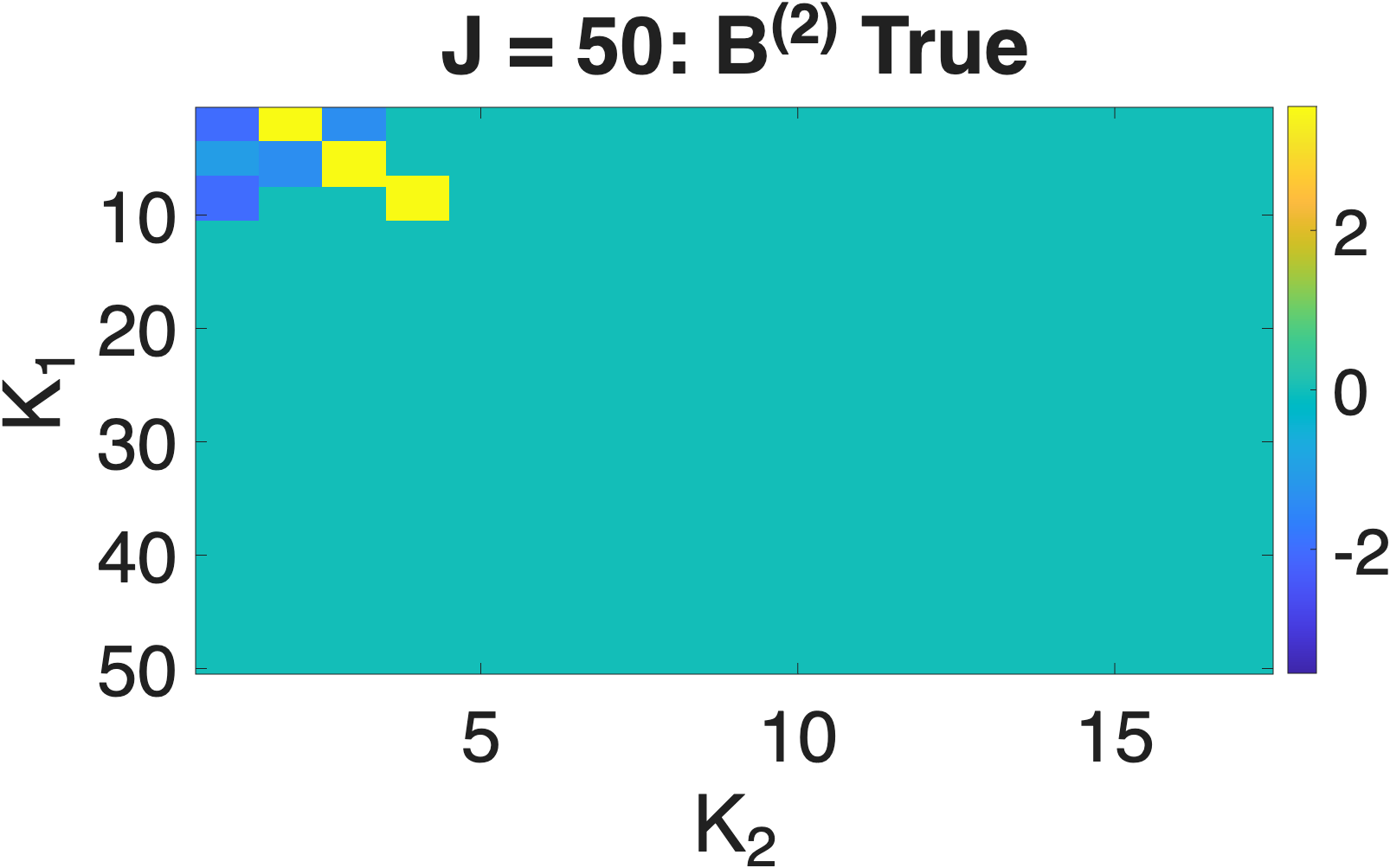}
    \caption{Data generating $\{\boldsymbol B^{(d)}\}_{d=1}^{2}$ for $J = 150$}
    \label{fig:DGPsim150}
\end{figure}
\begin{table}[h]
\centering
\small
\begin{tabular}{lcccccc}
\hline
 & \multicolumn{6}{c}{$n$} \\
\cline{2-7}
 & 500 & 1000 & 2000 & 4000 & 8000 & 16000 \\
\hline

\textbf{Average recovery of $G_1$} \\
RL CSP    & 0.982 & 0.992 & 0.996 & 0.999 & 1.000 & 1.000 \\
RL Gibbs  & 0.982 & 0.991 & 0.996 & 0.999 & 1.000 & 0.999 \\
RL SAEM   & 0.979 & 0.985 & 0.990 & 0.994 & 0.995 & 0.997 \\
Poisson DDE & 0.959 & 0.953 & 0.943 & 0.943 & 0.944 & 0.945 \\

\textbf{Average recovery of $G_2$} \\
RL CSP    & 0.985 & 0.992 & 0.995 & 0.999 & 0.999 & 0.999 \\
RL Gibbs  & 0.977 & 0.978 & 0.973 & 0.967 & 0.964 & 0.962 \\
RL SAEM   & 0.796 & 0.800 & 0.827 & 0.845 & 0.852 & 0.852 \\
Poisson DDE & 0.678 & 0.698 & 0.700 & 0.725 & 0.736 & 0.741 \\

\textbf{Average MAP estimate of Layer 1 dimension $K^{(1)}$} \\
RL CSP    & 8.180 & 8.970 & 9.250 & 9.930 & 10.020 & 10.050 \\
RL Gibbs  & 8.060 & 8.700 & 9.280 & 9.900 & 10.030 & 10.110 \\
RL SAEM   & -- & -- & -- & -- & -- & -- \\
Poisson DDE & -- & -- & -- & -- & -- & -- \\

\textbf{Average MAP estimate of Layer 2 dimension $K^{(2)}$} \\
RL CSP    & 3.680 & 3.280 & 3.010 & 3.000 & 3.020 & 3.040 \\
RL Gibbs  & 5.750 & 6.000 & 7.150 & 8.490 & 8.930 & 9.130 \\
RL SAEM   & -- & -- & -- & -- & -- & -- \\
Poisson DDE & -- & -- & -- & -- & -- & -- \\

\hline
\end{tabular}
\caption{$J=150$: Sparsity pattern recovery accuracy and estimated number of active nodes across methods and sample sizes.}\label{tab:recovery_resultsJ150}
\end{table}

\begin{figure}[h]
    \centering
        \includegraphics[width=\linewidth,keepaspectratio]{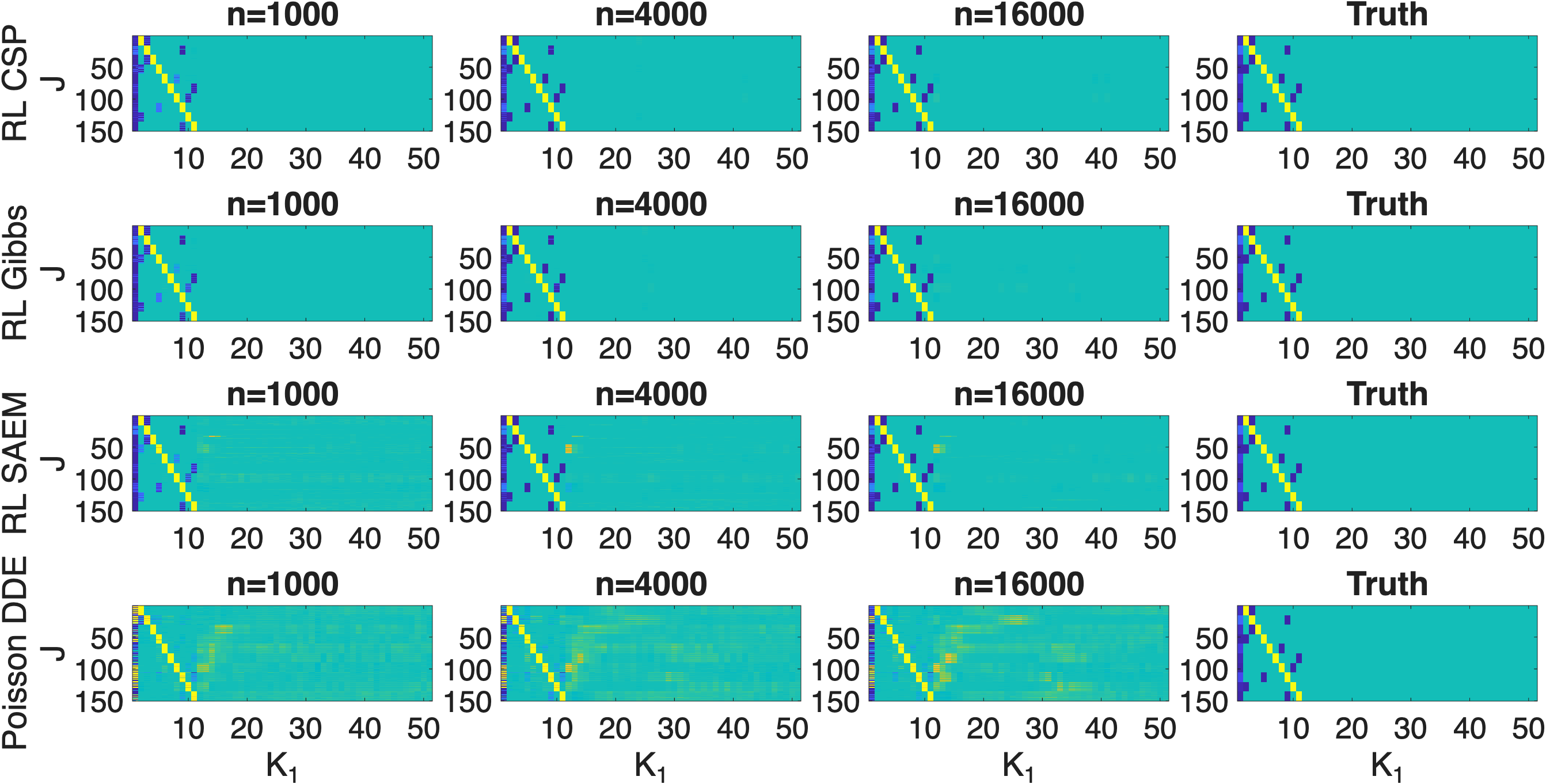}
    \caption{$J = 150$: Average point estimates of $\boldsymbol{B}^{(1)}$  across sample sizes and methods compared to the data generating values} 
    \label{fig:B1sJ150}
\end{figure}

\begin{figure}[h]
    \centering
        \includegraphics[width=\linewidth,keepaspectratio]{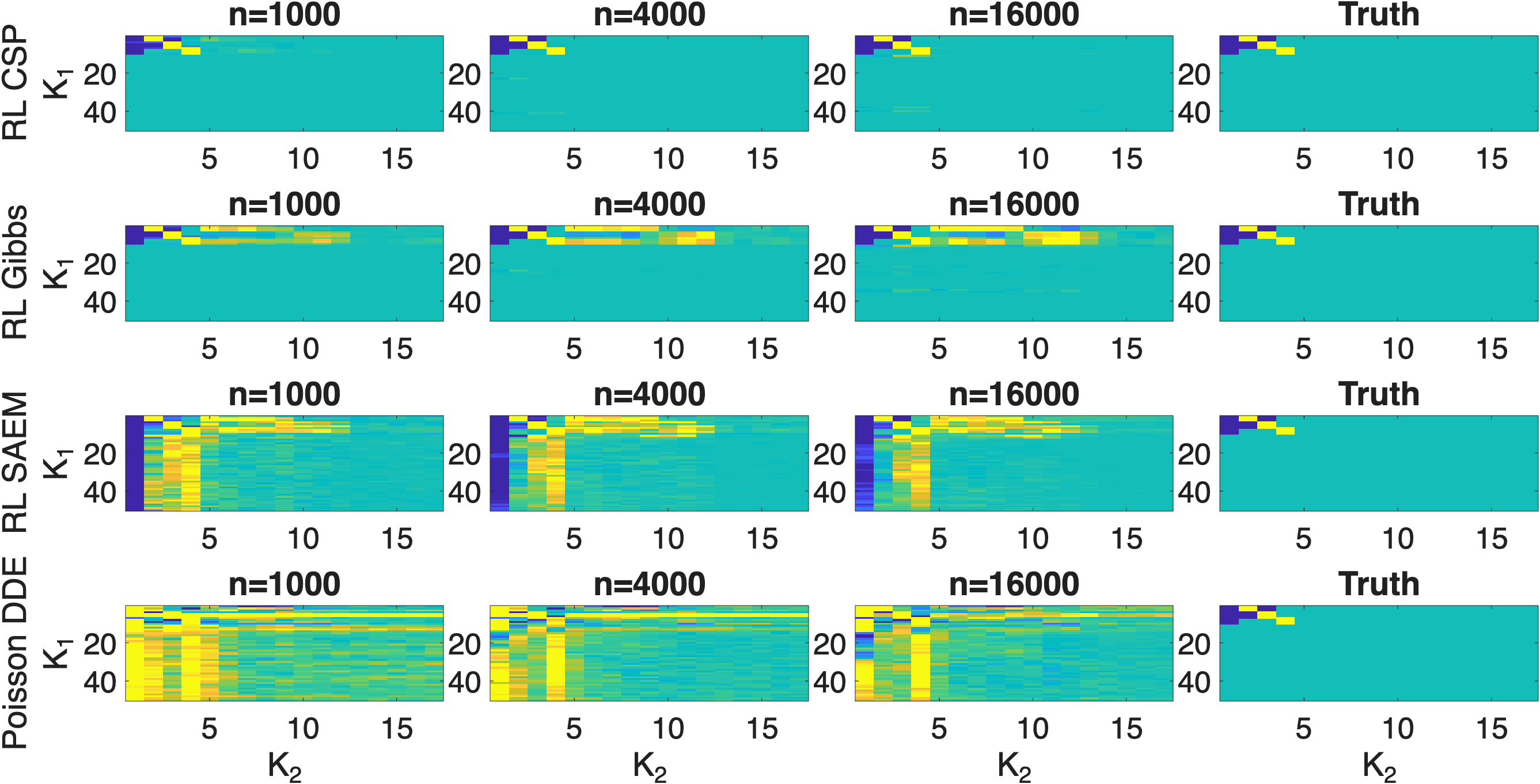}
    \caption{$J = 150$: Average point estimates of $\boldsymbol{B}^{(2)}$  across sample sizes and methods compared to the data generating values} 
    \label{fig:B2sJ150}
\end{figure}
\begin{figure}[h]
    \centering
    \includegraphics[width=0.7\linewidth, keepaspectratio]{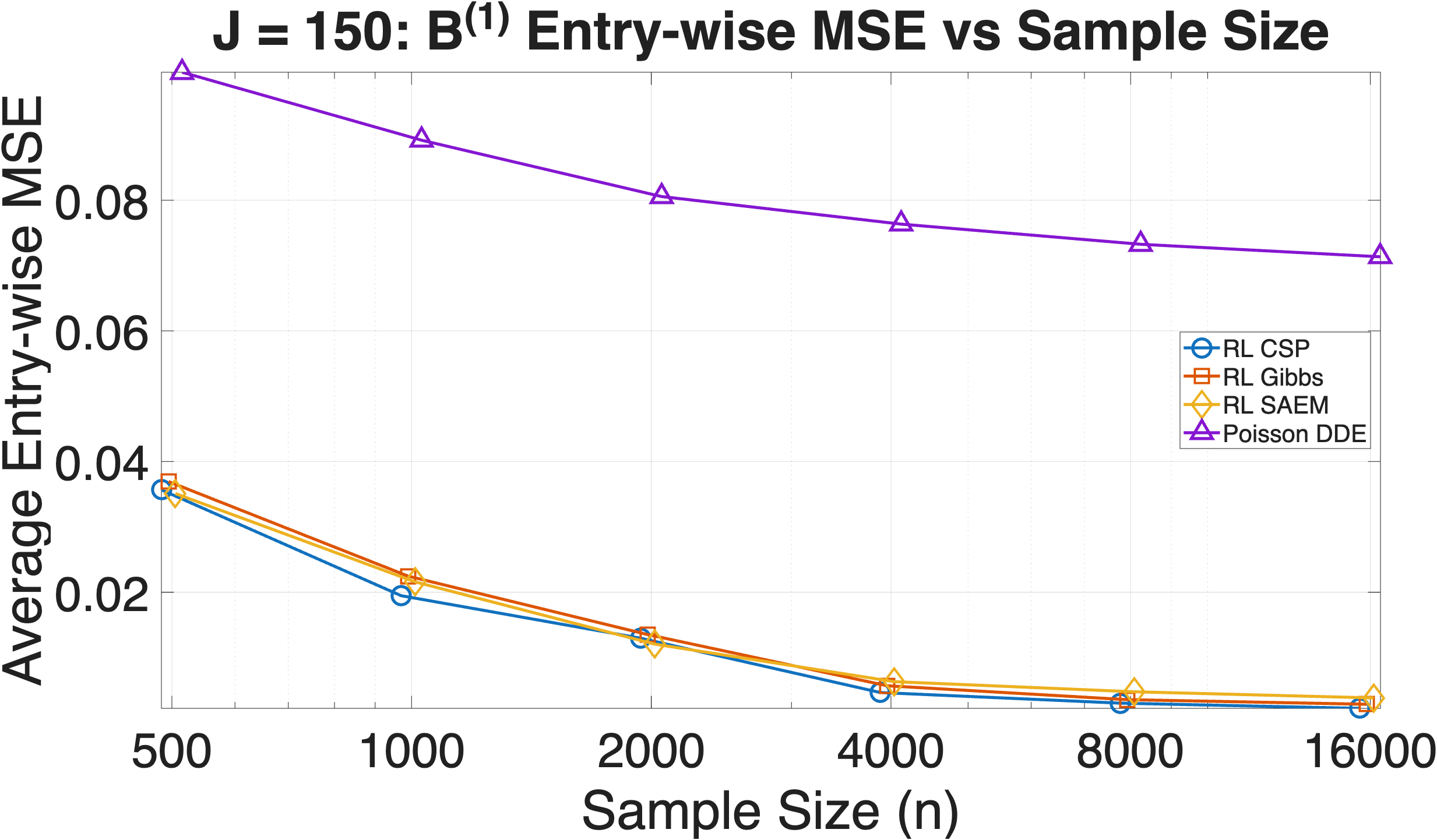}
        \includegraphics[width=0.7\linewidth,keepaspectratio]{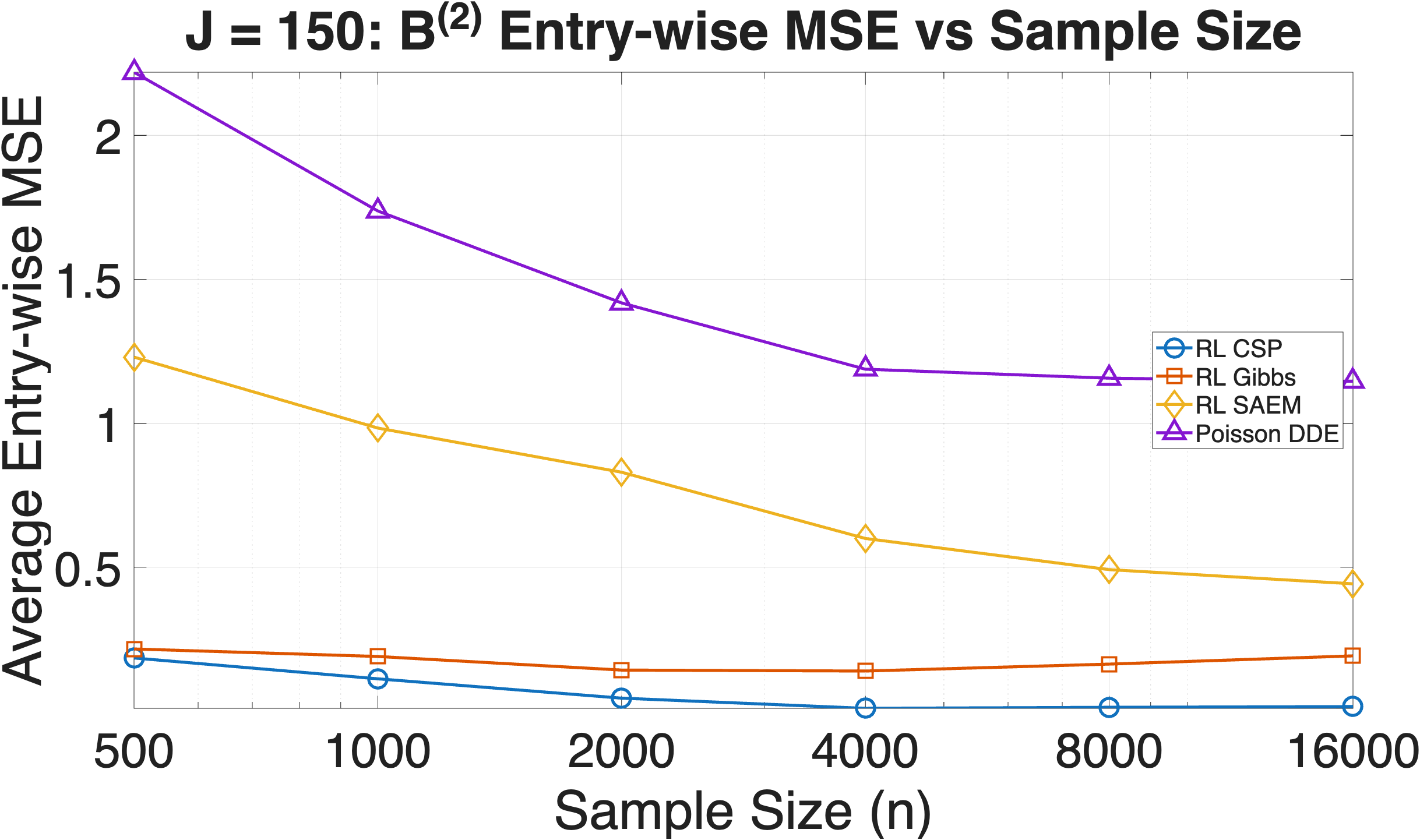}
    \caption{$J = 150$: Average entry-wise MSE for estimates of $\boldsymbol{B}^{(1)}$ (top four rows) and $\boldsymbol{B}^{(2)}$ across sample sizes and methods} 
    \label{fig:MSEBsJ150}
\end{figure}

\subsection{Deeper Simulations}\label{sec:deep}
We also ran simulations for deeper DDE Copulas ($D = 3$). We generated samples from a three layer DDE copula with $J = 108$, $K^{(1)} = 18$, $K^{(2)} = 6$, and $K^{(3)} = 2$. Data generating weight matrices are available in the right panel of Figure \ref{fig:D3avg}. We follow the pattern in Section \ref{sec:sim} for the marginals $F_{Y_{j}}$, alternating between irregular, discrete distributions. Like in Section \ref{sec:sim}, we generate 100 simulated data sets for each sample size and fit the DDE copula using Algorithm \ref{alg:mcem}, first initializing the parameters based on the maximal allowable DDE network size of $K^{(1)}_{\text{max}} = 36, K^{(2)}_{\text{max}} = 12, K^{(3)}_{\text{max}} = 4$. All models were fit under the same hyperparameter settings; $\tau = .9$, $\lambda_{1}^{(1)} = .01, \lambda_{1}^{(2)} = .03, \lambda_{1}^{(2)} = .005$.

MAP estimates for latent dimension, as well as the accuracy for the estimated sparsity structure, are available in Table \ref{tab:D3_recovery}. As the sample size increases, the latent dimension and sparsity structure is estimated with more accuracy. Similar to what was observed by \cite{lee2026dde}, estimation accuracy is generally inferior for deeper layers, though convergence is apparent. This model is considerably more dense -- marginally the induced mixture has $2^{18+6+2} = 67108864$ components -- and so larger sample sizes are required for accurate estimation. 
\begin{table}[ht]
\centering
\begin{tabular}{lcccccc}
\toprule
 & \multicolumn{6}{c}{$n$} \\
\cmidrule(lr){2-7}
 & 500 & 1000 & 2000 & 4000 & 8000 & 16000 \\
\midrule

\multicolumn{7}{l}{\textbf{Average recovery of $G_1$}} \\
Estimate & 0.9774 & 0.9840 & 0.9838 & 0.9934 & 0.9980 & 0.9990 \\

\multicolumn{7}{l}{\textbf{Average recovery of $G_2$}} \\
Estimate & 0.9131 & 0.9154 & 0.9177 & 0.9288 & 0.9366 & 0.9357 \\

\multicolumn{7}{l}{\textbf{Average recovery of $G_3$}} \\
Estimate & 0.5405 & 0.5772 & 0.6870 & 0.8010 & 0.8490 & 0.8670 \\

\multicolumn{7}{l}{\textbf{Average MAP estimate of Layer 1 dimension $K^{(1)}$}} \\
Estimate & 17.660 & 17.890 & 17.990 & 18.080 & 18.040 & 18.050 \\

\multicolumn{7}{l}{\textbf{Average MAP estimate of Layer 2 dimension $K^{(2)}$}} \\
Estimate & 9.820 & 8.930 & 7.750 & 6.660 & 6.220 & 6.310 \\

\multicolumn{7}{l}{\textbf{Average MAP estimate of Layer 3 dimension $K^{(3)}$}} \\
Estimate & 3.270 & 3.420 & 3.260 & 2.740 & 2.260 & 1.660 \\

\bottomrule
\end{tabular}
\caption{$D = 3$ Average graph recovery and MAP estimates of latent dimensions across sample sizes for the proposed method.}
\label{tab:D3_recovery}
\end{table}

\begin{figure}[h]
    \centering
    \includegraphics[width=0.5\linewidth, keepaspectratio]{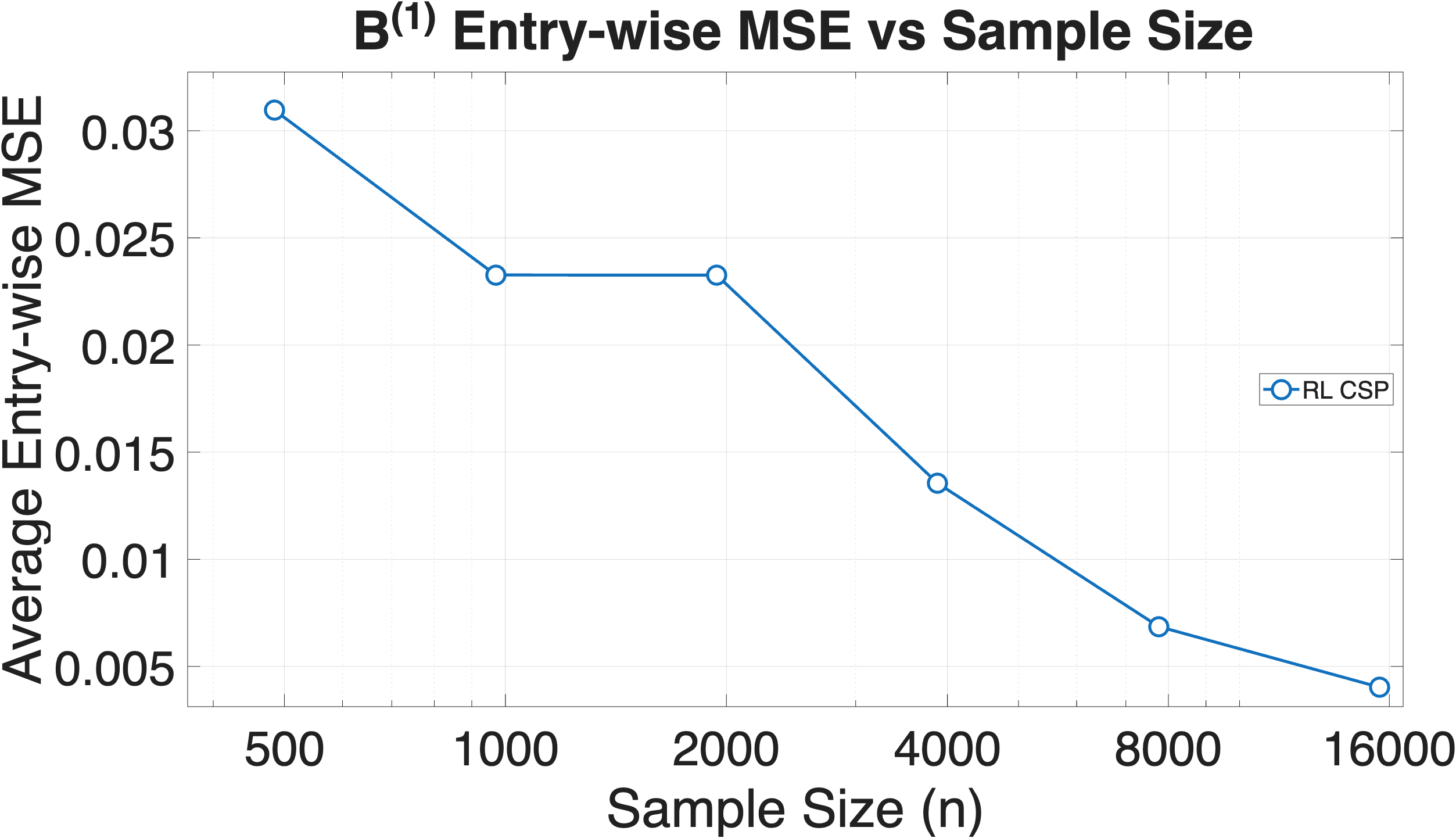}
        \includegraphics[width=0.5\linewidth,keepaspectratio]{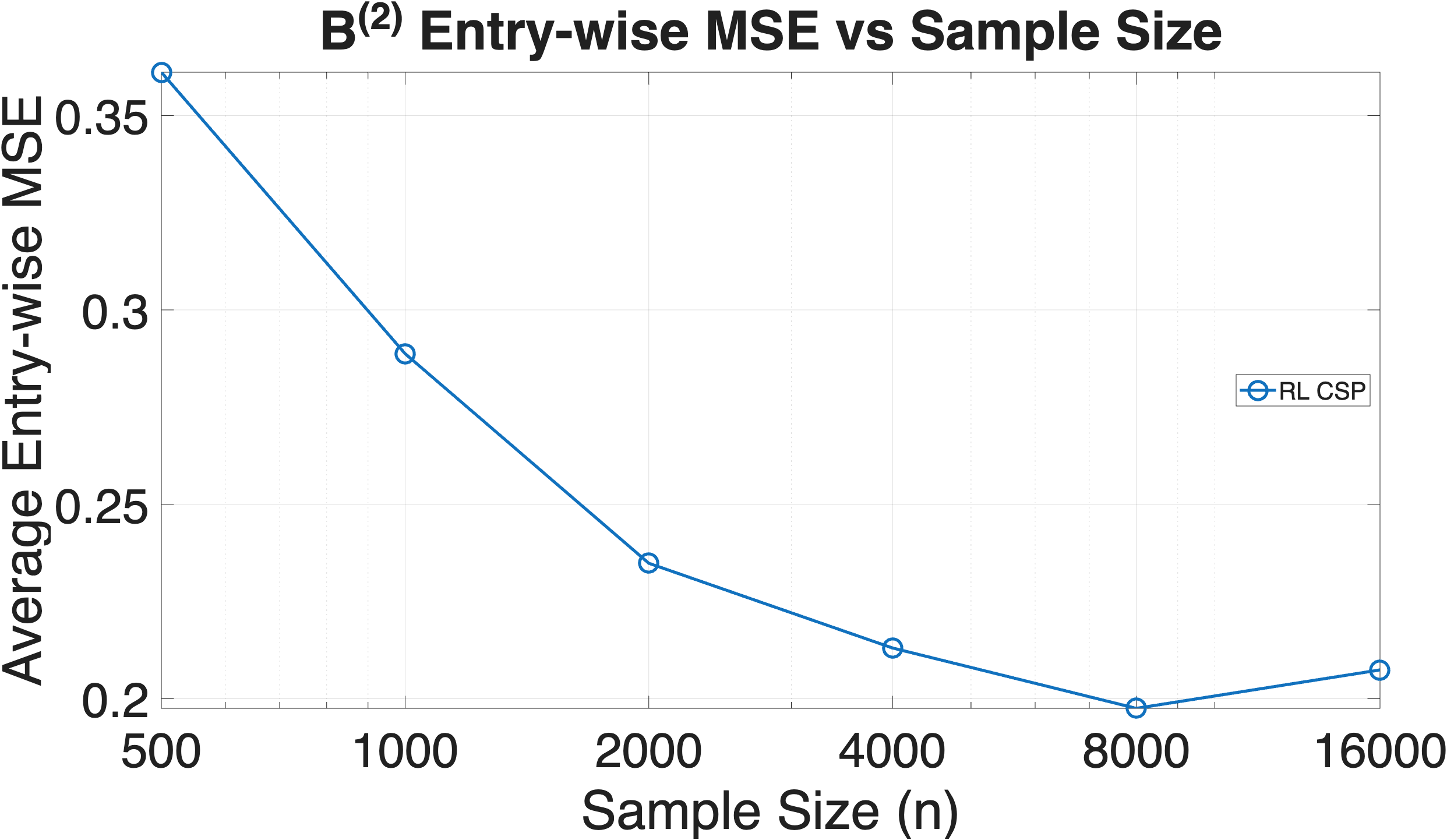}
        \includegraphics[width=0.5\linewidth,keepaspectratio]{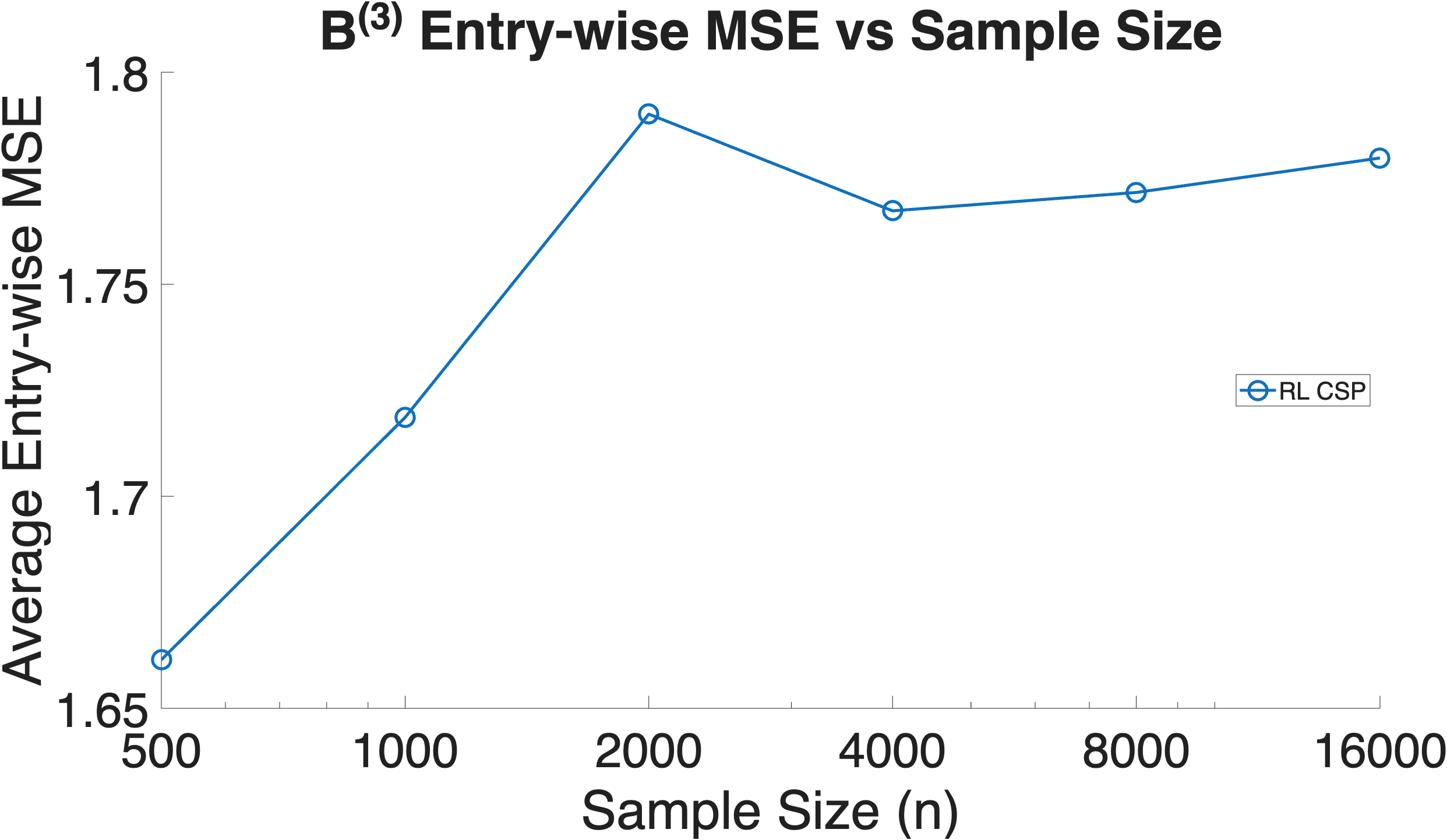}
        \includegraphics[width=0.5\linewidth,keepaspectratio]{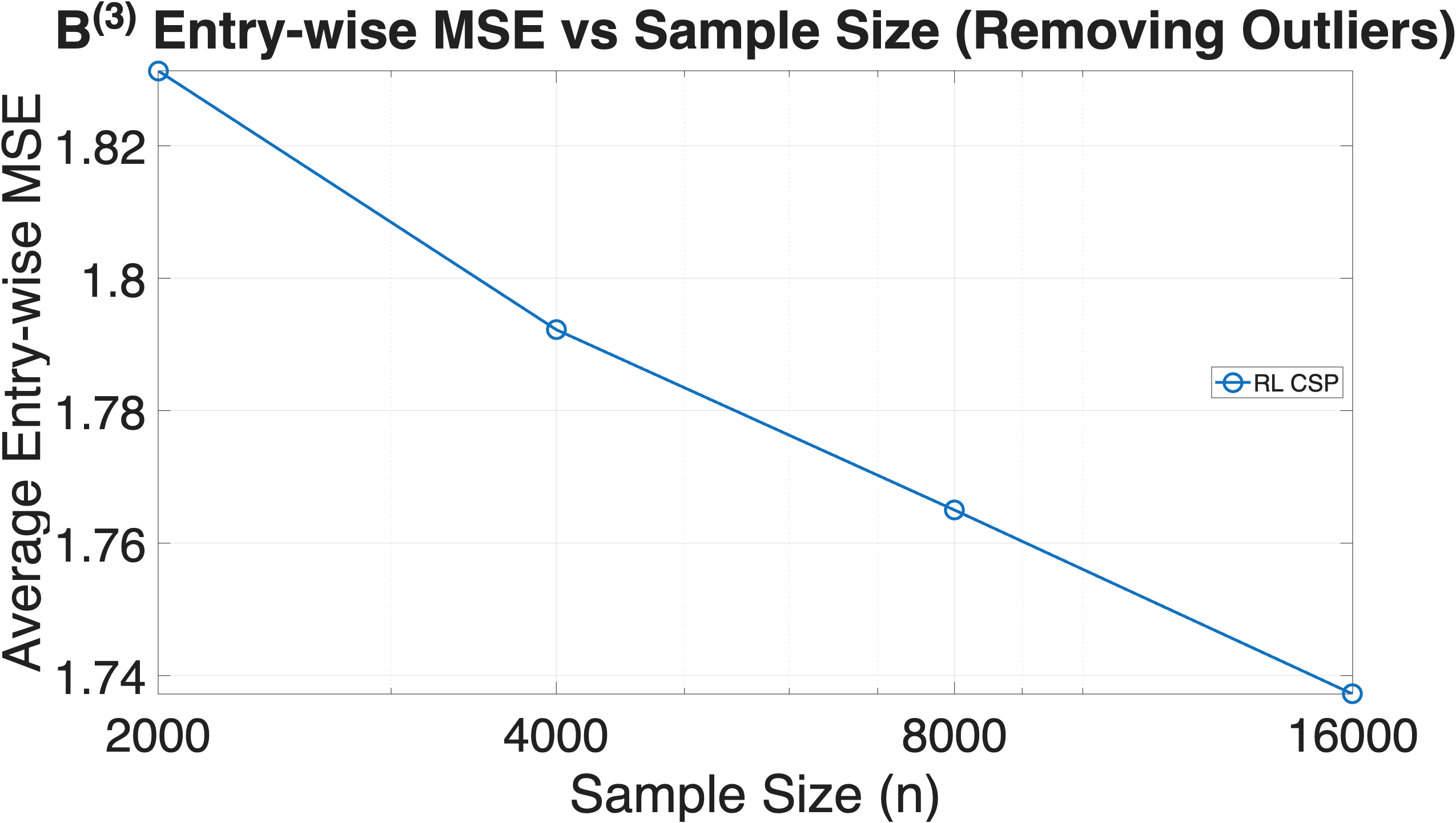}
    \caption{$D = 3$: Average entry-wise MSE for estimates of $\{\boldsymbol{B}^{(d)}\}_{d=1}^{3}$ under the RL CSP DDE Copula} 
    \label{fig:MSED3}
\end{figure}
We also include evaluations of the estimated weight matrices in Figures \ref{fig:D3avg}-\ref{fig:MSED3} via entry-wise MSE and the average MAP estimate of each layer-specific $\boldsymbol B^{(d)}$, respectively. In Figure \ref{fig:MSED3}, we point out that the apparent lack of convergence for the deepest layer is due to the fact that, for simplicity, we fit all models under the same hyperparameter settings. In many cases, this resulted in the estimated dimension for the third dimension to either 1 or 4, which severely skews entry-wise MSE. In practice, one would conduct predictive evaluations as outlined in Section \ref{sec:realdat} to tune these parameter. In addition, at the two smallest sample sizes, graph recovery is uniformly poor, that the apparent better performance in these regimes is purely an artifact of the simulation design and sampling variability. After removing these outliers from simulations and focusing on sample sizes where the graph recovery begins to stabilize $(n \geq 2000)$, the transition toward consistent recovery is most clearly visible.
\begin{figure}[h]
    \centering
    \includegraphics[width=\linewidth, keepaspectratio]{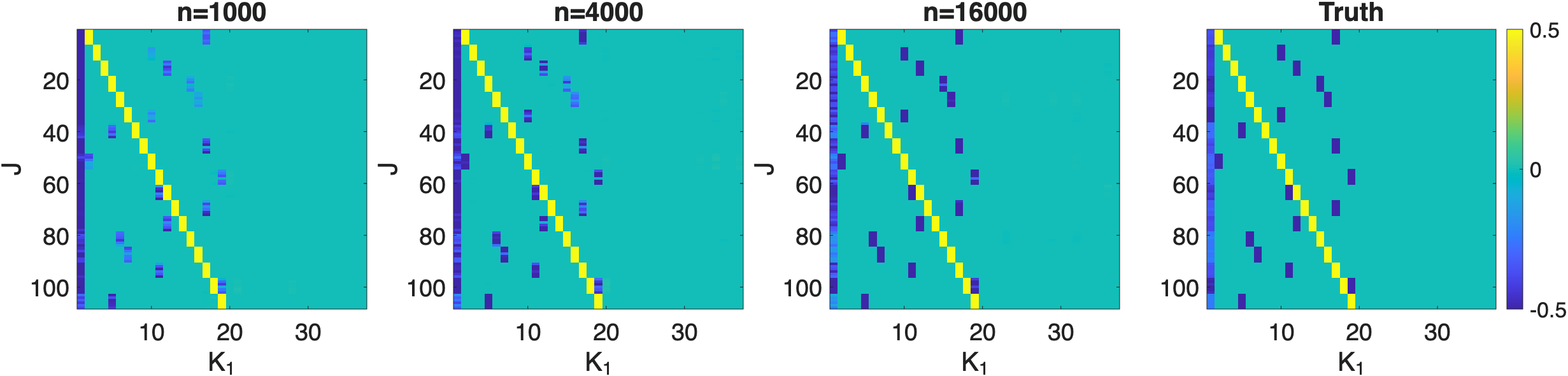}
        \includegraphics[width=\linewidth, keepaspectratio]{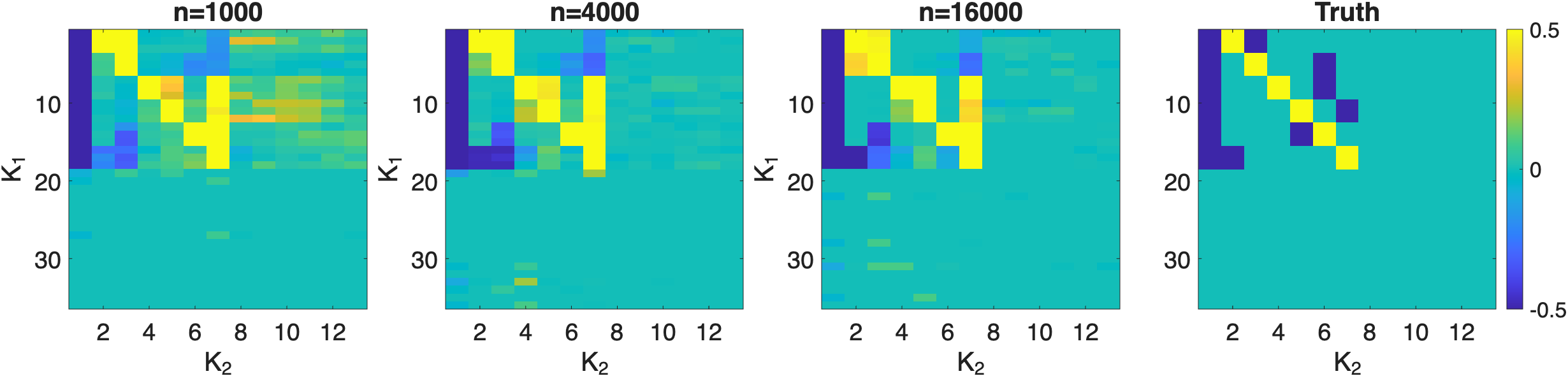}
            \includegraphics[width=\linewidth, keepaspectratio]{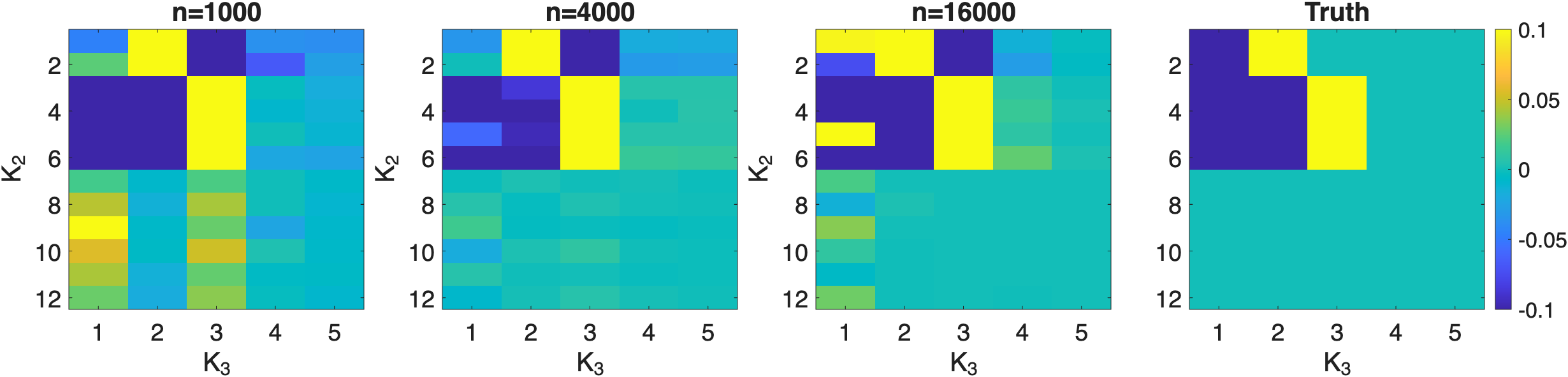}
    \caption{$D = 3$ average layer-wise MAP estimates compared to the truth across sample sizes.}
    \label{fig:D3avg}
\end{figure}

\section{Comparisons to iVAE}\label{sec:iVAE}

We provide additional comparisons by analyzing the Big5 survey with identifiable variational auto-encoders (ivae; \cite{ivae}). Motivated by our discovery of $K^{(1)^{*}} = 9$ and $K^{(2)^{*}} = 1$, we specify an iVAE with latent dimension $K = 1$, while both encoder and decoder neural network architectures (two-layer perceptrons) have 9 hidden variables in each layer. The iVAE requires auxiliary information to construct ``label" priors and we supply the model with the political ideology item used in Figure \ref{fig:ideology} in the main text. We utilize default settings provided by the authors for the algorithm.

\begin{figure}[h]
    \centering
    \includegraphics[width=0.6\linewidth]{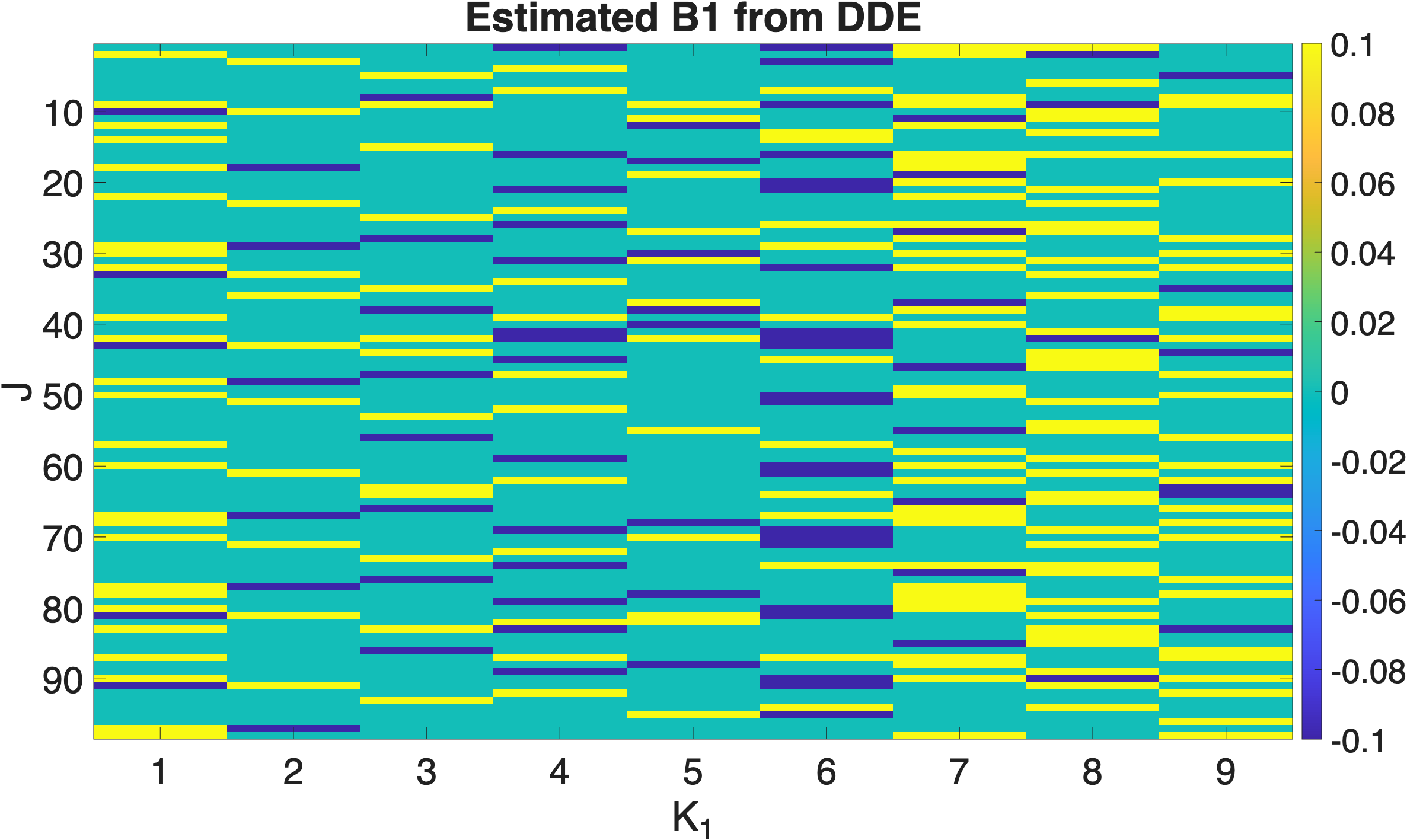}
    \includegraphics[width=0.49\linewidth]{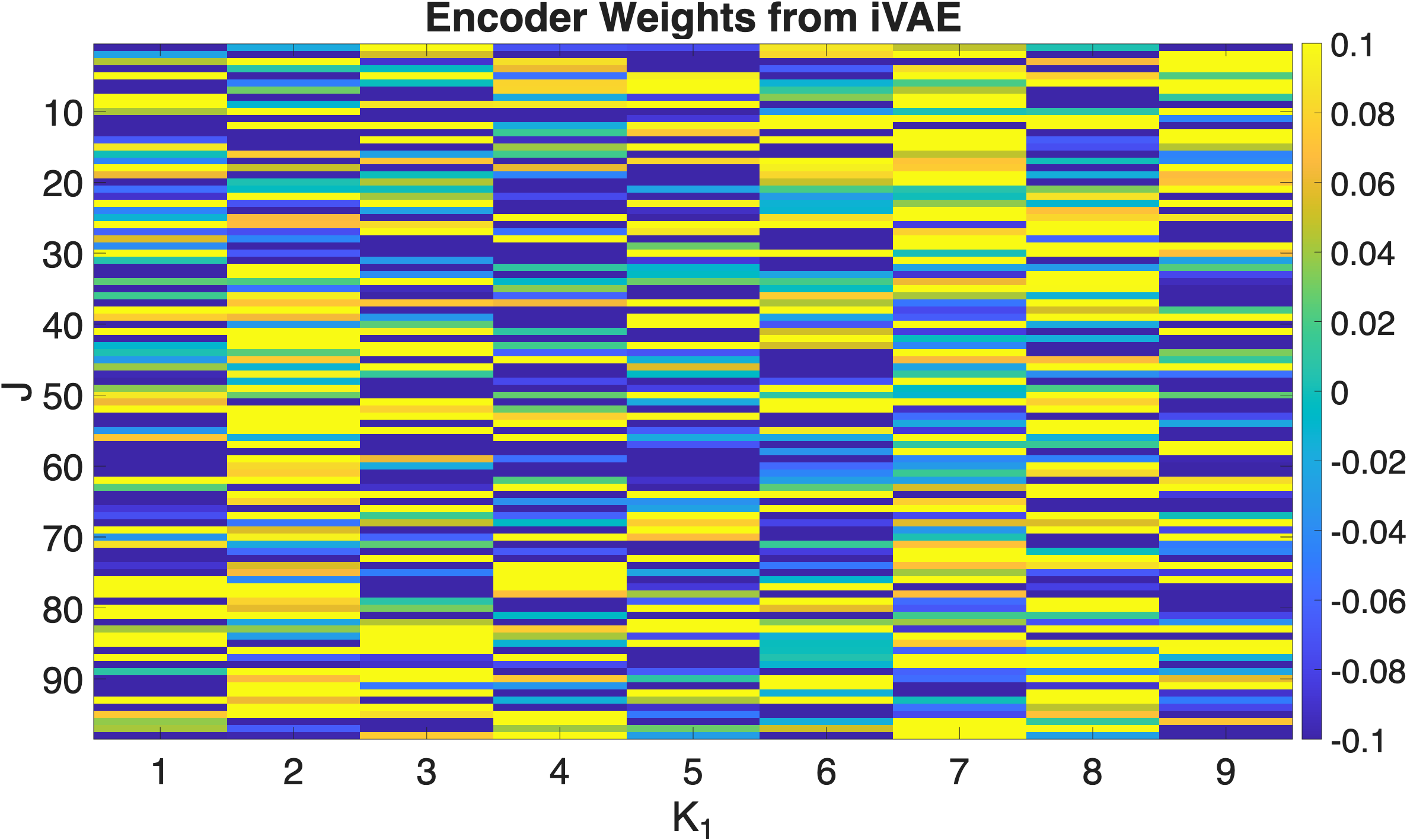}
    \includegraphics[width=0.49\linewidth]{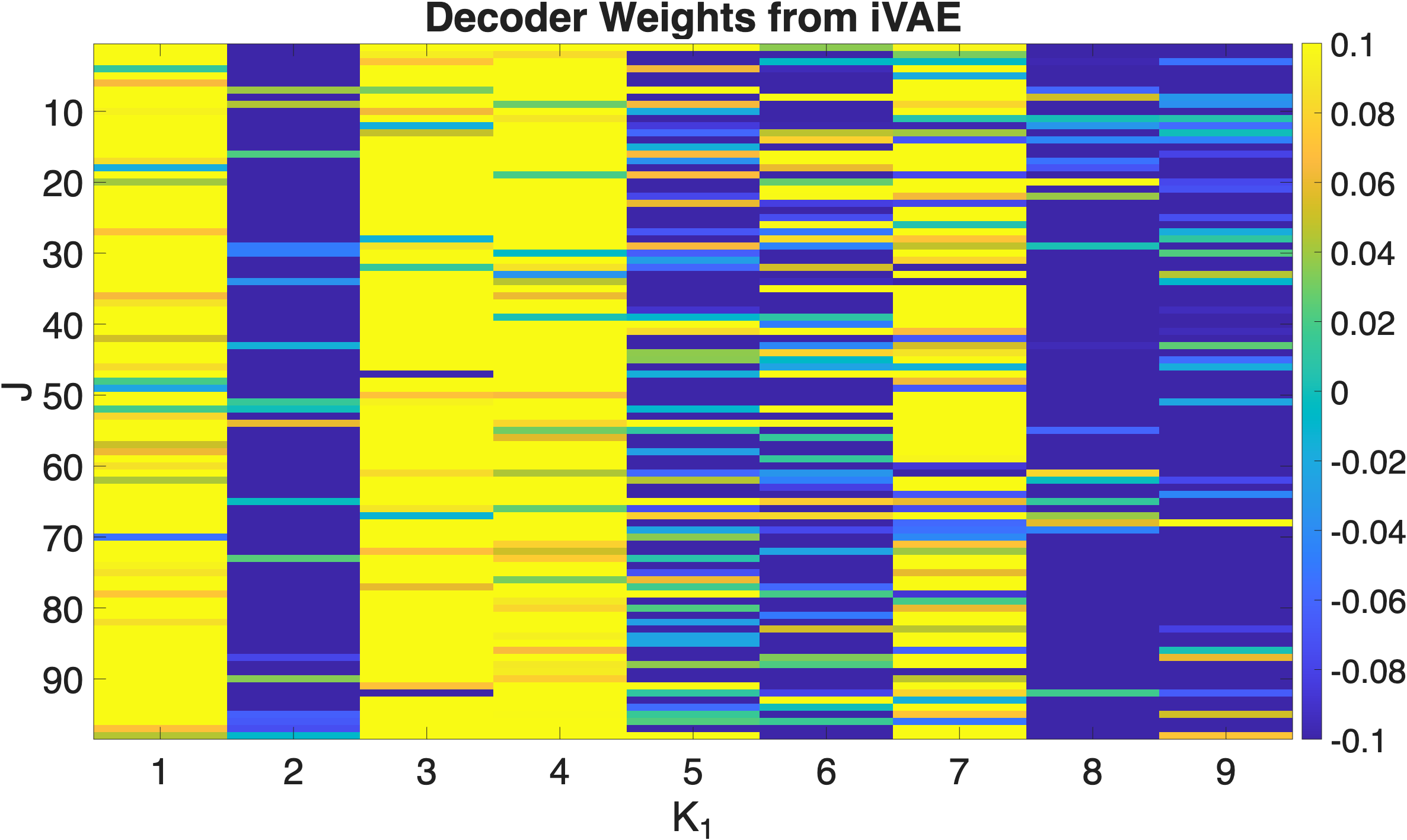}
    \caption{Comparison of estimated DDE weights to iVAE encoder and deconder weights fit to Big5 data.}
    \label{fig:iVAEcomp}
\end{figure}

We first compare the estimated encoder and decoder weights, each of which are $J \times 9$, to the estimated $\boldsymbol{B}^{(1)}$ under the DDE copula in Figure \ref{fig:iVAEcomp}. The estimated $\boldsymbol{B}^{(1)}$ is considerably more sparse than both the encoder and decoder weights. Next, we compute key items based on both the encoder and decoder weight matrices. We note that for the decoder weights, several columns were nearly entirely negative, and there were no items that surfaced based on the metric $\max\{\min_{l \neq k} (\beta_{jk} - \beta_{jl}),\, 0\}$. The results are available in Table \ref{comp:enc}.

For the iVAE weights, $A^{(1)}_{k}$ corresponds to the hidden variables in the respective perceptrons. To further facilitate comparisons, we include key items obtained from $\boldsymbol B^{(1)}$ from the DDE copula, which are presented in Table \ref{tab:keyitems} in the main text. Neither the encoder nor decoder weights capture structure in the data; key items are not meaningfully grouped across latent variables. 

{
\setlength{\itemsep}{0pt}
\setlength{\parsep}{0pt}
\setlength{\parskip}{0pt}
\renewcommand{\baselinestretch}{0.9}\normalsize
\begin{table}[ht]
\centering
\scriptsize
\setlength{\tabcolsep}{2pt}
\renewcommand{\arraystretch}{1.15}
\begin{tabular}{p{3cm} *{9}{>{\raggedright\arraybackslash}p{1.4cm}}}
\hline
& \multicolumn{9}{c}{\textbf{Binary latent feature}} \\
\cline{2-10}
& $A_1^{(1)}$ & $A_2^{(1)}$ & $A_3^{(1)}$ & $A_4^{(1)}$ & $A_5^{(1)}$ & $A_6^{(1)}$ & $A_7^{(1)}$ & $A_8^{(1)}$ & $A_9^{(1)}$ \\
\hline

\textbf{Key Item 1}
& Keep others at a distance
& Talk to a lot of different people at parties
& Follow through with my plans
& Do not like art
& Am very pleased with myself
& Do not like to draw attention to myself
& Worry about things
& Have a good word for everyone
& Shirk my duties
\\ \hline

\textbf{Key Item 2}
& Am hard to get to know
& Am the life of the party
& Carry out my plans
& Do not enjoy going to art museums
& Feel comfortable with myself
& Keep in the background
& Panic easily
& Enjoy hearing new ideas
& Do just enough work to get by
\\ \hline

\textbf{Key Item 3}
& Avoid contact with others
& Make friends easily
& Do things according to a plan
& Do not like poetry
& Seldom feel blue
& Don't talk a lot
& Get stressed out easily
& Sympathize with other's feelings
& Don't see things through
\\ \hline

\end{tabular}
\begin{tabular}{p{3cm} *{9}{>{\raggedright\arraybackslash}p{1.4cm}}}
\hline
& \multicolumn{9}{c}{\textbf{Binary latent feature}} \\
\cline{2-10}
& $A_1^{(1)}$ & $A_2^{(1)}$ & $A_3^{(1)}$ & $A_4^{(1)}$ & $A_5^{(1)}$ & $A_6^{(1)}$ & $A_7^{(1)}$ & $A_8^{(1)}$ & $A_9^{(1)}$ \\
\hline

\textbf{Key Item 1}
& Feel comfortable around people
& Do not mind being the center of attention
& Treat all people equally
& Seldom get mad
& Believe too much tax money supports artists
& Insult people
& Keep others at a distance
& Finish what I start
& Can say things beautifully
\\ \hline

\textbf{Key Item 2}
& Cut others to pieces
& Get back at others
& Have a good word for everyone
& Accept people as they are
& Seldom feel blue
& Get excited by new ideas
& Do not like art
& Make demands on others
& Avoid philosophical discussions
\\ \hline

\textbf{Key Item 3}
& Would describe my experiences as dull
& Enjoy hearing new ideas
& Cheer people up
& Do things according to a plan
& Have a sharp tongue
& Get stressed out easily
& Enjoy thinking about things
& Contradict others
& Have frequent mood swings
\\ \hline

\end{tabular}
\begin{tabular}{p{3cm} *{9}{>{\raggedright\arraybackslash}p{1.4cm}}}
\hline
& \multicolumn{9}{c}{\textbf{Binary latent feature}} \\
\cline{2-10}
& $A_1^{(1)}$ & $A_2^{(1)}$ & $A_3^{(1)}$ & $A_4^{(1)}$ & $A_5^{(1)}$ & $A_6^{(1)}$ & $A_7^{(1)}$ & $A_8^{(1)}$ & $A_9^{(1)}$ \\
\hline

\textbf{Key Item 1}
& Have a sharp tongue
& \text{---}
& Make people feel at ease
& Know how to captivate people
&\text{---}
& \text{---}
&Believe too much tax money supports artists
& \text{---}
& \text{---}
\\ \hline

\textbf{Key Item 2}
& Do things according to a plan
& \text{---}
& Am relaxed most of the time
& Pay attention to details
& \text{---}
& \text{---}
& Mess things up
& \text{---}
& \text{---}
\\ \hline

\textbf{Key Item 3}
& Am the life of the party
& \text{---}
& Have a good word for everyone
& Start conversations
& \text{---}
& \text{---}
& Seldom feel blue
& \text{---}
& \text{---}
\\ \hline

\end{tabular}
    \caption{Key items based on $\boldsymbol{B}^{(1)}$ from the DDE copula (top; as presented in Figure~\ref{tab:keyitems}), as well as the encoder (middle) and decoder (bottom) weights from the iVAE}
    \label{comp:enc}
\end{table}
}

\clearpage
\bibliographystyle{apalike}
\bibliography{bibliography}

\end{document}